\documentclass{article}

 \usepackage[preprint]{neurips_2026}

\usepackage[utf8]{inputenc} % allow utf-8 input
\usepackage[T1]{fontenc}    % use 8-bit T1 fonts
\usepackage{hyperref}       % hyperlinks
\usepackage{url}            % simple URL typesetting
\usepackage{booktabs}       % professional-quality tables
\usepackage{amsfonts}       % blackboard math symbols
\usepackage{nicefrac}       % compact symbols for 1/2, etc.
\usepackage{microtype}      % microtypography
\usepackage{xcolor}         % colors
\usepackage{graphicx}
\usepackage{subcaption}
\usepackage{amsmath,amsfonts,bm}

\def\eqref#1{equation~\ref{#1}}
\def\1{\bm{1}}

\DeclareMathAlphabet{\mathsfit}{\encodingdefault}{\sfdefault}{m}{sl}
\SetMathAlphabet{\mathsfit}{bold}{\encodingdefault}{\sfdefault}{bx}{n}

\newcommand{\E}{\mathbb{E}}

\newcommand{\R}{\mathbb{R}}

\usepackage{amsfonts,amsmath,amsthm,amssymb}
\usepackage{mathtools}
\usepackage{mathrsfs}
\usepackage{hyperref}
\usepackage{url}
\usepackage{float}
\usepackage{graphicx}
\usepackage{booktabs}
\usepackage{tabularx}
\usepackage{multirow}
\usepackage{makecell}
\usepackage{mathrsfs}
\usepackage{algorithm}                  % floating algorithm environment
\usepackage{algpseudocode}
\usepackage{enumitem}
\usepackage{wrapfig}
\usepackage{caption}
\newcommand{\OT}{\mathrm{OT}}
\newcommand{\STP}{\mathrm{STP}}
\newcommand{\mSTP}{\mathrm{min}\text{-}\mathrm{STP}}

\newcommand{\SOT}{\mathrm{SOT}}

\renewcommand{\R}{\mathbb{R}}
\renewcommand{\S}{\mathbb{S}}
\newcommand{\osc}{\mathcal{O}}

\newcommand{\scriptd}{\mathscr{D}}

\usepackage[capitalize,noabbrev]{cleveref}

\newtheorem{theorem}{Theorem}[section]
\newtheorem{proposition}[theorem]{Proposition}
\newtheorem{definition}[theorem]{Definition}
\newtheorem{corollary}[theorem]{Corollary}

\newtheorem{assumption}[theorem]{Assumption}

\newtheorem{lemma}[theorem]{Lemma}

\title{Efficient Transferable Optimal Transport via Min-Sliced Transport Plans}

\author{
Xinran Liu$^{1}$ \quad
Elaheh Akbari$^{1}$ \quad
Rocio Diaz Martin$^{2}$ \AND
Navid NaderiAlizadeh$^{3}$ \quad
Soheil Kolouri$^{1,4}$ \\
\\
$^{1}$Department of Computer Science, Vanderbilt University \\
$^{2}$Department of Mathematics, Florida State University \\
$^{3}$Department of Biostatistics \& Bioinformatics, Duke University \\
$^{4}$Department of Electrical \& Computer Engineering, Vanderbilt University \\
{\tt\small xinran.liu@vanderbilt.edu, elaheh.akbari@vanderbilt.edu, rd25v@fsu.edu,} \\
{\tt\small navid.naderi@duke.edu, soheil.kolouri@vanderbilt.edu}
}

\begin{document}

\maketitle

\begin{abstract}
Optimal Transport (OT) plans provide correspondences between distributions supporting alignment tasks in various domains. Sliced transport plans have been recently proposed as a computationally efficient alternative to OT plans. These methods optimize a one-dimensional projection (slice) to obtain a conditional transport plan that minimizes the transport cost in the ambient space. Despite their efficiency, it remains unclear whether learned slicers transfer to new distribution pairs under shift, an issue central to evolving data and repeated OT computations over related distributions. We study the min-Sliced Transport Plan (min-STP) framework and examine slicer transferability: can a slicer learned on one distribution pair produce effective transport plans for unseen pairs? Theoretically, we show that optimized slicers remain close under slight perturbations of the data distributions, enabling efficient transfer across related tasks. To further improve scalability, we introduce a minibatch formulation of min-STP and provide statistical rates on its accuracy. Empirically, we demonstrate that the transferable min-STP achieves strong one-shot matching performance and facilitates amortized training for point cloud and image analysis. 
\end{abstract}

\vspace{-0.15in}
\section{Introduction}
\label{sec:intro}
\vspace{-0.1in}

Many tasks in computer vision, natural language processing, biology, and operations research involve establishing correspondences between samples drawn from two distributions; a problem elegantly framed through the lens of optimal transport (OT) \cite{Villani2009Optimal,peyre2019computational}. OT provides a principled framework for determining the most efficient coupling between distributions while minimizing a transport cost, consistent with the principle of least effort (or least action) in transferring mass. Over the past decade, OT has been applied across a broad range of domains, including domain adaptation \cite{courty2016optimal,damodaran2018deepjdot}, generative modeling \cite{arjovsky2017wasserstein,gulrajani2017improved,kornilov2024optimal}, shape and image matching \cite{saleh2022bending,shen2021accurate,bai2022sliced}, word-embedding alignment \cite{yurochkin2019hierarchical,huynh2020otlda}, single-cell analysis \cite{schiebinger2019optimal,yang2020predicting}, and resource allocation \cite{guo2022online}, among many others \cite{khamis2024scalable}.

Despite its success, OT remains computationally expensive, with a complexity that scales cubically in the number of samples \cite{peyre2019computational}. To address this bottleneck, a large body of research has sought scalable approximations, including entropic regularization \cite{cuturi2013sinkhorn,altschuler2017near}, low-rank and linearized formulations \cite{scetbonlow,wang2013linear}, and hierarchical or multiscale methods \cite{schmitzer2016sparse,gerber2017multiscale,halmos2025hierarchical}. Slicing-based techniques \cite{bonnotte2013unidimensional,bonneel2015sliced,kolouri2022generalized} form another efficient class of methods, reducing high-dimensional problems to a set of one-dimensional projections. While effective for computing OT distances, these methods do not recover explicit transportation plans. Recent extensions \cite{rowland2019orthogonal,mahey2023fast,liu2025expected,chapel2025differentiable} have begun to bridge this gap by constructing one-dimensional transportation plans along slices and lifting them to the ambient space to obtain transportation plans between input distributions.

However, all existing efficient OT methods are designed for single-instance problems. In many real-world applications, one must repeatedly solve OT between evolving or structurally similar distributions, where valuable information from prior solutions remains unused. Even when two pairs of distributions are nearly identical, current solvers recompute everything from scratch. This inefficiency is particularly limiting in domains such as developmental biology, where cell populations evolve gradually and successive distributions differ only slightly.

In this work, we move beyond isolated optimization and propose a formal framework for quantifying and leveraging the \emph{transferability} of OT solutions. Our approach enables reusing information from previously solved instances, paving the way for adaptive, amortized OT across related problems. Building on differentiable generalized sliced plans~\cite{chapel2025differentiable}, we revisit the min-generalized slice formulation and present a unified theoretical and algorithmic framework with numerically stable implementations. The \emph{closed-form quantile function of the sum of Laplace distributions (LapSum)} \cite{struski2025lapsum} forms the backbone of both the theoretical and computational aspects of our method, allowing for efficient and differentiable evaluation of one-dimensional transport plans. Our central contribution is a \emph{transferability theorem} that establishes a formal link between distributional similarity and slicer proximity: when two source-target distribution pairs are similar, their optimal slicers are correspondingly close. To the best of our knowledge, this result provides the first theoretical foundation for amortized slicer reuse. Finally, we demonstrate that the slicer can be trained stochastically on \emph{mini-batches} ($n \ll N$), requiring only slice-and-sort operations on small subsets and enabling scalable optimization via stochastic gradient descent for large-scale datasets.

\noindent{\bf Contributions.} We advance the theoretical foundations and practical scalability of large-scale OT problems through the following contributions:

\begin{enumerate}[itemsep=0.2em, topsep=0pt, leftmargin=12pt]
    \item Characterizing the transferability of optimized slicers under small source-target perturbations;
    \item Deriving finite-sample statistical rates for slicers optimized using source-target mini-batches;
    \item Demonstrating transferability on point cloud alignment, mean-flow generations and image translation.
\end{enumerate}

\begin{figure*}[t!]
    \centering
     \vspace{-0.25in}
    \includegraphics[width=\linewidth]{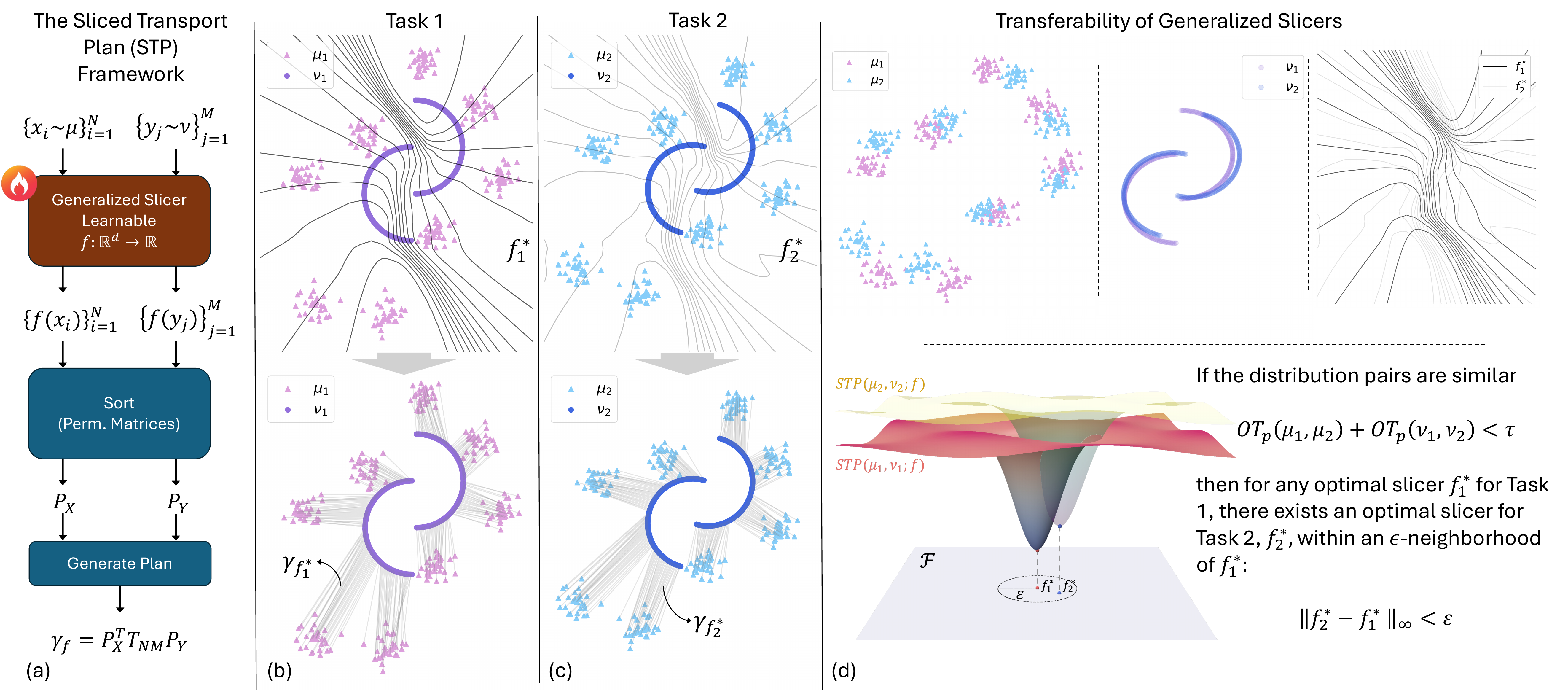}
    \caption{Overview of the Sliced Transport Plan (STP) framework and our transferability results.
    (a–c) The STP framework computes transport plans $\gamma_f$ using a generalized slicer $f : \mathbb{R}^d \to \mathbb{R}$, which projects high-dimensional samples onto one-dimensional marginals, sorts them via (soft) permutation matrices, and generates slice-wise transport plans that are lifted back to the ambient space to obtain a transportation plan between the input measures. 
    The min-STP framework extends STP by learning an optimal slicer $f^{*}$ that minimizes the transportation cost in the ambient space.
    (d) In this paper, we establish a \emph{transferability theorem} showing that if two source-target distribution pairs $(\mu_1, \nu_1)$ and $(\mu_2, \nu_2)$ are close, then there exists an optimal slicer $f_2^{*}$ for $(
    \mu_2,\nu_2)$ in an $\varepsilon$-vicinity of $f_1^{*}$. 
    This result enables efficient slicer reuse across related tasks, achieving amortized optimal transport.}
    \label{fig:teaser}
    \vspace{-0.2in}
\end{figure*}

\vspace{-0.15in}
\section{Preliminaries}
\vspace{-0.1in}
\begin{definition}\textbf{Optimal Transport (OT).}
Let $\mathcal{X}\subseteq\mathbb{R}^d$ endowed with a metric $c:\mathcal{X}\times \mathcal{X}\to \mathbb{R}$. We use $\mathcal{P}(\mathcal{X})$ to denote the set of Borel probability measures defined on $\mathcal{X}$, and  $\mathcal{P}_p(\mathcal{X})$ to denote the subset of probability measures with finite $p$\textsuperscript{th} moment ($p\in [1,\infty)$), i.e.,  $\mathcal{P}_p(\mathcal{X})\coloneqq\{\mu\in\mathcal{P}(\mathcal{X})\mid \int_{\mathcal{X}} c^p(x,x_0)\mathrm{d}\mu(x) < \infty, \text{ for some } x_0 \in \mathcal{X}\}$. For $\mu,\nu\in \mathcal{P}_p(\mathcal{X})$, the optimal transport distance (aka the p-Wasserstein distance) is defined as: 
\vspace{-.1in}
\begin{equation}
\OT_p(\mu,\nu)\coloneqq\left(\inf_{\gamma\in\Gamma(\mu,\nu)} \int_{\mathcal{X}\times\mathcal{X}} c^p(x,y) \, \mathrm{d}\gamma(x,y)\right)^{\frac{1}{p}},
    \label{eq:ot}
\end{equation}
where $\Gamma(\mu,\nu)\coloneqq\{\gamma\in \mathcal{P}_p(\mathcal{X}\times \mathcal{X})\mid \pi^1_\#\gamma =\mu,~\pi^2_\#\gamma=\nu\}$, $\pi^i$'s denote the canonical projections, and $f_\#\mu$ denotes the pushforward of $\mu$ through $f$, which is formally defined as $f_\#\mu(A)=\mu(f^{-1}(A))$ for any Borel subset $A\subseteq \mathcal{X}$.
\end{definition}

For $p\in[1,\infty)$, $\OT_p$ defines a proper metric between $\mu$ and $\nu$, with many favorable geometric characteristics. 

A prominent computational paradigm in optimal transport is the sliced-OT framework \citep{rabin2011wasserstein}, which leverages the efficiency of OT \citep{peyre2019computational,bai2022sliced} solvers on one-dimensional measures by computing the OT distances between the one-dimensional marginals of $d$-dimensional measures. Owing to the injectivity of the Radon transform and the Cramér-Wold theorem (probabilistic route) or the Fourier Slice theorem (analytic viewpoint), comparing these marginals defines a valid metric between the original $d$-dimensional measures \citep{nadjahi2020statistical,kolouri2022generalized}. 

 \begin{definition} \textbf{Sliced OT.}
      Let $f:\R^d\times \mathbb{S}^{d-1}\to \R$ defined by the scalar projection $f(x,\theta)=\langle x,\theta\rangle$, where $\langle x,\theta\rangle$ denotes the standard inner product in $\R^d$. The sliced OT (SOT) problem between two probability measures $\mu,\nu\in 
      \mathcal{P}_p(\mathcal{X})$ is defined as: 
 \begin{equation*}
    \SOT_p(\mu,\nu)\coloneqq\left( \int_{\mathbb{S}^{d-1}} \OT_p^p( f(\cdot,\theta)_\#\mu,f(\cdot,\theta)_\#\nu)\mathrm{d}\sigma(\theta)\right)^{\frac{1}{p}}\hspace{-5pt}, 
    \label{eq:sot}
\end{equation*}
where $\sigma\in\mathcal{P}(\S^{d-1})$ is a probability measure with non-zero density everywhere on the $d$-dimensional unit sphere $\S^{d-1}$. 
\end{definition}

The notion of slicing was later extended to \emph{generalized slices} by allowing $f(\cdot, \theta): \mathbb{R}^d \times \mathbb{R}^{d'} \to \mathbb{R}$ to be a nonlinear parametric function. %Throughout this paper, the slicers are parametric functions. 
With a slight abuse of notation, we drop $\theta$ and write $f$ when unambiguous. Under suitable assumptions on $f$ \citep{kolouri2019generalized,chen2020augmented}, it has been shown that SOT defines a proper metric between measures $\mu, \nu \in \mathcal{P}_p(\mathbb{R}^d)$. Notably, extensive work has also been dedicated to the choice of $\sigma$, which is orthogonal to our discussion and is not discussed in this paper \citep{nguyen2021distributional,nguyen2023energybased,nguyen2024sliced}. 
\vspace{-0.1in}
\subsection{Sliced Transportation Plans}
\vspace{-0.1in}
The central idea behind methods that estimate transportation plans via slicing is that the one-dimensional transportation plan between the pushforward measures \( f_{\#}\mu \) and \( f_{\#}\nu \), can be lifted back to define a coupling in the original space, denoted by \( \gamma_f \). Following \citep{liu2025expected}, in the case of discrete measures and injective projections \( f\), this lifting is straightforward: the unique one-dimensional optimal plan calculated in the sliced  space, \( \sigma_f \in \Gamma(f_{\#}\mu ,f_{\#}\nu) \),   corresponds directly to a valid transportation plan for the original measures, $\gamma_f\in \Gamma(\mu,\nu)$. To the best of our knowledge, this idea of lifting via injective projections was first introduced by \cite{rowland2019orthogonal} as part of the Projected Wasserstein (PW) distance framework between empirical distributions and was later adopted in the influential work of \cite{mahey2023fast}. When the projections \( f_{\#}\mu \) and \( f_{\#}\nu \) are not injective, however, the lifting becomes non-trivial and requires more careful handling, as discussed in \cite{liu2025expected}. 

% For simplicity, in the remainder of this paper, we do not distinguish between the optimal one-dimensional transportation plan \( \Lambda_\theta \) and its lifted counterpart \( \gamma_\theta \).

\begin{definition}
\label{def:stp}
\textbf{Sliced transportation plan (STP).} Let \( \gamma_f \) denote the transportation plan between discrete probability measures \( \mu, \nu \in \mathcal{P}_p(\mathbb{R}^d) \), obtained by lifting the optimal 1D plan from slices via an injective map \( f \) defined on the supports of \( \mu \) and \( \nu \). We define the sliced transportation plan dissimilarity between $\mu$ and $\nu$ as:
\begin{equation*}
    \STP_p(\mu, \nu;f) \coloneqq \left(\sum_{x \in \mathrm{supp}(\mu)} \sum_{y \in \mathrm{supp}(\nu)} c^p(x,y) \, \gamma_f(x, y)\right)^{\frac{1}{p}}.\hspace{-5pt}
\end{equation*}
\end{definition}

\noindent For general measures $\mu, \nu\in \mathcal{P}_p(\mathcal{X})$, we formally write 
 \begin{equation*}
    \STP_p(\mu, \nu;f) \coloneqq \left(\int_{\mathcal{X}\times\mathcal{X}} c^p(x,y) \, \mathrm d \gamma_f(x, y)\right)^{\frac{1}{p}}
\end{equation*}
whenever $\gamma_f$ is a well-defined lift coupling; that is, if $\sigma_f$ denotes the unique optimal transport plan between $f_\#\mu$ and $f_\#\nu$, then $\gamma_f$ is the uniquely determined element of $\Gamma(\mu,\nu)$ satisfying the property $(f,f)_\#\gamma_f=\sigma_f$. Such uniqueness is guaranteed, for instance, when $f$ is injective. 

 % Note that $\text{STP}_p(\mu, \nu;\theta)$ is implicitly dependent on $f$, and we will use $\text{STP}_p(\mu, \nu;\theta)$ and $\text{STP}_p(\mu, \nu; f)$ interchangeably when $f$ is parametrized by $\theta$.
 When $f(\cdot)=\langle \cdot,\theta\rangle$, for a certain $\theta\in \mathbb S^{d-1}$, and the cost function is $c(x,y)=\|x-y\|_p$, $\text{STP}_{p}$ coincides with the Sliced Wasserstein Generalized Geodesics (SWGG) distance \citep{mahey2023fast}, recently generalized for general measures \cite{tanguy2025sliced}.

\begin{proposition}(Proposition~2 in \cite{chapel2025differentiable})
Let $\mu,\nu \in \mathcal{P}_p(\mathcal{X})$. Let $f:\mathcal{X}\to\mathbb{R}$ be an injective map on the supports of $\mu$, and $\nu$. Then, for $p\geq 1$, $\STP_p(\cdot,\cdot~;f)$ is a distance on $\mathcal{P}_p(\mathcal{X})$.
\end{proposition}

Since $\gamma_f$ is a valid coupling of $\mu$ and $\nu$, though not necessarily optimal, the cost induced by $\gamma_f$ upper bounds the infimum in \eqref{eq:ot}, i.e.,  
\begin{equation*}
    \OT_p(\mu,\nu)\leq \STP_p(\mu,\nu;f).
\end{equation*}
% The interesting point about $\text{STP}_p$ is that it provides a transportation plan $\gamma_f$ between the input measures, although this transportation plan is not necessarily an optimal transportation plan, i.e., $\gamma^*\neq \gamma_f$, where $\gamma^*$ is the optimizer for problem  \ref{eq:ot}. 
While $\STP_p$ provides an upper bound for $\OT_p$, the gap between the two could be large due to the potential discrepancy between $\gamma_f$ and an optimal transport plan $\gamma^*$ for \eqref{eq:ot}. This motivates the introduction of the $\text{min-}\STP_p$, which coincides with the min-SWGG \citep{mahey2023fast} and Differentiable Generalized Sliced Wasserstein Plan (DGSWP) defined by \cite{chapel2025differentiable} in specific cases.

\begin{definition}
$\boldsymbol{\mSTP_p}$. Let $\mu,\nu\in \mathcal{P}_p(\mathcal{X})$. The minimum Sliced Transportation Plan ($\mSTP_p$) between $\mu$ and $\nu$  is defined as: 
\begin{equation}
    \mathrm{min-}\STP_p(\mu,\nu)=\min_{f\in\mathcal{F}} ~\STP_p(\mu,\nu;f),
    \label{eq:minSTP}
\end{equation}
where $\mathcal{F}$ is a class of parametric functions. 
\end{definition}
Clearly, $\mSTP_p$ satisfies:
\begin{equation*}
    \OT_p(\mu,\nu)\leq  \mSTP_p(\mu,\nu) \leq \STP_p(\mu,\nu;f).
\end{equation*}
Moreover, if $f:\mathcal{X}\to\mathbb{R}$ is injective on $\mathcal{X}\subset\mathbb{R}^d$, then $\mSTP_p$ provides a pseudo-metric on $\mathcal{P}_p(\mathcal{X})$ ($p\geq 1$). 
% \begin{proposition}
%     Let $f:\mathcal{X}\to\mathbb{R}$ be an injective map on $\mathcal{X}\subset\mathbb{R}^d$ and $p\geq 1$. Then, min-STP$_p$ provides a pseudo-metric on $\mathcal{P}_p(\mathcal{X})$. 
% \end{proposition}

\vspace{-0.1in}
\section{Method and Theoretical Results}
\label{sec:method}
\vspace{-0.1in}

\subsection{Transferability}
\vspace{-0.1in}
We first discuss the transferability of optimal slicers $f^\star:\mathcal{X}\rightarrow\mathbb{R}$ in problem \ref{eq:minSTP}. %We show that if the input (source and target) distributions are perturbed slightly, a new corresponding optimal slicer exhibits only minor deviations.
We show that when the source and target distributions experience small variations, the corresponding optimal slicer varies only marginally. This stability property ensures that \(f^\star\) can be reliably transferred as a prior for the varied distributions.

\begin{definition}[Perturbed slicers]\label{def:lin-laplace}
Let $\mathsf{Lap}_d(\eta^2\Sigma)$ denote the symmetric multivariate Laplace distribution on $\mathbb R^d$ with mean 0 and covariance matrix $\eta^2\Sigma$, for a scale parameter $\eta>0$ and a symmetric positive definite matrix $\Sigma\in\mathbb{R}^{d\times d}$.
Given a measurable function $f:\mathcal{X}\to \R$, we define its perturbation by
\[
g_{\xi_\eta, f}(x)\ \coloneqq\ f(x)+\langle \xi_\eta, x\rangle, \text{ where } \xi_\eta\sim \mathsf{Lap}_d(\eta^2\Sigma).
\]
Unless more than a single function $f$ is involved, we will omit the subscript $f$ and write $g_{\xi_\eta}$.
\end{definition}

Throughout this paper, $\Sigma$ stays fixed and we focus on varying scales $\eta$.  This choice is motivated by LapSum \cite{struski2025lapsum}, which applies Laplace perturbations for order-type smoothing, and it integrates naturally with LapSum’s differentiable sorting mechanisms used in our implementations (see Section \ref{subsec: lapsum}). Specifically, the perturbation ensures that, for any $\mu\in \mathcal{P}(\mathcal{X})$,  with $\mu$-probability one, the projection $g_{\xi_\eta}$ is injective on the support of $\mu$ (see Lemma \ref{lem:ae-inj} in the Appendix). This thereby enables the lift of a 1D coupling to the original space: Given $\mu,\nu\in \mathcal{P}(\mathcal{X})$, for each $\sigma\in \Gamma((g_{\xi_\eta})_\#\mu, (g_{\xi_\eta})_\#\nu)\subset \mathcal{P}(\R\times \mathbb{R})$, there exists one and only one (lifted) coupling $\gamma\in \Gamma(\mu,\nu)\subset \mathcal{P}(\mathcal{X}\times \mathcal{X})$ such that $(g_{\xi_\eta},g_{\xi_\eta})_\#\gamma=\sigma$ (see Proposition \ref{prop:unique-lift} in the Appendix).

Generalizing the  classical projectors $f(x,\theta)=\langle x,\theta\rangle$ while preserving some of their fundamental properties, our assumptions are as follows:

\begin{assumption}\label{asmp:model} Let $(\mathcal X,c)$ be a compact metric space in $\mathbb R^d$, and let $\mathcal F$ consist of all Lipschitz continuous functions $f:\mathcal X\!\to\!\mathbb R$ with global Lipschitz constant  $L>0$, and uniformly bounded by $M>0$, i.e., $\mathcal{F} \coloneqq \{f\in C(\mathcal{X})\mid \, Lip(f)\le L, \, \|f\|_\infty\leq M\}$, where $Lip(f)$ denotes the Lipschitz constant of $f$. Specifically, we work with discrete measures and further require $\mathcal X$ to be discrete with at most countable samples.
\end{assumption}

Under these assumptions, the slicer family $\mathcal F$ is compact in the uniform norm (as a consequence of the Arzelà–Ascoli theorem; see Lemma \ref{rem: precompact} in the Appendix). Thus, minimizers of regular functionals are attained over $\mathcal F$. In particular, we next introduce an objective obtained by applying Laplace perturbations to $\STP$, which has better regularity properties and recovers $\STP$ in the limit as the perturbations vanish.

\begin{definition}[Expected lifted plan]\label{def:Jeta-lift}
Let $\mu,\nu\in \mathcal{P}(\mathcal{X})$, and $f\in \mathcal{F}$ with perturbation $g_{\xi_\eta}$ for some $\eta>0$. 
Let $\sigma_{\xi_\eta}$ denote the 1D optimal plan between $(g_{\xi_\eta})_\#\mu$ and $(g_{\xi_\eta})_\#\nu$, with corresponding (unique) lifted coupling
\(
\gamma_{\xi_\eta}\in\ \Gamma(\mu,\nu),\)
such that \(
(g_{\xi_\eta},g_{\xi_\eta})_{\#}\gamma_{\xi_\eta}=\sigma_{\xi_\eta}.
\)
We define the expected lifted plan and the smoothed objective by
\begin{gather}
\bar\gamma_f^{(\eta)}\ \coloneqq\ \mathbb E_{\xi_\eta}\big[\,\gamma_{\xi_\eta}\,\big],\nonumber
\\
J_\eta(\mu,\nu;f)\ \coloneqq\ \int_{\mathcal X^2} c(x,y)^p \, \mathrm{d}\bar\gamma_f^{(\eta)}(x,y).
\end{gather}
\end{definition}

\begin{theorem}\label{thm}[Properties]\,
    \begin{enumerate}
        \item Let $\mu,\nu\in \mathcal{P}(\mathcal{X})$, $f\in \mathcal{F}$. Then, $\lim_{\eta\to 0} J_\eta(\mu,\nu;f)\ =\ \STP_p^p(\mu,\nu;f).$
        \item Let $\mu_1,\nu_1,\mu_2,\nu_2\in \mathcal{P}(\mathcal{X})$. Given $\varepsilon>0$, there exists $\tau>0$ such that if $\OT_p(\mu_1, \mu_2)+\OT_p(\nu_1, \nu_2)<\tau$, then for every $ f_1^\star\in\arg\min_{f\in\mathcal F} J_\eta(\mu_1,\nu_1;f)$, there exists $ f_2^\star\in\arg\min_{f\in\mathcal F} J_\eta(\mu_2,\nu_2;f)$ in the $\varepsilon$ neighborhood of $f_1^\star$, i.e., $\|f_2^\star-f_1^\star\|_\infty<\varepsilon$.
    \end{enumerate}
\end{theorem}

For the first property, we refer to Proposition \ref{prop:jeta-to-stp} in the Appendix, which uses the following auxiliary inequalities:
for $f, g\in\mathcal{F}$, $\kappa, \kappa'\in\mathcal{P}(\mathcal{X})$, and $\mu_i,\nu_i\in\mathcal{P}(\mathbb R)$ with optimal couplings $\gamma_i$, $i=1, 2$,
\begin{itemize}[noitemsep, topsep=0pt]
    \item \(\OT_p\big(f_\#\kappa,\;f_\#\kappa'\big)\;\le\;L\cdot \OT_p(\kappa,\kappa')\)
    \item \(
    \OT_p\big(f_\#\kappa,\;g_\#\kappa\big)\le\|f-g\|_\infty\)
    \item  $\OT_p^p\!\big(\gamma_1, \gamma_2\big)
\ \le\
 \OT_p(\mu_1,\mu_2)^p + \OT_p(\nu_1,\nu_2)^p$
\end{itemize}
(see Lemmas  \ref{lem:basic} and \ref{lem:tools} in the Appendix).

  \begin{wrapfigure}{r}{0.55\textwidth}
    \centering
    \vspace{-0.2in}
    \includegraphics[width=0.55\textwidth]{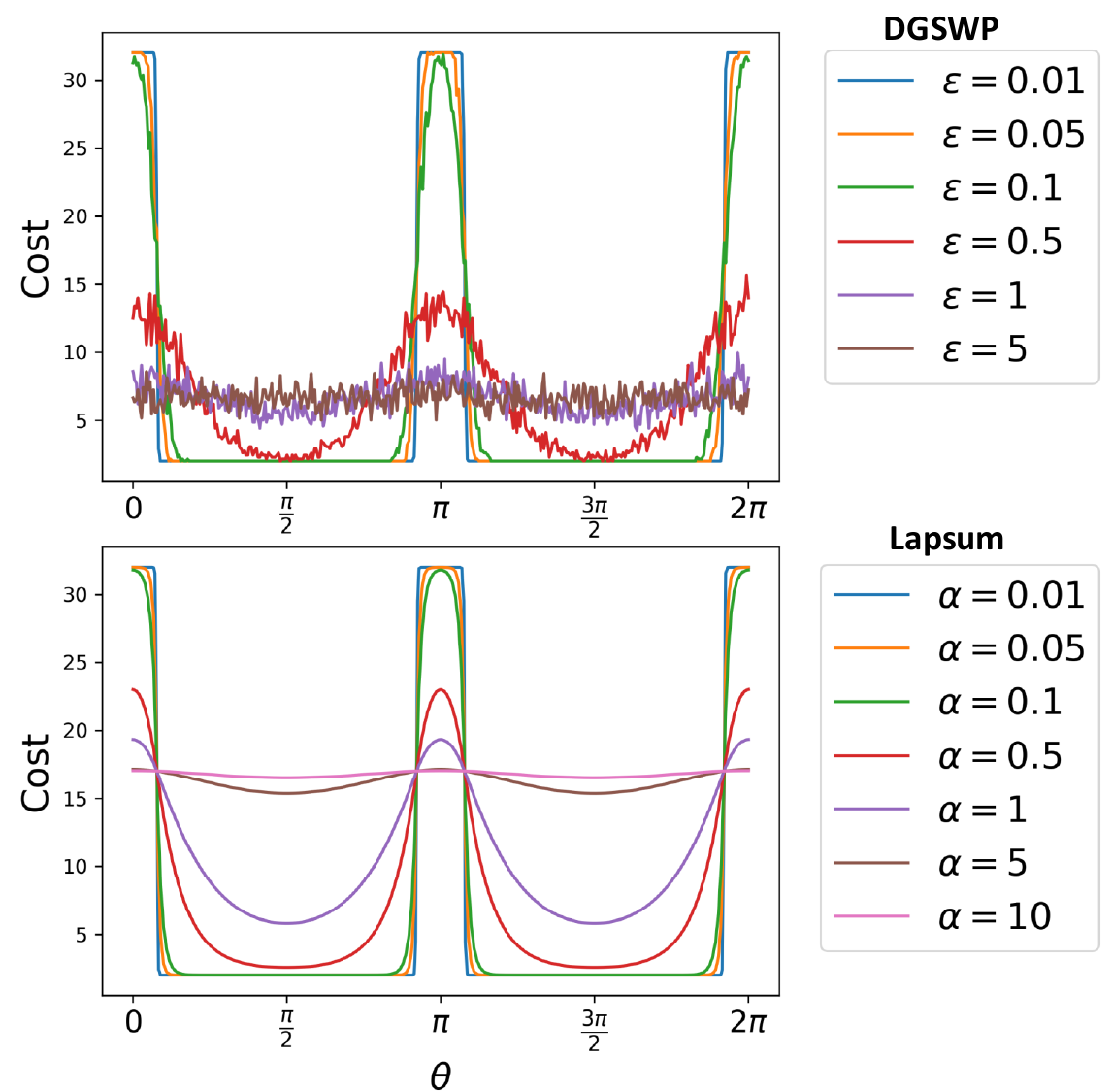}
    \vspace{-2em}
    \caption{Transport costs with respect to slicing directions $\theta\in\mathbb{S}^1$ for distribution pair $\mu=\frac{1}{2}\delta_{[0, 0]}+\frac{1}{2}\delta_{[1, 1]}$ and $\nu=\frac{1}{2}\delta_{[0, 1]}+\frac{1}{2}\delta_{[1, 0]}$, using DGSWP~\cite{chapel2025differentiable} and LapSum~\cite{struski2025lapsum}.}
    \label{fig:dgswp-vs-lapsum}
    \vspace{-2.0em}
\end{wrapfigure}

 The second property in Theorem \ref{thm} is  \emph{transferability}. It holds as an application of Berge’s Maximum Theorem \cite{aliprantis2006infinite} using compactness of $\mathcal F$ (Model Assumption \ref{asmp:model}) together with the continuity of $J_\eta$.  
 Indeed, Proposition \ref{prop:Jeta-cont} in the Appendix establishes that $J_\eta$ is continuous as a function on $\left(\mathcal{P}(\mathcal{X})\times\mathcal{P}(\mathcal{X})\right)\times \mathcal F$, and its proof relies on the three inequalities listed above.
 
For an illustration, see Figure \ref{fig:teaser} (d). For a preliminary discussion on the relationship between $\tau$ and $\varepsilon$, we refer the reader to the Appendix.
 
 % We expect that this analysis will motivate several directions for future research. Without appealing to the Berge’s Maximum Theorem, we can essentially say the following: For $i=1,2$, let $\mathcal{S}_i:=\arg\min_{f\in \mathcal F} J_\eta(\mu_i,\nu_i;f)$ denote the set of minimizers, which can be shown to be nonempty and compact under our assumptions. Given $\varepsilon>0$, let $U_\varepsilon(\mathcal{S}_i)$ denote the subset of functions in $\mathcal{F}$ whose distance to $\mathcal{S}_i$ is strictly less than $\varepsilon$. Consider the gaps $m_\varepsilon^{(i)}\ :=\ \min_{f\in \mathcal{F}\setminus U_\varepsilon(\mathcal{S}_i)} J_\eta(\mu_i,\nu_i;f)\ -\ \min_{g\in\mathcal F}J_\eta(\mu_i,\nu_i;g)$. Then if  $\mathcal{X}$ is finite, then with high probability one can choose $\tau\leq m_\varepsilon^{(1)}/C$, for a constant $C$ depending on $\eta$, and obtain $\mathcal{S}_2\subset U_\varepsilon(\mathcal{S}_1)$. Similarly, if $\tau\leq m_\varepsilon^{(2)}/C$, then $\mathcal{S}_1\subset U_\varepsilon(\mathcal{S}_2)$, which in particular implies that for any $f_1^*\in \mathcal{S}_1$, there exists $f_2\in \mathcal{S}_2$ with $\|f_2^\star-f_1^\star\|_\infty<\varepsilon$. The technical details, additional precision, and open directions are included in the Appendix, Section~\ref{sec: tau eps}.

\vspace{-0.1in}
\subsection{Smoothed Sorting (LapSum)}
\label{subsec: lapsum}

An essential step in minimizing the $\STP$ objective is computing the one-dimensional optimal transport plan on $\mathbb{R}$. In one dimension, the OT problem reduces to sorting, which is inherently non-differentiable. To enable gradient-based optimization, a differentiable variant of the $\STP$ framework was introduced in \cite{chapel2025differentiable}. The authors proposed the differentiable generalized sliced transport plan (DGSWP), which formulates the search for a sliced plan as a bilevel optimization problem and smooths the outer, non-differentiable $\STP$ objective by injecting Gaussian noise into the slice parameters and estimating gradients using Stein’s lemma. While this yields a differentiable surrogate, it suffers in practice from high variance when naive Monte Carlo estimators are used, as noted by the authors in \cite{chapel2025differentiable} and shown in Figure~\ref{fig:dgswp-vs-lapsum}. To mitigate this, they employed control variates as a variance-reduction technique to stabilize training. Nevertheless, the computational overhead of repeatedly solving one-dimensional OT problems for every perturbed slicer remains a significant bottleneck.

We adopt a different approach to inject differentiability, which aligns with the smooth surrogate $J_\eta$ proposed in Definition \ref{def:Jeta-lift} and avoids the extra cost of repeated 1-D OT problems. This is achieved via LapSum, the differentiable sorting method introduced in \cite{struski2025lapsum}. Let $x_1,\dots,x_n$ denote $n$ scalars in $\mathbb{R}$ with the corresponding empirical cumulative distribution function (CDF)
\vspace{-5pt}
\[
\hat F_x(t)\;=\;\frac{1}{n}\sum_{i=1}^n \mathbf{1}\{x_i\le t\}.
\]
Then the rank of $x_i$ and sorted $k$-th value can be written as 
\vspace{-3pt}
\begin{align*}
&r_i \;=\; n\,\hat F_x(x_i), \quad x_{(k)} \;=\; \hat F_x^{-1}\!\left(k/n\right),\ \ k=1,\dots,n.
\end{align*}
LapSum injects i.i.d.\ Laplace noise $\xi_i\sim\mathsf{Lap}(0,\alpha)$ and smooths each of the hard functions $\mathbf{1}\{x_i\le t\}$ as the CDF of $\mathsf{Lap}(0,\alpha)$, denoted by $\Phi_\alpha$, resulting in a smoothed CDF
$
\hat F_{x,\alpha}(t)\;=\;\frac{1}{n}\sum_{i=1}^n \Phi_\alpha(t-x_i)$. This yields differentiable surrogates for rank and sorting: $
\tilde r_i =n\hat F_{x,\alpha}(x_i)$ and
$\tilde x_{(k)}=\hat F_{x,\alpha}^{-1}\!\left(\frac{k}{n}\right)$,
with $\hat F_{x,\alpha}\!\to\!\hat F_x$ and $\tilde x_{(k)}\!\to\!\mathrm{sort}(x)$ as $\alpha\rightarrow 0$. The elegance of LapSum is that the inverse $\hat F_{x,\alpha}^{-1}$ can be calculated in closed-form and the algorithm has a complexity of only $\mathcal{O}(n\log n)$ (see Algorithm \ref{alg:lapsum} in the Appendix for deriving LapSum-based soft permutation).

In our smooth surrogate in Definition~\ref{def:lin-laplace}, since the noise $\xi\sim \mathsf{Lap}_d(\eta^2\Sigma)$, each projection $\langle \xi,x\rangle$ is a 1D Laplace distribution with scale $b_x=\eta\sqrt{\tfrac12\,x^\top\Sigma x}$ (see Lemma \ref{rem:proj-lap} in the Appendix). In other words, we inject Laplacian noise into each of the one-dimensional samples and calculate the expectation of plans in one shot by LapSum. This closely connects to the probabilistic assumptions of LapSum’s soft comparisons. In practice, we conveniently adopt a fixed global scale for the Laplace noise across all samples, which allows for absorbing the perturbations in the calculations of the 1D optimal transport plan using LapSum, notably with a time complexity of only $\mathcal{O}(n\log n)$ and $\mathcal{O}(n)$ memory requirement.

\vspace{-0.1in}
\subsection{Mini-batch Training}
\vspace{-0.1in}
Training a slicer network on large-scale datasets is computationally demanding because each update requires evaluating the full $\STP$ objective. We mitigate this cost by optimizing a mini-batched $\STP$ loss, which enables fast iterations. Theoretically, we establish a stability and convergence guarantee demonstrating how the mini-batch objective approaches the full-batch objective as the batch size becomes large.

Let $X=\{x_i\}_{i=1}^N$ and $Y=\{y_j\}_{j=1}^M$ be two sets of discrete samples, and define $[n]\coloneqq\{1,\dots,n\}$. The discrete $\STP$ objective is given by
$J_{N,M}(f)\coloneqq\sum_{i=1}^N\sum_{j=1}^M \gamma_{ij}\,c(x_i,y_j)^p$,
where $\gamma$ denotes the lifted transport plan by $f:\mathbb{R}^d\to\mathbb{R}$.
Fix a batch size $B\le \min(N, M)$ and let $\binom{[n]}{B}$ represent the set of all possible batches of $B$ indices drawn from $[n]$ without replacement. For $S\in\binom{[N]}{B},T\in\binom{[M]}{B}$ and their corresponding batches $x_S=(x_i)_{i\in S}$ and $y_T=(y_j)_{j\in T}$, define the kernel with slicer $f:\mathbb{R}^d\to\mathbb{R}$ as
$
h_B\!\left(f; x_S,\, y_T\right)
\coloneqq \frac{1}{B}\sum_{i=1}^B c \big(x^{f}_{(i)},\,y^{f}_{(i)}\big)^p,
$
where $x^{f}_{(i)},\,y^{f}_{(i)}\in\mathbb{R}^d$ denote the re-ordered $x_S$ and $y_T$, paired by the lifted coupling from 1D coupling, i.e., $f(x^{f}_{(1)})\le \cdots \le f(x^{f}_{(B)})$ and $f(y^{f}_{(1)})\le \cdots \le f(y^{f}_{(B)})$ for the slicer $f:\mathbb{R}^d\to\mathbb{R}$. The mini-batch $\STP$ loss is then defined as 
\vspace{-0.1in}
\[
J_{N, M, B}(f)
\coloneqq \frac{1}{\binom{N}{B}\binom{M}{B}}
\sum_{S\in\binom{[N]}{B}}\;
\sum_{T\in\binom{[M]}{B}}
h_B\!\big(f;\, x_S,\, y_T\big).
\]
\vspace{-0.1in}

In practical implementations, the incomplete estimator draws $K$ i.i.d. batch pairs $\{(S_k,T_k)\}_{k=1}^K$ uniformly from 
$\binom{[N]}{B}\times\binom{[M]}{B}$, and averages:
\vspace{-0.1in}
\[
\overline{J}_{B,K}(f)
\coloneqq \frac{1}{K}\sum_{k=1}^K h_B\!\big(f;\,x_{S_k},\,y_{T_k}\big).
\]
\vspace{-0.1in}

\begin{proposition}
For any $\delta\in(0,1)$, with probability at least $1-\delta$, there exist constants $C_1, C_2, C_3, R$ (which depend on the datasets and $f$) such that
\vspace{-0.1in}
\begin{gather*}
    \bigl|\overline{J}_{B,K}(f)-J_{N,M}(f)\bigr|
\;\le\; \frac{C_1}{B}
\;+\;C_2\omega_X(\tfrac12\sqrt{\tfrac{N-B}{NB}})\\
\;+\;C_3\omega_Y(\tfrac12\sqrt{\tfrac{M-B}{MB}})
\;+\;R\,\sqrt{\frac{\log(2/\delta)}{2K}}.
\end{gather*}
where $\omega_X,\omega_Y$ denote the quantile moduli of $f_\#\mu_N$ and $f_\#\nu_M$, which are non-decreasing, and satisfy $\omega_X(0)=\omega_Y(0)=0$. 
% If we further assume the CDFs of $f_\#\mu_N$ and $f_\#\nu_M$ are smoothed and admit positive density lower bounds, then
% \[
% \bigl|\overline{J}_{B,K}(f)-J_{N,M}(f)\bigr|
% % \;\le\mathcal{O}(B^{-1})+
% \;\le\mathcal{O}(B^{-1/2})+\mathcal{O}(K^{-1/2})+\varepsilon_{\mathrm{approx}}.
% \]
\end{proposition}

\begin{proof}[Proof sketch]
The proof proceeds in two steps: (i) controlling the mini-batch bias $\bigl|J_{N,M,B}(f)-J_{N,M}(f)\bigr|$ and (ii) controlling the Monte Carlo fluctuation over $K$ batches $\bigl|\overline{J}_{B,K}(f)-J_{N,M,B}(f)\bigr|$. The first term is deterministically bounded by a $\mathcal{O}(B^{-1})$ term from well-stratified batches, together with the quantile moduli
$\omega_X, \omega_Y$ capturing deviations from this ideal stratification. The second term utilizes the fact that $\mathbb{E}\overline{J}_{B,K}(f)=J_{N,M,B}(f)$. By the boundedness of finite discrete samples, applying Hoeffding's inequality gives the probabilistic bound $\mathcal{O}(K^{-1/2})$. 
% With additional regularity assumptions, $\omega_X$ and $\omega_Y$ are Lipschitz, which sharpens the bound. 
Details can be found in Propositions \ref{prop:incomplete} and \ref{prop:mb-vs-full-bilip} in the Appendix.
\end{proof}

\vspace{-.2in}
\section{Empirical Results}
\vspace{-.1in}
% \vspace{-.1in}

\subsection{Implementations}
\vspace{-.1in}
We train the slicer with symmetric two-branch gradient flows (see Figure \ref{fig:symmetric_train}). 
For the projected samples $\{f(x_i)\}_{i=1}^N$ and $\{f(y_j)\}_{j=1}^M$, the differentiable LapSum operator produces soft permutation matrices $\tilde P_X,\tilde P_Y$. 
Let $P_X,P_Y$ be the corresponding hard permutations and let $T_{N,M}\in\mathbb{R}^{N\times M}$ be the fixed optimal transport plan between one-dimensional sorted $N$ points and sorted $M$ points.
As differentiating through both soft permutations simultaneously may amplify noise and result in high variance in training, we form two plans during backpropagation,
\[
\gamma_{1}=\tilde P_X^{\!\top}\,T_{N,M}\,P_Y,
\qquad
\gamma_{2}=P_X^{\top}\,T_{N,M}\,\tilde P_Y,
\]
and optimize the loss using their average, $\tilde\gamma \;:=\; \tfrac12\big(\gamma_{1}+\gamma_{2}\big)$, for more stable optimization. %Algorithm \ref{alg:minstp-singlepair-gd} shows the steps of the mini-batch training. 
We empirically test the training scheme by comparing the transport plans produced by different training strategies, shown in Figure \ref{fig:mini-batch}. The visualizations illustrate that the mini-batch $\mSTP$ maintains high-quality matching even for smaller batch sizes, closely resembling both the optimal transport plan and the full-batch $\mSTP$ solution.

\begin{figure*}[t]
    \centering
    \vspace{-.4in}
    \includegraphics[width=\linewidth]{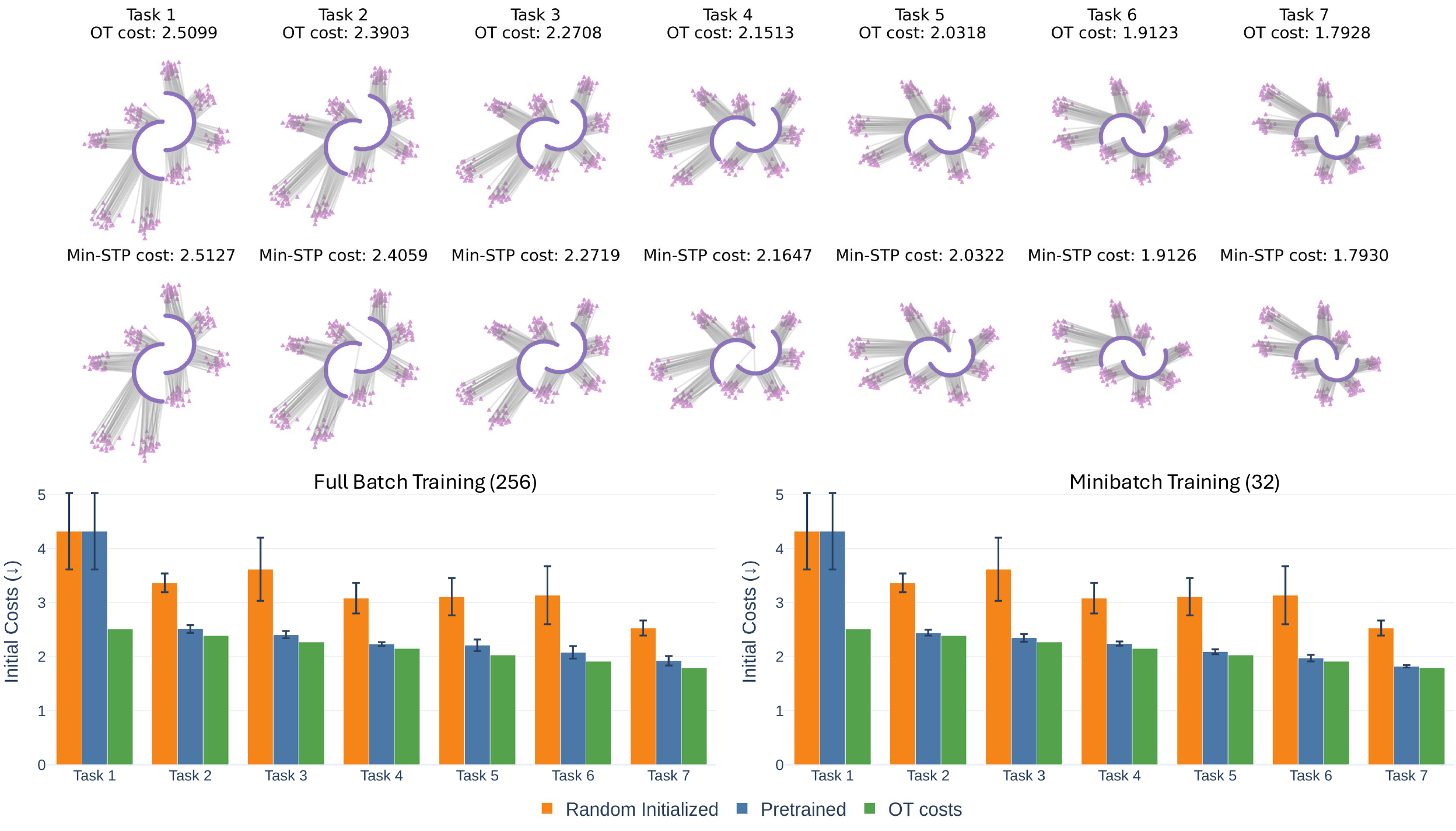}
    \vspace{-.2in}
    \caption{(Top row) Generated tasks $\{(\mu_t,\nu_t)\}_{t=1}^{7}$ with OT plans and corresponding costs. (Middle row) $\mSTP$ plans and costs optimized with a generalized slicer (i.e., a neural network).
      (Bottom row) \textit{Initial} slicer costs (mean over 5 runs) for a randomly initialized neural network, the optimal slicer for previous task, and the OT lower bound for full-batch (left) and mini-batch (right) training.
}
    \vspace{-.2in}
    \label{fig:exp1}
\end{figure*}
\vspace{-.1in}
\subsection{Transferability under Gradual Drift}
\label{exp:1}
\vspace{-.1in}

We evaluate slicer transferability on a sequence of gradually changing distribution pairs \(\{(\mu_t,\nu_t)\}_{t=1}^{7}\) generated from a fixed two-moons plus eight-Gaussians template in \(\mathbb{R}^2\). Each distribution has 256 samples. In each step, we apply a small rotation and zoom-out to produce the next distribution pair. For every task, we use a Set-Transformer \cite{lee2019set} as the parameterized slicer to minimize the $\STP$ objective. For $t\ge 2$, we compare initial values of the $\STP$ objective under two initializations: (i) \textbf{random initializations} (fresh and random parameters), and (ii) \textbf{transferred / pretrained initializations}, which reuses the optimal slicer \(f_{t-1}^\star\) from the previous task \((\mu_{t-1}, \nu_{t-1})\). As a lower bound, we also compute the OT cost between \((\mu_t,\nu_t)\). Figure~\ref{fig:exp1} shows that in all steps, the pretrained initialization starts from a substantially lower objective than random and sits noticeably closer to the OT reference, indicating that leveraging \(f_{t-1}^\star\) places us at a much stronger starting point for the new task and the learned slicer captures geometry that persists under small rotations and scale changes.

\vspace{-.1in}
\subsection{Amortized Min-STP for Point Cloud Alignment}
\vspace{-.1in}

To further assess the transferability of $\mSTP$, we investigate whether a single amortized slicer can generalize across a family of related distribution pairs. We use the ModelNet10 dataset \cite{wu20153d}, which contains ten object categories (bathtub, bed, chair, desk, dresser, monitor, night\_stand, sofa, table, toilet), each represented as a point cloud of size 1024. For every ordered pair of categories, we construct training and test sets of distribution pairs by sampling point clouds from the corresponding classes. A single amortized slicer $f_a^\star$ is then learned by minimizing a global objective over all training pairs. The slicer is implemented as an MLP with 3 hidden layers of dimensions $[256,512,256]$. As input, we concatenate each 3D point with a context vector, obtained from a pretrained point cloud auto-encoder. More details can be found in Section \ref{app:details}. We then compare the transport costs induced by $f_a^\star$ with the exact OT costs and against per-pair $\mSTP$, $\STP$ with a random slicer, and Wasserstein Wormhole \cite{haviv2024wasserstein} as a global embedding from point clouds into a Euclidean space, trained to preserve pairwise Wasserstein distances.  
\begin{wraptable}{r}{0.55\textwidth}
\small
\centering

\begin{tabular}{lcc}
\toprule
Method & Train Corr. & Test Corr. \\
\midrule
$\mSTP$ (Amortized) & $0.907^{\pm 0.048}$ & $0.902^{\pm 0.054}$ \\
$\mSTP$ & $0.959^{\pm 0.046}$ & $0.947^{\pm 0.073}$ \\
$\STP$ (Random Slicer) & $0.702^{\pm 0.090}$ & $0.719^{\pm 0.113}$ \\
Wasserstein Wormhole & $0.831^{\pm 0.132}$ & $0.770^{\pm 0.177}$ \\
\bottomrule
\end{tabular}
\caption{Mean and standard deviation of Pearson correlations with exact OT over all ModelNet10 category pairs. Amortized $\mSTP$ matches per-pair $\mSTP$ and outperforms the baselines. Note that unlike $\STP$-based methods Wasserstein Wormhole does not produce a transportation plan.}
\vspace{-.2in}
\label{tab:amor}

\end{wraptable}
Table~\ref{tab:amor} reports the mean$\pm$std Pearson correlations with OT across category pairs. Figure~\ref{fig:ds} shows an example pair (desk/sofa), with additional pairs in Figure~\ref{fig:morepairs} (Supplementary). Amortized $\mSTP$ matches OT correlations comparable to per-pair $\mSTP$—an upper bound that learns a separate slicer per pair—while consistently outperforming the other baselines on both train and test pairs, indicating that a single slicer generalizes across distributions.

\begin{figure}[t!]
    \centering
    \includegraphics[width=\linewidth]{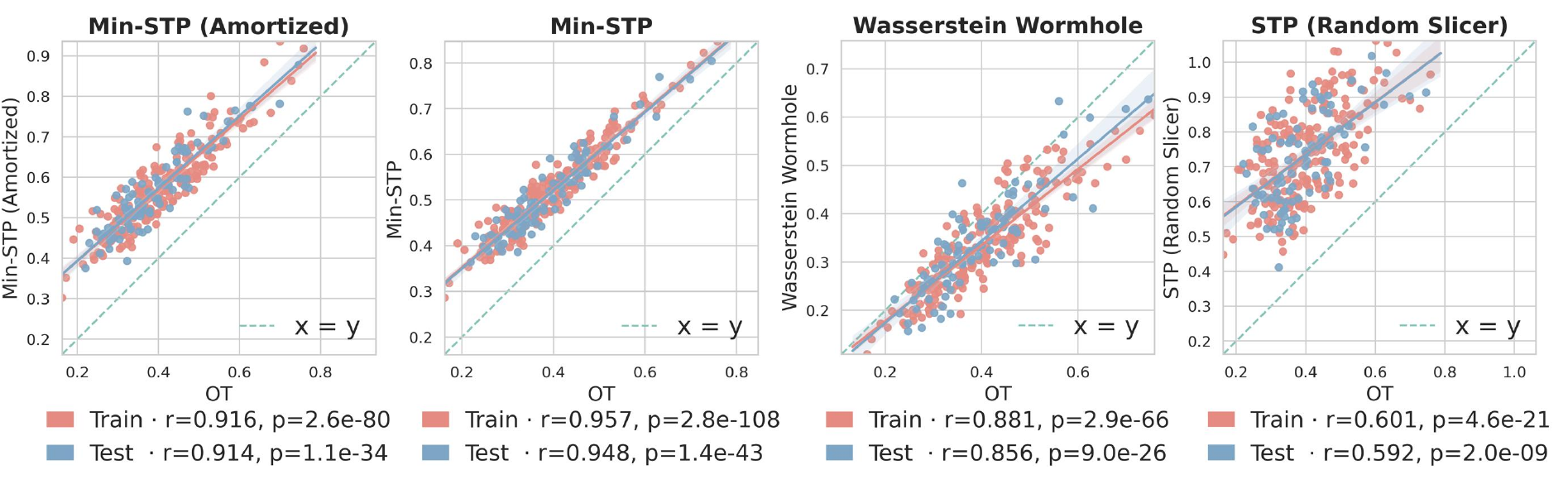}
    \vspace{-.2in}
    \caption{Correlations for desk and sofa pairs in ModelNet10. Amortized $\mSTP$ matches the performance of per‑pair $\mSTP$ and outperforms the two other baselines.}
    \label{fig:ds}
    \vspace{-.1in}
\end{figure}

\subsection{Min-STP based Flow}
\vspace{-.1in}

A promising application of optimal transport is its use in improving flow-based generative modeling \cite{tong2024improving} and single-step generation \cite{kornilov2024optimal}. OT-MF \cite{akbari2026transportbasedmeanflows} demonstrates that incorporating optimal transport into the MeanFlow algorithm \cite{geng2025meanflows} can enhance single-step generation 
performance across various tasks. 

Motivated by these findings, we aim to investigate the effectiveness of $\mSTP$ in a flow-based setting. Specifically, we adopt the experimental setup and datasets from \cite{akbari2026transportbasedmeanflows} for point cloud generation, using the “Chair” class from the ShapeNet dataset \cite{chang2015shapenet}. Additional results and implementation details are available in section \ref{subsec:flow}.

\begin{wrapfigure}{r}{0.7\textwidth}
    \centering
    \vspace{-.1in}
    \includegraphics[width=\linewidth]{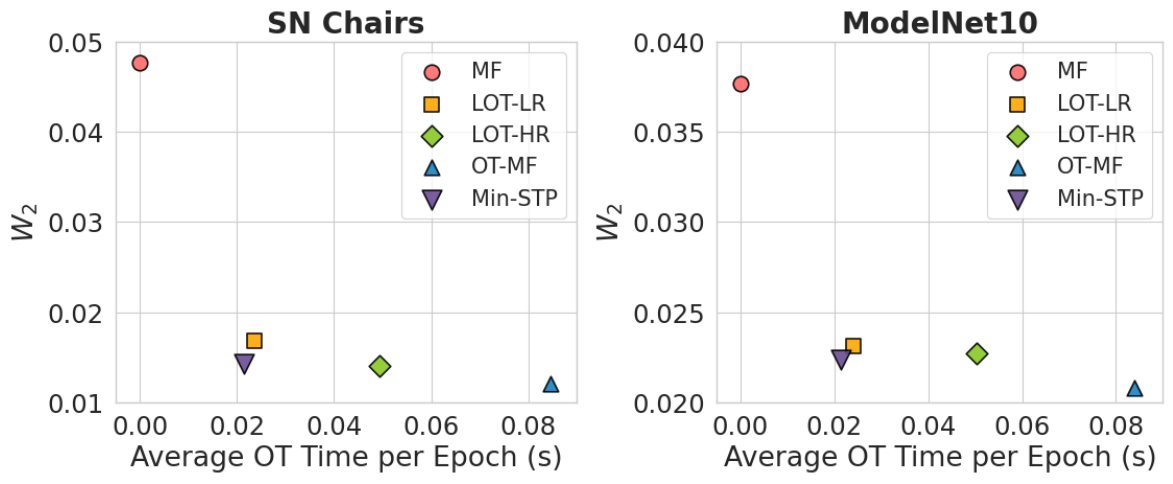}
    \vspace{-.2in}
    \caption{Average OT compute time per epoch vs. the $\OT_2 (W_2)$ distance for each method. For both metrics, lower is better.}
    \vspace{-0.1in}
    \label{fig:ot_time}
\end{wrapfigure}

% \begin{figure}[H]
%     \centering
%     \includegraphics[width=.9\linewidth]{cvpr2026/figures/corr_final_table_bathtub.png}
%     \caption{Correlations for table and bathtub pairs in ModelNet10. Amortized min‑STP matches the performance of per‑pair min‑STP and clearly outperforms the two baselines, highlighting its transferability across pairs.}
%     \label{fig:tt}
% \end{figure}

Figure \ref{fig:ot_time} shows the average time required to compute OT per epoch vs. the 2-Wasserstein distance ($W_2$) for $\mSTP$ and the baseline methods. A lower average $W_2$ across test samples shows the ability of the method to generate point cloud samples that are closest to the ground truth. An ideal method should also achieve a lower computation time for OT to reduce the overall training cost. As shown in Figure \ref{fig:ot_time}, $\mSTP$ appears in the bottom-left region of the plot, demonstrating its ability to balance low computation time with low $W_2$, resulting in higher-quality generation with faster training. Figure \ref{fig:sn_plots} (in the appendix) shows generated samples (1-step) for ShapeNet-Chairs. $\mSTP$ produces samples that more closely resemble the ground truth while adding only a slight overhead during training.

\vspace{-.1in}
\begin{wrapfigure}{r}{0.52\textwidth}
\centering
\vspace{-0.15in}
\begin{subfigure}{\linewidth}
    \centering
    \vspace{-1.2em}
    \includegraphics[width=\linewidth]{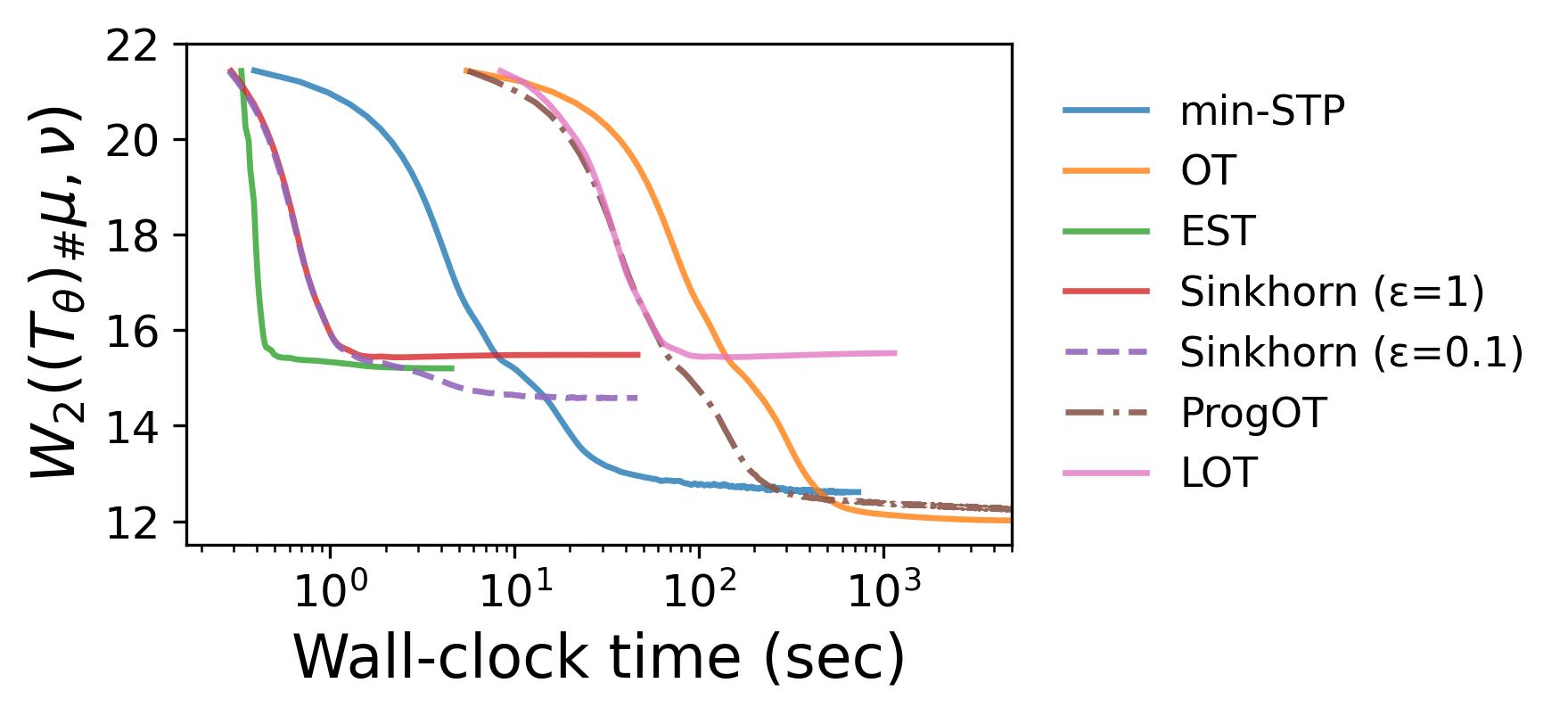}
    \vspace{-0.25in}
    \caption{Adult-to-Child translation performance ($W_2$) against wall-clock training time.}
    \label{fig:w2_vs_time}

\end{subfigure}
\hfill
\vspace{0.05in}
\begin{subfigure}{\linewidth}
        \centering
    \includegraphics[width=\linewidth]{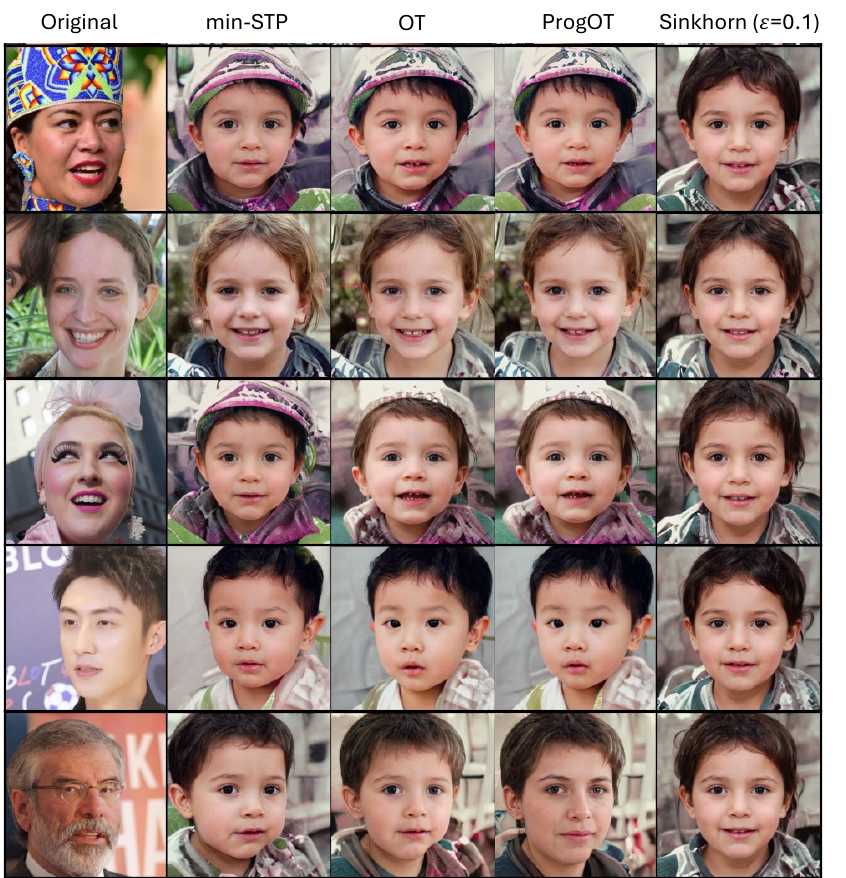}

    \caption{Adult-to-Child translation visualizations with different methods.}
    \label{fig:alae_plots}
    \vspace{-0.05in}
\end{subfigure}
\caption{Adult-to-Child translation results.}

\label{fig:combined}
\end{wrapfigure}

\subsection{Unpaired Image-to-Image Translation}
% \begin{wrapfigure}{r}{0.55\textwidth}
%     \centering
%     \vspace{-0.8em}
%     \includegraphics[width=\linewidth]{icml2026/w2_vs_time_logx.png}
%     \caption{Adult-to-Child translation evaluation against wall-clock training time.}
%     \label{fig:w2_vs_time}
%     \vspace{-1.0em}
% \end{wrapfigure}

% \vspace{-.1in}
In the unpaired image-to-image translation setting, we aim to learn a mapping from a source distribution $\mu$ to a target distribution $\nu$.
Here we use the face images from the FFHQ dataset~\cite{karras2019style} and work with 512-dimensional latent vectors encoded by ALAE \cite{pidhorskyi2020adversarial}. Given a transport plan $\gamma$ between distributions $\mu$ and $\nu$ over the latent vectors, we train an MLP
$T_\theta$,
\vspace{-0.1in}
\begin{equation}
\min_\theta \sum_{i,j} \left\| T_\theta(x_i) - y_j \right\|_2^2 \, \gamma_{i,j},
\end{equation}
where $\{x_i\} \sim \mu$ and $\{y_j\} \sim \nu$.
Training is performed in a mini-batch manner: at each iteration, we sample batches from $\mu$ and $\nu$, compute a transport plan $\gamma$ between the two batches, and update $\theta$ using the corresponding batch-wise objective. We consider several transport plans, including $\mSTP$, entropic optimal transport (Sinkhorn) \cite{cuturi2013sinkhorn}, Expected Transport Plan (EST) \cite{liu2025expected}, Progressive Entropic OT (ProgOT) \cite{kassraie2024progressive}, and Low-Rank OT (LOT) \cite{scetbon2021low}. Since a transport plan is recomputed at every training iteration, this setting requires repeatedly solving transport problems on different mini-batches. Among the methods considered, only $\mSTP$ admits transferable transport structure across batches, whereas the other approaches must be recomputed independently for each batch. For each method, we evaluate the resulting trained mapping
$T_{\theta^*}$ using the $2$--Wasserstein distance
$W_2\big( (T_{\theta^*})_\# \mu, \nu \big)$.
The Adult-to-Child translation results are reported in Figure~\ref{fig:w2_vs_time} and Figure~\ref{fig:alae_plots}.

% \begin{figure}[t]
%     \centering
%     \vspace{-0.15in}
%     \includegraphics[width=0.78\linewidth]{cvpr2026/figures/ALAE_plots.pdf}
%     \vspace{-0.12in}
%     \caption{Adult-to-Child translation with different methods.}
%     \label{fig:alae_plots}
%     \vspace{-0.15in}
% \end{figure}

% % \begin{wraptable}[10]{r}{.4\textwidth}
% \begin{table}[t!]
% % \vspace{-3em}
% \footnotesize
% % \setlength{\tabcolsep}{5pt}
% \caption{$W_{2}$ and Average Train Time per epoch (TR) reported on \textit{ShapeNet (SN) Chairs} and \textit{ModelNet10}. Best values are bold, followed by second-bests denoted by underline.}
% \label{pointnet}
% \begin{tabular}{lcc|cc}
% \toprule
% & \multicolumn{2}{c}{\textit{SN Chairs}} & \multicolumn{2}{c}{\textit{ModelNet10}} \\
% \cmidrule(lr){2-3} \cmidrule(lr){4-5}
% Method & $W_{2}$ & TR(s) & $W_{2}$ & TR(s) \\
% \midrule
% MF      &   0.0477   & \textbf{16.32}  & 0.0377     & \textbf{23.41}   \\
% LOT-LR  &   0.0168  & 17.22  &  0.0231   &  24.82    \\
% LOT-HR  &   \underline{0.0141}  & 18.24  &  0.0227    &  26.52   \\
% OT-MF   &   \textbf{0.0121}  & 20.81  & \textbf{0.0208}  &  28.64  \\
% \hline
% \textbf{$\mSTP$}   &   0.0143  & \underline{16.40}  & \underline{0.0224}   &  \underline{24.72}  \\
% \bottomrule
% \end{tabular}
% \end{table}
% % \end{wraptable}

\section{Conclusion}

We studied the transferability of optimized slicers in the min-Sliced Transport Plan ($\mSTP$) framework. Leveraging smoothing via Laplace distributions, our theoretical results establish that optimal slicers are stable under small distributional shifts, hence enabling transfer across related source-target tasks. To improve scalability, we also introduced a mini-batch formulation of $\mSTP$ and provided guarantees for its approximation accuracy. We empirically showed that transferred slicers maintain strong performance in point cloud alignment, flow-based generative modeling, and image translation. Our work highlights the promise of geometry-aware transport methods that generalize across domains and opens new avenues for scalable, transferable OT in dynamic, multi-domain settings.

\bibliography{example_paper}

\begin{thebibliography}{10}

\bibitem{akbari2026transportbasedmeanflows}
Elaheh Akbari, Ping He, Ahmadreza Moradipari, Yikun Bai, and Soheil Kolouri.
\newblock Transport based mean flows for generative modeling, 2026.

\bibitem{aliprantis2006infinite}
Charalambos~D Aliprantis and Kim~C Border.
\newblock {\em Infinite dimensional analysis: a hitchhiker’s guide}.
\newblock Springer, 2006.

\bibitem{altschuler2017near}
Jason Altschuler, Jonathan Weed, and Philippe Rigollet.
\newblock Near-linear time approximation algorithms for optimal transport via sinkhorn iteration.
\newblock {\em Advances in Neural Information Processing Systems}, 2017.

\bibitem{arjovsky2017wasserstein}
Martin Arjovsky, Soumith Chintala, and L{\'e}on Bottou.
\newblock {Wasserstein} generative adversarial networks.
\newblock In {\em International conference on machine learning}, pages 214--223. PMLR, 2017.

\bibitem{bai2022sliced}
Yikun Bai, Bernhard Schmitzer, Matthew Thorpe, and Soheil Kolouri.
\newblock Sliced optimal partial transport.
\newblock In {\em Proceedings of the IEEE/CVF Conference on Computer Vision and Pattern Recognition}, pages 13681--13690, 2023.

\bibitem{bonneel2015sliced}
Nicolas Bonneel, Julien Rabin, Gabriel Peyr{\'e}, and Hanspeter Pfister.
\newblock Sliced and {Radon} {Wasserstein} barycenters of measures.
\newblock {\em Journal of Mathematical Imaging and Vision}, 51(1):22--45, 2015.

\bibitem{bonnotte2013unidimensional}
Nicolas Bonnotte.
\newblock {\em Unidimensional and evolution methods for optimal transportation}.
\newblock PhD thesis, Universit{\'e} Paris Sud-Paris XI; Scuola normale superiore (Pise, Italie), 2013.

\bibitem{chang2015shapenet}
Angel~X. Chang, Thomas Funkhouser, Leonidas Guibas, Pat Hanrahan, Qixing Huang, Zimo Li, Silvio Savarese, Manolis Savva, Shuran Song, Hao Su, Jianxiong Xiao, Li~Yi, and Fisher Yu.
\newblock Shapenet: An information-rich 3d model repository.
\newblock {\em arXiv preprint arXiv:1512.03012}, 2015.

\bibitem{chapel2025differentiable}
Laetitia Chapel, Romain Tavenard, and Samuel Vaiter.
\newblock Differentiable generalized sliced {Wasserstein} plans.
\newblock In {\em The Thirty-ninth Annual Conference on Neural Information Processing Systems}, 2025.

\bibitem{chen2020augmented}
Xiongjie Chen, Yongxin Yang, and Yunpeng Li.
\newblock Augmented sliced wasserstein distances.
\newblock {\em arXiv preprint arXiv:2006.08812}, 2020.

\bibitem{courty2016optimal}
Nicolas Courty, R{\'e}mi Flamary, Devis Tuia, and Alain Rakotomamonjy.
\newblock Optimal transport for domain adaptation.
\newblock {\em IEEE transactions on pattern analysis and machine intelligence}, 39(9):1853--1865, 2016.

\bibitem{cuturi2013sinkhorn}
Marco Cuturi.
\newblock Sinkhorn distances: Lightspeed computation of optimal transport.
\newblock {\em Advances in neural information processing systems}, 26:2292--2300, 2013.

\bibitem{damodaran2018deepjdot}
Bharath~Bhushan Damodaran, Benjamin Kellenberger, R{\'e}mi Flamary, Devis Tuia, and Nicolas Courty.
\newblock Deepjdot: Deep joint distribution optimal transport for unsupervised domain adaptation.
\newblock In {\em Proceedings of the European conference on computer vision (ECCV)}, pages 447--463, 2018.

\bibitem{geng2025meanflows}
Zhengyang Geng, Mingyang Deng, Xingjian Bai, J.~Zico Kolter, and Kaiming He.
\newblock Mean flows for one-step generative modeling.
\newblock {\em arXiv preprint arXiv:2505.13447}, 2025.

\bibitem{gerber2017multiscale}
Samuel Gerber and Mauro Maggioni.
\newblock Multiscale strategies for computing optimal transport.
\newblock {\em Journal of Machine Learning Research}, 18(72):1--32, 2017.

\bibitem{gulrajani2017improved}
Ishaan Gulrajani, Faruk Ahmed, Martin Arjovsky, Vincent Dumoulin, and Aaron~C Courville.
\newblock Improved training of wasserstein gans.
\newblock {\em Advances in neural information processing systems}, 30, 2017.

\bibitem{guo2022online}
Wenxuan Guo, YoonHaeng Hur, Tengyuan Liang, and Chris Ryan.
\newblock Online learning to transport via the minimal selection principle.
\newblock In {\em Conference on Learning Theory}, pages 4085--4109. PMLR, 2022.

\bibitem{halmos2025hierarchical}
Peter Halmos, Julian Gold, Xinhao Liu, and Benjamin Raphael.
\newblock Hierarchical refinement: Optimal transport to infinity and beyond.
\newblock In {\em Forty-second International Conference on Machine Learning}, 2025.

\bibitem{haviv2024wasserstein}
Doron Haviv, Russell~Zhang Kunes, Thomas Dougherty, Cassandra Burdziak, Tal Nawy, Anna Gilbert, and Dana Pe’Er.
\newblock Wasserstein wormhole: Scalable optimal transport distance with transformers.
\newblock {\em ArXiv}, pages arXiv--2404, 2024.

\bibitem{huynh2020otlda}
Viet Huynh, He~Zhao, and Dinh Phung.
\newblock Otlda: A geometry-aware optimal transport approach for topic modeling.
\newblock {\em Advances in Neural Information Processing Systems}, 33:18573--18582, 2020.

\bibitem{karras2019style}
Tero Karras, Samuli Laine, and Timo Aila.
\newblock A style-based generator architecture for generative adversarial networks.
\newblock In {\em Proceedings of the IEEE/CVF conference on computer vision and pattern recognition}, pages 4401--4410, 2019.

\bibitem{kassraie2024progressive}
Parnian Kassraie, Aram-Alexandre Pooladian, Michal Klein, James Thornton, Jonathan Niles-Weed, and Marco Cuturi.
\newblock Progressive entropic optimal transport solvers.
\newblock {\em Advances in Neural Information Processing Systems}, 37:19561--19590, 2024.

\bibitem{khamis2024scalable}
Abdelwahed Khamis, Russell Tsuchida, Mohamed Tarek, Vivien Rolland, and Lars Petersson.
\newblock Scalable optimal transport methods in machine learning: A contemporary survey.
\newblock {\em IEEE transactions on pattern analysis and machine intelligence}, 2024.

\bibitem{kolouri2022generalized}
Soheil Kolouri, Kimia Nadjahi, Shahin Shahrampour, and Umut {\c{S}}im{\c{s}}ekli.
\newblock Generalized sliced probability metrics.
\newblock In {\em ICASSP 2022-2022 IEEE International Conference on Acoustics, Speech and Signal Processing (ICASSP)}, pages 4513--4517. IEEE, 2022.

\bibitem{kolouri2019generalized}
Soheil Kolouri, Kimia Nadjahi, Umut Simsekli, Roland Badeau, and Gustavo Rohde.
\newblock Generalized sliced {W}asserstein distances.
\newblock In H.~Wallach, H.~Larochelle, A.~Beygelzimer, F.~d\textquotesingle Alch\'{e}-Buc, E.~Fox, and R.~Garnett, editors, {\em Advances in Neural Information Processing Systems}, volume~32. Curran Associates, Inc., 2019.

\bibitem{kornilov2024optimal}
Nikita Kornilov, Petr Mokrov, Alexander Gasnikov, and Alexander Korotin.
\newblock Optimal flow matching: Learning straight trajectories in just one step.
\newblock In A.~Globerson, L.~Mackey, D.~Belgrave, A.~Fan, U.~Paquet, J.~Tomczak, and C.~Zhang, editors, {\em Advances in Neural Information Processing Systems}, volume~37, pages 104180--104204. Curran Associates, Inc., 2024.

\bibitem{lee2002memory}
Albert~K Lee and Matthew~A Wilson.
\newblock Memory of sequential experience in the hippocampus during slow wave sleep.
\newblock {\em Neuron}, 36(6):1183--1194, 2002.

\bibitem{lee2019set}
Juho Lee, Yoonho Lee, Jungtaek Kim, Adam Kosiorek, Seungjin Choi, and Yee~Whye Teh.
\newblock Set transformer: A framework for attention-based permutation-invariant neural networks.
\newblock In {\em International conference on machine learning}, pages 3744--3753. PMLR, 2019.

\bibitem{liu2025expected}
Xinran Liu, Rocio~Diaz Martin, Yikun Bai, Ashkan Shahbazi, Matthew Thorpe, Akram Aldroubi, and Soheil Kolouri.
\newblock Expected sliced transport plans.
\newblock In {\em The Thirteenth International Conference on Learning Representations}, 2025.

\bibitem{mahey2023fast}
Guillaume Mahey, Laetitia Chapel, Gilles Gasso, Cl{\'e}ment Bonet, and Nicolas Courty.
\newblock Fast optimal transport through sliced generalized wasserstein geodesics.
\newblock {\em Advances in Neural Information Processing Systems}, 36:35350--35385, 2023.

\bibitem{nadjahi2020statistical}
Kimia Nadjahi, Alain Durmus, L{\'e}na{\"\i}c Chizat, Soheil Kolouri, Shahin Shahrampour, and Umut Simsekli.
\newblock Statistical and topological properties of sliced probability divergences.
\newblock {\em Advances in Neural Information Processing Systems}, 33:20802--20812, 2020.

\bibitem{nguyen2023energybased}
Khai Nguyen and Nhat Ho.
\newblock Energy-based sliced wasserstein distance.
\newblock In {\em Thirty-seventh Conference on Neural Information Processing Systems}, 2023.

\bibitem{nguyen2021distributional}
Khai Nguyen, Nhat Ho, Tung Pham, and Hung Bui.
\newblock Distributional sliced-wasserstein and applications to generative modeling.
\newblock In {\em International Conference on Learning Representations}, 2021.

\bibitem{nguyen2024sliced}
Khai Nguyen, Shujian Zhang, Tam Le, and Nhat Ho.
\newblock Sliced wasserstein with random-path projecting directions.
\newblock In {\em International Conference on Machine Learning}, pages 37879--37899. PMLR, 2024.

\bibitem{peyre2019computational}
Gabriel Peyr{\'e} and Marco Cuturi.
\newblock Computational optimal transport: With applications to data science.
\newblock {\em Foundations and Trends in Machine Learning}, 11(5-6):355--607, 2019.

\bibitem{pidhorskyi2020adversarial}
Stanislav Pidhorskyi, Donald~A Adjeroh, and Gianfranco Doretto.
\newblock Adversarial latent autoencoders.
\newblock In {\em Proceedings of the IEEE/CVF conference on computer vision and pattern recognition}, pages 14104--14113, 2020.

\bibitem{rabin2011wasserstein}
Julien Rabin, Gabriel Peyr{\'e}, Julie Delon, and Marc Bernot.
\newblock Wasserstein barycenter and its application to texture mixing.
\newblock In {\em International Conference on Scale Space and Variational Methods in Computer Vision}, pages 435--446. Springer, 2011.

\bibitem{rowland2019orthogonal}
Mark Rowland, Jiri Hron, Yunhao Tang, Krzysztof Choromanski, Tamas Sarlos, and Adrian Weller.
\newblock Orthogonal estimation of wasserstein distances.
\newblock In {\em The 22nd International Conference on Artificial Intelligence and Statistics}, pages 186--195. PMLR, 2019.

\bibitem{saleh2022bending}
Mahdi Saleh, Shun-Cheng Wu, Luca Cosmo, Nassir Navab, Benjamin Busam, and Federico Tombari.
\newblock Bending graphs: Hierarchical shape matching using gated optimal transport.
\newblock In {\em Proceedings of the IEEE/CVF Conference on Computer Vision and Pattern Recognition}, pages 11757--11767, 2022.

\bibitem{scetbonlow}
Meyer Scetbon and Marco Cuturi.
\newblock Low-rank optimal transport: Approximation, statistics and debiasing.
\newblock {\em Advances in Neural Information Processing Systems}, 35:6802--6814, 2022.

\bibitem{scetbon2021low}
Meyer Scetbon, Marco Cuturi, and Gabriel Peyr{\'e}.
\newblock Low-rank sinkhorn factorization.
\newblock In {\em International Conference on Machine Learning}, pages 9344--9354. PMLR, 2021.

\bibitem{schiebinger2019optimal}
Geoffrey Schiebinger, Jian Shu, Marcin Tabaka, Brian Cleary, Vidya Subramanian, Aryeh Solomon, Joshua Gould, Siyan Liu, Stacie Lin, Peter Berube, et~al.
\newblock Optimal-transport analysis of single-cell gene expression identifies developmental trajectories in reprogramming.
\newblock {\em Cell}, 176(4):928--943, 2019.

\bibitem{schmitzer2016sparse}
Bernhard Schmitzer.
\newblock A sparse multiscale algorithm for dense optimal transport.
\newblock {\em Journal of Mathematical Imaging and Vision}, 56(2):238--259, 2016.

\bibitem{shen2021accurate}
Zhengyang Shen, Jean Feydy, Peirong Liu, Ariel~H Curiale, Ruben San Jose~Estepar, Raul San Jose~Estepar, and Marc Niethammer.
\newblock Accurate point cloud registration with robust optimal transport.
\newblock {\em Advances in Neural Information Processing Systems}, 34:5373--5389, 2021.

\bibitem{struski2025lapsum}
Lukasz Struski, Michal~B. Bednarczyk, Igor~T. Podolak, and Jacek Tabor.
\newblock Lapsum - one method to differentiate them all: Ranking, sorting and top-k selection.
\newblock In {\em Forty-second International Conference on Machine Learning}, 2025.

\bibitem{tanguy2025sliced}
Eloi Tanguy, Laetitia Chapel, and Julie Delon.
\newblock Sliced optimal transport plans.
\newblock {\em arXiv preprint arXiv:2508.01243}, 2025.

\bibitem{tong2024improving}
Alexander Tong, Kilian FATRAS, Nikolay Malkin, Guillaume Huguet, Yanlei Zhang, Jarrid Rector-Brooks, Guy Wolf, and Yoshua Bengio.
\newblock Improving and generalizing flow-based generative models with minibatch optimal transport.
\newblock {\em Transactions on Machine Learning Research}, 2024.
\newblock Expert Certification.

\bibitem{Villani2009Optimal}
Cedric Villani.
\newblock {\em Optimal transport: old and new}.
\newblock Springer, 2009.

\bibitem{wang2013linear}
Wei Wang, Dejan Slep{\v{c}}ev, Saurav Basu, John~A Ozolek, and Gustavo~K Rohde.
\newblock A linear optimal transportation framework for quantifying and visualizing variations in sets of images.
\newblock {\em International journal of computer vision}, 101(2):254--269, 2013.

\bibitem{wu20153d}
Zhirong Wu, Shuran Song, Aditya Khosla, Fisher Yu, Linguang Zhang, Xiaoou Tang, and Jianxiong Xiao.
\newblock 3d shapenets: A deep representation for volumetric shapes.
\newblock In {\em Proceedings of the IEEE conference on computer vision and pattern recognition}, pages 1912--1920, 2015.

\bibitem{yang2020predicting}
Karren~Dai Yang, Karthik Damodaran, Saradha Venkatachalapathy, Ali~C Soylemezoglu, GV~Shivashankar, and Caroline Uhler.
\newblock Predicting cell lineages using autoencoders and optimal transport.
\newblock {\em PLoS computational biology}, 16(4):e1007828, 2020.

\bibitem{yurochkin2019hierarchical}
Mikhail Yurochkin, Sebastian Claici, Edward Chien, Farzaneh Mirzazadeh, and Justin~M Solomon.
\newblock Hierarchical optimal transport for document representation.
\newblock {\em Advances in neural information processing systems}, 32, 2019.

\end{thebibliography}
\bibliographystyle{plain}

%%%%%%%%%%%%%%%%%%%%%%%%%%%%%%%%%%%%%%%%%%%%%%%%%%%%%%%%%%%%
\newpage
\appendix

\section{Transferability}
\label{sec:transfer}
% In this section, we discuss the transferability of the optimal slicer $f^\star:\mathcal{X}\rightarrow\mathbb{R}$ on a compact metric space $\mathcal{X}$. We show that if the input distributions are perturbed 
% slightly in the Wasserstein distance \(W_p\), a corresponding optimal slicer exhibits only minor deviations. This stability property ensures that \(f^\star\) can be effectively transferred as a prior for the perturbed distributions.

% \begin{assumption}[Model class and regularity]%\label{asmp:model}
% $\mathcal X$ is a compact metric space in $\mathbb{R}^d$, with $d(x, y) = \|x-y\|_p$ for $p\ge 1$, and $\mathcal F\subset{C(\mathcal X)}$ consists of all Lipschitz functions $f:\mathcal X\!\to\!\mathbb R$ with global Lipschitz constant $L>0$, and uniformly bounded by $M>0$, i.e. $\mathcal{F} = \{f\in C(\mathcal{X}): Lip(f)\le L, \|f\|_\infty\leq M\}$, where $Lip(f)$ denotes the Lipschitz constant of $f$.
% \end{assumption}
% All of the following results will be based on this assumption.

\noindent\textbf{1. Model assumption and properties.}

Under Assumption \ref{asmp:model}, we take $\mathcal{X}\subset \mathbb{R}^d$ to be a compact metric space. Since $\mathcal{X}$ is compact, all probability measures on $\mathcal{X}$ automatically have finite 
$p$-th moments. Therefore,
$$\mathcal{P}_p(\mathcal{X})=\mathcal{P}(\mathcal{X}).$$

\begin{lemma}
\label{rem: precompact} Consider the family $\mathcal{F}$ defined over the compact metric space $(\mathcal{X}\subset \mathbb R^d,d)$ under Assumption \ref{asmp:model}. Then $\mathcal F$ is compact in $(C(\mathcal X),\|\cdot\|_\infty)$.
\end{lemma}

\begin{proof}
For any $\varepsilon>0$, choose $\delta:=\varepsilon/L$. If $d(x, y)<\delta$, then for all $f\in\mathcal F$, we have $|f(x)-f(y)|\le Ld(x, y)<\varepsilon$. By uniform boundedness, $\|f\|_\infty<M$ for all $f\in\mathcal{F}$.
For each $x\in\mathcal X$, the set $\{f(x):f\in\mathcal F\}\subset[-M,M]$ is compact.
By Arzel\`a--Ascoli theorem on the compact space $\mathcal X$, $\mathcal F$ is precompact in $\|\cdot\|_\infty$ (i.e., its closure is compact in the ambient space). Moreover, $\mathcal F$ is closed as uniform limits preserve the Lipschitz constant $L$ and the sup bound $M$. Therefore, $\mathcal{F}$ is compact since $C(\mathcal X)$ is complete under the uniform topology (due to the fact that $\mathcal X$ is compact). %\textcolor{red}{That is, $\mathcal{F}$ is a compact subset of the ambient space $C(\mathcal{X})$, a complete but generally non-compact metric space.}
\end{proof}

\noindent\textbf{2. Perturbed version of STP.}

We recall that a 1D random variable has a 
Laplace distribution $\mathsf{Lap}_1(0,b)$ if its probability density function is
$\frac{1}{2b}e^{-\frac{|x|}{b}}$.
For the symmetric multivariate Laplace distribution $\mathsf{Lap}_d(0,\Sigma)$ in $\R^d$, where $\Sigma\in R^{d\times d}$ is a symmetric positive definite matrix, a typical characterization is given through its characteristic function: $\big(1+\tfrac12\,\,t^\top\Sigma t\big)^{-1}$, $t\in \R^d$.

\begin{lemma}[Noisy perturbations are Laplace]\label{rem:proj-lap}
Let $\xi\sim\mathsf{Lap}_d(\eta^2\Sigma)$.
For any fixed $x\in\mathbb{R}^d$,
\[
\langle \xi, x\rangle \ \sim\ \mathsf{Lap}_1\bigl(0,\ b_x\bigr),
\qquad
b_x=\eta\,\sqrt{\tfrac{1}{2}\,x^\top\Sigma x}.
\]
In particular, with $\Sigma=2I_d$ one gets
$\langle \xi,x\rangle\sim\mathsf{Lap}_1(0,\eta\|x\|_2)$.
\end{lemma}

\begin{proof}
The characteristic function of $\xi$ is
$\phi_\xi(t)=\big(1+\tfrac12\,\eta^2\,t^\top\Sigma t\big)^{-1}$, $t\in\mathbb{R}^d$.
For $Y:=\langle\xi,x\rangle$, $\phi_Y(t)=\phi_\xi(tx)
=\big(1+\tfrac12\,\eta^2\,t^2\,x^\top\Sigma x\big)^{-1}$,
which is the characteristic function of a univariate Laplace with 0 mean and scale
$b_x=\eta\sqrt{(x^\top\Sigma x)/2}$. 
\end{proof}

\begin{lemma}[Injectivity of projection]\label{lem:ae-inj}
Fix $\eta>0$, let $\mu\in\mathcal P(\mathcal X)$ be purely atomic with at most countable support,i.e. $\mu=\sum_{i\in I} a_i\,\delta_{x_i}$, $a_i>0$, $\sum_{i\in I} a_i=1,$ where $I$ is finite or countable and the support $\{x_i\}_{i\in I}\subset \mathcal X$. Let $f\in\mathcal F$, and let $g_{\xi_\eta}$ be as in Definition 3.1. Then for $\mathsf{Lap}_d(\eta^2\Sigma)$-a.e. $\xi_\eta$, the map $g_{\xi_\eta}$ is injective on $supp(\mu)=\{x_i:i\in I\}$. In particular, for $\mathsf{Lap}_d(\eta^2\Sigma)$-a.e. $\xi_\eta$, there exists a Borel set $A_{\xi_\eta}\subseteq \mathcal X$ with $\mu(A_{\xi_\eta})=1$ such that $g_{\xi_\eta}|_{A_{\xi_\eta}}:A_{\xi_\eta}\to g_{\xi_\eta}(A_{\xi_\eta})$ is injective.
\end{lemma}

\begin{proof}
For any two distinct support points $x_i\neq x_j$,
$$
g_{\xi_\eta}(x_i)=g_{\xi_\eta}(x_j) \iff \langle \xi_\eta, x_i-x_j\rangle = f(x_j)-f(x_i),
$$
which defines an affine hyperplane in $\mathbb R^d$ for $\xi_\eta$. Since the Laplace perturbation law is absolutely continuous with respect to Lebesgue measure, each such event has probability zero. Since there are only countably many distinct pairs of support points, the union of all such events still has probability zero. Hence, for almost every perturbation $\xi_\eta$, the values $\{g_{\xi_\eta}(x_i)\}_{i\in I}$ are all distinct, so $g_{\xi_\eta}$ is injective on $\operatorname{supp}(\mu)$. Taking $A_{\xi_\eta}=\operatorname{supp}(\mu),$ we obtain a Borel full-$\mu$-measure set on which $g_{\xi_\eta}$ is injective.

\end{proof}

\textbf{Unless stated otherwise, we assume the probability measures are discrete with finite, or countable \& bounded support in all of the following results.}

\begin{proposition}[Uniqueness of the lift]\label{prop:unique-lift}
Fix $\eta>0$, $\mu,\nu\in\mathcal P(\mathcal X)$ and $f\in\mathcal F$.
Fix $\xi_\eta$, $g_{\xi_\eta}$ as in Definition~\ref{def:lin-laplace} and set
\[
\alpha_{\xi_\eta}:=(g_{\xi_\eta})_\#\mu,\qquad \beta_{\xi_\eta}:=(g_{\xi_\eta})_\#\nu.
\]
Then for every $\sigma\in\Gamma(\alpha_{\xi_\eta},\beta_{\xi_\eta})$,
if $\tilde\gamma,\hat\gamma\in\Gamma(\mu,\nu)$ satisfy
\[
(g_{\xi_\eta},g_{\xi_\eta})_{\#}\tilde\gamma\ =\ \sigma\ =\ (g_{\xi_\eta},g_{\xi_\eta})_{\#}\hat\gamma,
\]
then $\tilde\gamma=\hat\gamma$.
\end{proposition}

\begin{proof}
Let $A_{\xi_\eta},B_{\xi_\eta}$ 
be Borel sets as in Lemma~\ref{lem:ae-inj} with $\mu(A_{\xi_\eta})=1=\nu(B_{\xi_\eta})$ so that $g_{\xi_\eta}|_{A_{\xi_\eta}}:A_{\xi_\eta}\!\to g_{\xi_\eta}(A_{\xi_\eta})$ and $g_{\xi_\eta}|_{B_{\xi_\eta}}:B_{\xi_\eta}\!\to g_{\xi_\eta}(B_{\xi_\eta})$ 
are bijections with inverses $j^{g_{\xi_\eta},\mu},j^{g_{\xi_\eta},\nu}$, respectively. By Lusin–Souslin theorem, such inverses are Borel measurable functions.

First note $\tilde\gamma(A_{\xi_\eta}\times B_{\xi_\eta})=1=\hat\gamma(A_{\xi_\eta}\times B_{\xi_\eta})$
because $\mu(A_{\xi_\eta})=\nu(B_{\xi_\eta})=1$.
Now, for any Borel $E\subset A_{\xi_\eta}\times B_{\xi_\eta}$, the map
\[
(g_{\xi_\eta},g_{\xi_\eta})|_{A_{\xi_\eta}\times B_{\xi_\eta}}:\ A_{\xi_\eta}\times B_{\xi_\eta}\ \longrightarrow\ g_{\xi_\eta}(A_{\xi_\eta})\times g_{\xi_\eta}(B_{\xi_\eta})
\]
is a Borel bijection with Borel inverse $\Phi:=(j^{g_{\xi_\eta},\mu},j^{g_{\xi_\eta},\nu})$.
Hence
\begin{align*}
\tilde\gamma(E) &=\tilde\gamma\big(\Phi\big((g_{\xi_\eta},g_{\xi_\eta})(E)\big)\big)=\sigma\big((g_{\xi_\eta},g_{\xi_\eta})(E)\big)\\
&=\hat\gamma\big(\Phi\big((g_{\xi_\eta},g_{\xi_\eta})(E)\big)\big)=\hat\gamma(E),
\end{align*}
where we used $(g_{\xi_\eta},g_{\xi_\eta})_{\#}\tilde\gamma=\sigma=(g_{\xi_\eta},g_{\xi_\eta})_{\#}\hat\gamma$.
Therefore, $\tilde\gamma$ and $\hat\gamma$ coincide on $A_{\xi_\eta}\times B_{\xi_\eta}$. Since $\tilde\gamma(A_{\xi_\eta}\times B_{\xi_\eta})=\hat\gamma(A_{\xi_\eta}\times B_{\xi_\eta})=1$, for any bounded Borel function $\varphi$ on $\mathcal X^2$,
\begin{align*}
\int\varphi\,\mathrm d\tilde\gamma
&=\int\varphi\,\mathbf 1_{A_{\xi_\eta}\times B_{\xi_\eta}}\,\mathrm d\tilde\gamma\\
&=\int\varphi\,\mathbf 1_{A_{\xi_\eta}\times B_{\xi_\eta}}\,\mathrm d\hat\gamma
=\int\varphi\,\mathrm d\hat\gamma.
\end{align*}
Thus, we conclude
$\tilde\gamma=\hat\gamma$.
\end{proof}

\noindent \textbf{Notation}: From now on, we will assume $$c(x,y)=\|x-y\|_p,$$ and for $\mu,\nu\in \mathcal{P}(\mathcal X)$ we write
$OT_p(\mu,\nu)=W_p(\mu,\nu)$ to refer to the $p$-Wasserstein
 distance defined in \ref{eq:ot}.

\begin{lemma}\label{lem:basic} For $f, g\in\mathcal{F}$ and $\kappa, \kappa'\in\mathcal{P}(\mathcal{X})$,
\begin{enumerate}
    \item[(1)] \(W_p\big(f_\#\kappa,\;f_\#\kappa'\big)\;\le\;L\,W_p(\kappa,\kappa')\)
    \item[(2)] \(W_p\big(f_\#\kappa,\;g_\#\kappa\big)\le\|f-g\|_\infty\)
\end{enumerate}

\end{lemma}
\begin{proof}
    (1) Let $\gamma\in\Gamma(\kappa,\kappa')$ be any coupling between $\kappa$ and $\kappa'$. Then $(f,f)_\#\gamma \in \Gamma(f_\#\kappa, f_\#\kappa')$, therefore
    \[W_p^p\!\big(f_\#\kappa,\,f_\#\kappa'\big) \;\le\;\int_{\mathbb{R}^2} |u-v|^p \, \mathrm d((f,f)_\#\gamma)(u,v);\]  
    and by $L$-Lipschitzness,
    \begin{align*}
    &\int_{\mathbb{R}^2} |u-v|^p \, \mathrm d((f,f)_\#\gamma)(u,v)\\
    &= \int_{\mathcal{X}^2} |f(x)-f(y)|^p \, \mathrm d\gamma(x,y)\\
    &\le L^p \int_{\mathcal{X}^2} c(x,y)^p \, \mathrm d\gamma(x,y).
    \end{align*}
    Since this holds for any coupling $\gamma\in\Gamma(\kappa, \kappa')$,
    \begin{align*}
    W_p^p\!\big(f_\#\kappa,\,f_\#\kappa'\big)&\leq L^p \inf_{\gamma\in\Gamma(\kappa, \kappa')}\int_{\mathcal{X}^2} c(x,y)^p \mathrm d\gamma(x,y)\\
    &= L^p\,W_p^p(\kappa,\kappa'),
    \end{align*}
    which proves the claim (1).

    (2)
    Consider the identity coupling $\gamma_{id} \in \Gamma(\kappa,\kappa)$. Then
    $(f,g)_\#\gamma_{id} \in \Gamma(f_\#\kappa, g_\#\kappa)$ and
    \begin{align*}
    W_p^p\!\big(f_\#\kappa,\,g_\#\kappa\big)
    &\le\int |f(x)-g(x)|^p\, \mathrm d\gamma_{id}(x, x)\\
    &= \int |f(x)-g(x)|^p\, \mathrm d\kappa(x)\leq\|f-g\|_\infty^p.
     \end{align*}
    Taking the $p$-th root gives the desired bound.
\end{proof}

 \begin{lemma}\label{lem:tools} For 1D probability measures $\mu_i,\nu_i\in\mathcal{P}(\mathbb R)$ with optimal couplings $\gamma_i$, $i=1, 2$,
\[
W_p\!\big(\gamma_1, \gamma_2\big)
\ \le\
\big( W_p(\mu_1,\mu_2)^p + W_p(\nu_1,\nu_2)^p \big)^{1/p},
\]
where the ground metric on $\mathbb R^2$ is defined as $\|(s,t)-(s',t')\|_p=(|s-s'|^p+|t-t'|^p)^{1/p}$ for $s, t, s', t'\in\mathbb{R}$.

\end{lemma}

\begin{proof}
Let $Q_{\mu_i}$ and $Q_{\nu_i}$ denote the quantile functions of $\mu_i, \nu_i$. Then the 1D $p$-Wasserstein distances between $\mu_i$ and $\nu_i$ are:
\begin{equation*}
    W_p^p(\mu_i, \nu_i) = \int_0^1|Q_{\mu_i}(t)-Q_{\nu_i}(t)|^p \mathrm dt ,
\end{equation*}
and the corresponding optimal transport plan $\gamma_i=(Q_{\mu_i}, Q_{\nu_i})_\#\lambda_{[0, 1]}$, where $\lambda_{[0,1]}$ is the Lebesgue measure on $[0, 1]$. Define $\widetilde\gamma:=(Q_{\mu_1}, Q_{\nu_1}, Q_{\mu_2}, Q_{\nu_2})_\# \lambda_{[0, 1]}\in\Gamma(\gamma_1, \gamma_2)$, then
\begin{align*}
    &W_p^p(\gamma_1, \gamma_2)\le \int_{(\mathbb{R}^2)^2}\|(u_1, v_1)-(u_2, v_2)\|_p^p \, \mathrm d\widetilde\gamma \\
    &= \int_{(\mathbb{R}^2)^2}|u_1-u_2|^p+|v_1- v_2|^p \, \mathrm d\widetilde\gamma\\
    %& = \int_{\mathbb{R}^4}\!|u_1-u_2|^p\!+\!|v_1- v_2|^pd(Q_{\mu_1}, Q_{\nu_1}, Q_{\mu_2}, Q_{\nu_2})_\# \lambda_{[0, 1]}\\
    &=\int_{[0, 1]} |Q_{\mu_1}(s)-Q_{\mu_2}(s)|^p+|Q_{\nu_1}(s)-Q_{\nu_2}(s)|^p \, \mathrm ds\\
    %&= \int_0^1 \!|Q_{\mu_1}(s)-Q_{\mu_2}(s)|^pds\!+\!\int_0^1\! |Q_{\nu_1}(s)-Q_{\nu_2}(s)|^pds\\
    &= W_p^p(\mu_1, \mu_2)+W_p^p(\nu_1, \nu_2).
\end{align*}
Taking $p$-th root gives the claim.
\end{proof}

\begin{proposition}[Recovery of $\mathrm{STP}$ as $\eta\to 0$ for injective slicer]\label{prop:jeta-to-stp}
Assume 
%$\mathcal X\subset\mathbb R^d$ is compact, $d(x,y)=\|x-y\|_p$, and 
$f\in\mathcal F$ and
\textbf{injective}. Let $\alpha:=f_{\#}\mu$, $\beta:=f_{\#}\nu$ and let $\sigma_f$
be the 1D optimal coupling between $\alpha$ and $\beta$.
Define the (unique) lift $\gamma_f\in \Gamma(\mu,\nu)$ such that
\[
%\gamma_f:=(j^f,j^f)_{\#}\sigma_f\ \in\ \Gamma(\mu,\nu),
%\qquad 
(f,f)_{\#}\gamma_f=\sigma_f,
\]
and recall 
\[
\text{STP}_p^p(\mu,\nu;f):=\int_{\mathcal X^2} c(x,y)^p\, \mathrm d\gamma_f(x,y).
\]
Let $J_\eta(\mu,\nu;f)$ be the smoothed objective from Definition~\ref{def:Jeta-lift},
then
\[
\lim_{\eta\to 0} J_\eta(\mu,\nu;f)\ =\ \text{STP}_p^p(\mu,\nu;f).
\]
\end{proposition}

\begin{proof}
Write $\xi_\eta=\eta Z$ with $Z\sim\mathsf{Lap}_d(\Sigma)$, so that $g_{\xi_\eta}(x)=f(x)+\eta\langle Z,x\rangle$.
For each instance $Z$, define the 1D marginals
\[
\alpha_{\eta,Z}:=(g_{\xi_\eta})_{\#}\mu,\qquad \beta_{\eta,Z}:=(g_{\xi_\eta})_{\#}\nu,
\]
and their optimal plan $\sigma_{\eta,Z}$.
Let $\gamma_{\eta,Z}$ be the (unique) lift of $\sigma_{\eta,Z}$ given by Proposition~\ref{prop:unique-lift}:
\[
(g_{\xi_\eta},g_{\xi_\eta})_{\#}\gamma_{\eta,Z}=\sigma_{\eta,Z},\qquad \gamma_{\eta,Z}\in\Gamma(\mu,\nu).
\]
By Definition~\ref{def:Jeta-lift}, $J_\eta(\mu,\nu;f)=\mathbb E_Z\big[\int c^p(x,y)\, \mathrm d\gamma_{\eta,Z}\big]$.
Since $\mathcal X$ is compact, let $R:=\sup_{x\in\mathcal X}\|x\|_p<\infty$,
and by Hölder's inequality,
\[
\|g_{\xi_\eta}-f\|_\infty\ \le\ \eta\,\|Z\|_q\,R;
\]
where $\frac{1}{p}+\frac{1}{q}=1$. By Lemma~\ref{lem:basic} part (2),
\begin{align*}
&W_p(\alpha_{\eta,Z},\alpha)\ \le\ \|g_{\xi_\eta}-f\|_\infty\ \le\ \eta\|Z\|_qR,\\
&\text{similarly, }W_p(\beta_{\eta,Z},\beta)\ \le\ \eta\|Z\|_qR.
\end{align*}
Hence, as $\eta\to 0$, we have $\alpha_{\eta,Z}\to \alpha$ and $\beta_{\eta,Z}\to \beta$ in the one-dimensional $W_p$-distance for each fixed $Z$.
By Lemma~\ref{lem:tools}, as $\eta\to 0$, we have
\[
W_p^p\big(\sigma_{\eta,Z},\sigma_f\big)\ \le\
W_p^p(\alpha_{\eta,Z},\alpha)+W_p^p(\beta_{\eta,Z},\beta)\ \rightarrow\ 0,
\]
for each fixed $Z$. 

Thus, $\sigma_{\eta,Z}\rightharpoonup\sigma_f$ weakly as $\eta\to 0$ (for each $Z$).

Let $\eta_n\to 0$ with corresponding $\{\gamma_{\eta_n,Z}\}_n$. Since $\mathcal X^2$ is compact, $\mathcal{P}(\mathcal{X}^2)$ is also compact, and so there exists a subsequence
$\gamma_{\eta_{n_m},Z}\rightharpoonup\gamma^\star$ for some probability measure $\gamma^\star$ on $\mathcal X^2$.
For any bounded $\varphi\in C(\mathbb R^2)$,
\begin{align*}
&\int\varphi\, \mathrm d(g_{\xi_{\eta_{n_m}}},g_{\xi_{\eta_{n_m}}})_{\#}\gamma_{\eta_{n_m},Z}\\
=&\int \varphi \big(g_{\xi_{\eta_{n_m}}}(x),g_{\xi_{\eta_{n_m}}}(y)\big)\, \mathrm d\gamma_{\eta_{n_m},Z}(x,y)\\
=&\int \varphi\, \mathrm d\sigma_{\eta_{n_m},Z}.
\end{align*}
Because $g_{\xi_{\eta_{n_m}}}\to f$ uniformly on $\mathcal X$ and $\gamma_{\eta_{n_m},Z}\rightharpoonup\gamma^\star$,
the LHS tends to $\int \varphi\big(f(x),f(y)\big)\, \mathrm d\gamma^\star(x,y)$.
The RHS tends to $\int\varphi\, \mathrm d\sigma_f$.
Hence,
\[
(f,f)_{\#}\gamma^\star=\sigma_f.
\]
Since $f$ is injective and continuous on compact $\mathcal X$, 
it is a continuous bijection from the compact space $\mathcal X$ onto $f(\mathcal X)\subset \mathbb R$ (Hausdorff), and so it is a homeomorphism: 
its inverse $j^f:f(\mathcal X)\to\mathcal X$ is continuous. Then, 
the lift that solves $(f,f)_{\#}\gamma=\sigma_f$ is \emph{unique} and equals $\gamma_f:=(j^f,j^f)_{\#}\sigma_f$.
In particular, $\gamma^\star=\gamma_f$. As the subsequence is arbitrary, we conclude that for \emph{each fixed $Z$},
\[
\gamma_{\eta,Z}\ \rightharpoonup\ \gamma_f,\qquad \text{as }\eta\to 0.
\]
Thus, since $c^p$ is bounded and continuous on the compact $\mathcal X^2$,
\[
\int c^p\, \mathrm d\gamma_{\eta,Z}\ \underset{\eta\to 0}{\longrightarrow}\ \int c^p\, \mathrm d\gamma_f
\quad\text{for each fixed }Z.
\]
Moreover,
\[
0\ \le\ \int c^p(x,y)\, \mathrm d\gamma_{\eta,Z}(x,y)\ \le\ (\mathrm{diam}(\mathcal X))^p\ <\infty,
\]
so by Dominated Convergence (with respect to $Z$),
\begin{align*}
J_\eta(\mu,\nu;f)
=\mathbb E_Z\!\Big[\int c^p\, \mathrm d\gamma_{\eta,Z}\Big]
\underset{\eta\to 0}{\longrightarrow}\
&\mathbb E_Z\!\Big[\int c^p\, \mathrm d\gamma_f\Big]\\
&=\int c^p\, \mathrm d\gamma_f \\
&=\text{STP}_p^p(\mu,\nu;f).
\end{align*}
\end{proof}

\noindent\textbf{3. Continuity of $J_\eta$.}

\begin{proposition}[Continuity of $J_\eta$]\label{prop:Jeta-cont}
Fix $\eta>0$ and let $J_\eta$ be given by Definition~\ref{def:Jeta-lift}.
If $(\mu_n,\nu_n,f_n)\to(\mu,\nu,f)$ with $W_p(\mu_n,\mu)\to0$, $W_p(\nu_n,\nu)\to0$ and $\|f_n-f\|_\infty\to0$ (as $n\to \infty$),
then
\[
J_\eta(\mu_n,\nu_n;f_n)\ \underset{n\to \infty}{\longrightarrow}\ J_\eta(\mu,\nu;f).
\]
\end{proposition}

\begin{proof}
We write $\xi_\eta=\eta Z$ with $Z\sim\mathsf{Lap}_d(\Sigma)$, and for each $n\in \mathbb N$ and $Z$ let
\[
g_{n,Z}(x):=f_n(x)+\eta\langle Z,x\rangle,\quad
g_{Z}(x):=f(x)+\eta\langle Z,x\rangle.
\]
Set the 1D pushforwards
\begin{gather*}
    \alpha_{n,Z}:=(g_{n,Z})_{\#}\mu_n,\qquad\beta_{n,Z}:=(g_{n,Z})_{\#}\nu_n,\\
\alpha_{Z}:=(g_{Z})_{\#}\mu,\qquad \beta_{Z}:=(g_{Z})_{\#}\nu,
\end{gather*}
and their optimal couplings
$\sigma_{n,Z}$ and $\sigma_{Z}$.
By Lemma~\ref{lem:ae-inj}, for $\mathbb P$–a.e.\ $Z$, the lifts
$\gamma_{n,Z}\in\Gamma(\mu_n,\nu_n)$ and $\gamma_{Z}\in\Gamma(\mu,\nu)$ are uniquely defined by
\[
(g_{n,Z},g_{n,Z})_{\#}\gamma_{n,Z}=\sigma_{n,Z},\qquad
(g_{Z},g_{Z})_{\#}\gamma_{Z}=\sigma_{Z}.
\]
By Definition~\ref{def:Jeta-lift},
\begin{gather*}
J_\eta(\mu_n,\nu_n;f_n)=\mathbb E_Z\!\left[\int_{\mathcal X^2} c(x,y)^p\, \mathrm d\gamma_{n,Z}(x,y)\right],\\
J_\eta(\mu,\nu;f)=\mathbb E_Z\!\left[\int c^p\, \mathrm d\gamma_{Z}\right].
\end{gather*}

\medskip\noindent

Since $f_n\to f$ uniformly and $x\mapsto\langle Z,x\rangle$ is $\|Z\|_q$–Lipschitz (with $1/p+1/q=1$),
we have $g_{n,Z}\to g_Z$ uniformly and $\mathrm{Lip}(g_{n,Z})\le L+\eta\|Z\|_q$.
By Lemma~\ref{lem:basic},
\begin{align*}
&W_p(\alpha_{n,Z},\alpha_Z)\\
&\le W_p\big((g_{n,Z})_{\#}\mu_n,(g_{n,Z})_{\#}\mu\big)+W_p\big((g_{n,Z})_{\#}\mu,(g_Z)_{\#}\mu\big)\\
&\le (L+\eta\|Z\|_q)\,W_p(\mu_n,\mu)+\|g_{n,Z}-g_Z\|_\infty \ \xrightarrow[n\to\infty]{}\ 0,
\end{align*}
and similarly $W_p(\beta_{n,Z},\beta_Z) \underset{n\to \infty}{\longrightarrow}0$.
By Lemma~\ref{lem:tools},
\begin{align*}
W_p\!\big(\sigma_{n,Z},\sigma_Z\big)
\le&\big( W_p(\alpha_{n,Z},\alpha_Z)^p+W_p(\beta_{n,Z},\beta_Z)^p\big)^{1/p}\\
&\xrightarrow[n\to\infty]{}\ 0
\qquad\text{for each fixed }Z.
\end{align*}

For each $n\in\mathbb N_0$ (with $n=0$ denoting $(\mu,\nu,f)$ and $n\ge1$ denoting $(\mu_n,\nu_n,f_n)$),
Lemma~\ref{lem:ae-inj} yields a set $G_n\subset\mathbb R^d$ with $\mathbb P(Z\in G_n)=1$
such that $g_{n,Z}$ is $\mu_n$–a.e.\ and $\nu_n$–a.e.\ injective for all $Z\in G_n$.
Let $G:=\bigcap_{n=0}^\infty G_n$; then $\mathbb P(Z\in G)=1$ and all lifts above are defined uniquely for $Z\in G$.

Fix $Z\in G$. Since $\mathcal X$ is compact, let $\gamma_{n_m,Z}\rightharpoonup\bar\gamma$ be a convergent subsequence.
Write $\Phi_n(x,y):=(g_{n,Z}(x),g_{n,Z}(y))$ and $\Phi(x,y):=(g_Z(x),g_Z(y))$.
Then $\|\Phi_n-\Phi\|_\infty\to 0$, and for any bounded function $\varphi\in C(\mathbb R^2)$,
\[
\int \varphi\, \mathrm d\sigma_{n_m,Z}
=\int \varphi\circ\Phi_{n_m}\, \mathrm d\gamma_{n_m,Z}
\ \xrightarrow[m\to\infty]{}
\int \varphi\circ\Phi\, \mathrm d\bar\gamma,
\]
while the LHS tends to $\int \varphi\, \mathrm d\sigma_Z$. Hence $(g_Z,g_Z)_{\#}\bar\gamma=\sigma_Z$.
Since $Z\in G$, $g_Z$ is $\mu$–a.e.\ and $\nu$–a.e.\ injective, so the lift of $\sigma_Z$ is \emph{unique};
thus $\bar\gamma=\gamma_Z$. As the subsequence is arbitrary, the whole sequence converges:
\[
\gamma_{n,Z}\ \rightharpoonup\ \gamma_Z\qquad(Z\in G).
\]
In particular, since $c^p$ is bounded and continuous on compact $\mathcal X^2$,
\[
\int c^p\, \mathrm d\gamma_{n,Z}\ \longrightarrow\ \int c^p\, \mathrm d\gamma_Z\qquad\text{for each }Z\in G.
\]
Moreover, $0\le \int c^p\, \mathrm d\gamma_{n,Z}\le (\mathrm{diam}\,\mathcal X)^p$ for all $n,Z$.
Hence, by dominated convergence (indeed, bounded convergence) with respect to $Z$,
\begin{align*}
J_\eta(\mu_n,\nu_n;f_n)
=&\mathbb E_Z\!\left[\int c^p\, \mathrm d\gamma_{n,Z}\right]\\
& \underset{n\to \infty}{\longrightarrow}\
\mathbb E_Z\!\left[\int c^p\, \mathrm d\gamma_Z\right]
=J_\eta(\mu,\nu;f).
\end{align*}
\end{proof}

\noindent\textbf{4. Compactness and application of the Berge's Maximum Theorem.}

We first review the Berge's Maximum Theorem \cite{aliprantis2006infinite}. %as a reference for the proof of Proposition \ref{prop:berge}. %and \ref{thm:w_f}.

% 14.1 Definition
\begin{definition}[Correspondence]
    A correspondence $\varphi$ from a set $X$ into a set $Y$ assigns to each $x\in X$ a subset $\varphi(x)$ of $Y$ (`point-to-set' assignment). We denote the correspondence as $\varphi: X \rightrightarrows Y$.
\end{definition}

Just as functions have inverses, so do correspondences. Here, we are only concerned with a ``strong inverse": For $\varphi: X \rightrightarrows Y$, the \emph{upper inverse} $\varphi^{u}$ or \emph{strong inverse} is
  defined by
  \[
    \varphi^{u}(A) \;=\; \{\, x\in X \mid  \varphi(x) \subset A \,\}.
  \]

Recall that a neighborhood of a set $A$ is any set $B$ for which there is an open set
$V$ satisfying $A \subset V \subset B$. Any open set $V$ that satisfies $A \subset V$
is called an \emph{open neighborhood} of $A$.

\begin{definition}[Upper Hemicontinuity]
A set-valued correspondence $\varphi: X \rightrightarrows Y$ between topological spaces is upper hemicontinuous at the point $x$ if for every open neighborhood
  $U$ of $\varphi(x)$, the upper inverse image $\varphi^{u}(U)$ is a neighborhood of $x$ in $X$. As with functions, we say $\varphi$ is upper hemicontinuous on $X$, abbreviated \emph{u.h.c.}, if it is upper hemicontinuous at every point of $X$.
\end{definition}

% 14.30 Berge's Maximum Theorem
\begin{theorem}[\textbf{Berge's Maximum Theorem.}]\label{thm:max}
Let $\varphi: X \rightrightarrows Y$ be a continuous correspondence with nonempty
compact values, and suppose $f: X \times Y \to \mathbb{R}$ is continuous. Define the ``value function'' $m: X \to \mathbb{R}$ by
\[
  m(x) \;:=\; \sup_{y \in \varphi(x)} f(x,y),
\]
and the correspondence $\mu: X \rightrightarrows Y$ of maximizers by
\[
  s(x) \;:=\; \{\, y \in \varphi(x) \mid f(x,y) = m(x) \,\}.
\]
Then, the value function $m$ is continuous, and the ``argmax'' correspondence
$\mu$ is upper hemicontinuous with non-empty and compact values. As a consequence, the 
$\sup$ may be replaced by 
$\max$.
\end{theorem}

\begin{proposition}\label{prop:berge}
Define the ``argmin" correspondence $\mathcal{S}:\mathcal{P}(\mathcal{X})\times\mathcal{P}(\mathcal{X})\to \mathcal{F}$
\[
\mathcal S(\mu,\nu):=\arg\min_{f\in \mathcal{F}} J_\eta(\mu,\nu;f),
\]
and the minimum value function $v:\mathcal{P}(\mathcal{X})\times\mathcal{P}(\mathcal{X})\to \mathbb R_{\geq 0}$
\[
v(\mu,\nu):=\min_{f\in \mathcal{F}}J_\eta(\mu,\nu;f).
\]
Then:
\begin{enumerate}
\item $v$ is continuous on $\Omega = (\mathcal P(\mathcal X)\times \mathcal P(\mathcal X), W_p+W_p)$;
\item for each $\mu,\nu\in \mathcal{P}(\mathcal{X})$, the set $\mathcal S(\mu,\nu)$ is nonempty and compact (in $\|\cdot\|_\infty$);
\item $\mathcal S$ is upper hemicontinuous (u.h.c.) at every $(\mu,\nu)$.
\end{enumerate}
\end{proposition}

\begin{proof}
By Proposition~\ref{prop:Jeta-cont}, the map
$J_\eta:\Omega\times \mathcal{F}\rightarrow \mathbb{R}$ is continuous on $\Omega\times \mathcal{F}$. As $\mathcal{F}$ is non-empty and compact, applying Berge’s Maximum Theorem \ref{thm:max} \cite{aliprantis2006infinite} with constant feasible correspondence $\varphi(\mu, \nu)\equiv\mathcal{F}$, the conclusion follows. Indeed, in the notation of Theorem \ref{thm:max}, consider $X=\mathcal{P}(\mathcal{X})\times\mathcal{P}(\mathcal{X})$, $Y=\mathcal{F}$, $f=-J_\eta$, and constant  $\varphi(\mu, \nu)=\mathcal{F}$ for each $(\mu,\nu)\in X$.
\end{proof}

\noindent\textbf{5. Transferability of the optimal slicer.}

\begin{theorem}
\label{thm:w_f}
    Given $\varepsilon>0$, there exists $\tau>0$ such that if $W_p(\mu_1, \mu_2)+W_p(\nu_1, \nu_2)<\tau$, then for every $ f_1^\star\in\arg\min_{f\in\mathcal F} J_\eta(\mu_1,\nu_1;f)$, there exists $ f_2^\star\in\arg\min_{f\in\mathcal F} J_\eta(\mu_2,\nu_2;f)$ in the $\varepsilon$ neighborhood of $f_1^\star$, i.e., $\|f_2^\star-f_1^\star\|_\infty<\varepsilon$.
\end{theorem}
\begin{proof}

Fix $(\mu_2,\nu_2)\in\Omega$ and $\varepsilon>0$. We will keep the notations in Proposition \ref{prop:berge} and denote $\mathcal S(\mu,\nu):=\arg\min_{f\in \mathcal{F}} J_\eta(\mu,\nu;f)$.
Consider the open neighborhood of $\mathcal S(\mu_2,\nu_2)$
\[
U_\varepsilon
:= \bigcup_{g\in \mathcal S(\mu_2,\nu_2)} \Big\{\, f\in \mathcal{F}:\ \|f-g\|_\infty < \varepsilon \,\Big\}.
\]
By upper hemicontinuity of $\mathcal S$ at $(\mu_2,\nu_2)$ (Proposition~\ref{prop:berge} and the
definition of u.h.c. for correspondences), there exists $\tau>0$ such that
\begin{align*}
d_\Omega\big((\mu_1,\nu_1),(\mu_2,\nu_2)\big)&=W_p(\mu_1, \mu_2)+W_p(\nu_1, \nu_2)<\tau\\
&\Longrightarrow\quad
\mathcal S(\mu_1,\nu_1)\subset U_\varepsilon .
\end{align*}
Hence for any $f_1^\star\in\mathcal S(\mu_1,\nu_1)$ there exists $f_2^\star\in\mathcal S(\mu_2,\nu_2)$
such that $\|f_2^\star-f_1^\star\|_\infty<\varepsilon$.

\end{proof}

\subsection{Quantitative $\varepsilon$--$\tau$ control.}\label{sec: tau eps}
We now quantify the $\varepsilon$--$\tau$ relation in Theorem~\ref{thm:w_f}. To do so, we make explicit how the smoothing scale $\eta$ and a bad-event with probability $\delta$ come into play. For simplicity, assume $p=2$ and $c(x,y)=\|x-y\|_2$.

 We expect that this analysis will motivate several directions for future research. Without appealing to the Berge’s Maximum Theorem, we will show that can essentially say the following: For $i=1,2$, let $$\mathcal{S}_i:=\arg\min_{f\in \mathcal F} J_\eta(\mu_i,\nu_i;f)$$ denote the set of minimizers, which are non-empty and compact under our assumptions. Given $\varepsilon>0$, let $U_\varepsilon(\mathcal{S}_i)$ denote the subset of functions in $\mathcal{F}$ whose distance to $\mathcal{S}_i$ is strictly less than $\varepsilon$. Consider the gaps $$m_\varepsilon^{(i)}\ :=\ \min_{f\in \mathcal{F}\setminus U_\varepsilon(\mathcal{S}_i)} J_\eta(\mu_i,\nu_i;f)\ -\ \min_{g\in\mathcal F}J_\eta(\mu_i,\nu_i;g).$$ Then if  $\mathcal{X}$ is finite, then with \emph{high probability} one can choose $\tau\leq m_\varepsilon^{(1)}/C$, for a constant $C$ depending on $\eta$, and obtain $\mathcal{S}_2\subset U_\varepsilon(\mathcal{S}_1)$. Similarly, if $\tau\leq m_\varepsilon^{(2)}/C$, then $\mathcal{S}_1\subset U_\varepsilon(\mathcal{S}_2)$, which in particular implies that for any $f_1^*\in \mathcal{S}_1$, there exists $f_2\in \mathcal{S}_2$ with $\|f_2^\star-f_1^\star\|_\infty<\varepsilon$. 

 Below we present the technical details, precisions, and future open directions.

\begin{assumption}\label{assump: extra}
Let $\Sigma=2 I_d$ and $Z\sim\mathsf{Lap}_d(\Sigma)$. We write $\xi_\eta=\eta Z$ and slicers $g_{\xi_\eta}(x)=f(x)+\eta\langle Z,x\rangle$ for $f\in \mathcal{F}$. 
To obtain a uniform inverse-Lipschitz bound for the slicers via a union bound, we will assume $\mathcal X$ is finite with $|\mathcal X|=B$; then there are $B'=\binom{B}{2}$ directions $u_j\in \mathbb S^{d-1}$ associated with the pairs of elements in $\mathcal{X}$.
\end{assumption}

\noindent \textbf{Step 1: A Lipschitz test for $d^2(\cdot,\cdot)$.} First, let $diam(\mathcal{X})=:D<\infty $ (since, in general, we assume $\mathcal{X}$ compact), and notice that the following inequalities
%\begin{small}
\begin{align*}
   &\left| \|x_1-y_1\|_2^2-\|x_2-y_2\|_2^2\right|\\
   &=\left|\langle (x_1-y_1) +(x_2-y_2) , (x_1-y_1) - (x_2-y_2)\rangle\right|\\
   &\leq\left(\|x_1-y_1\|_2+\|x_2-y_2\|_2\right)\|(x_1-y_1)-(x_2-y_2)\|_2\\
   &\leq 2D\left(\|x_1-x_2\|_2+\|y_1-y_2\|_2\right)\\ 
   &\leq \underbrace{2\sqrt{2}D}_{=:D'} \sqrt{\|x_1-x_2\|_2^2+\|y_1-y_2\|_2^2}\\
   &=D'\big\|(x_1,y_1)-(x_2,y_2)\big\|_2%\\
   %&\leq  3D \sqrt{\|x_1-x_2\|_2^2+\|y_1-y_2\|_2^2}.
\end{align*}
%\end{small}
imply that the function $d^2/D':\R^d\times \R^d\to \R$ is 1-Lipschitz in $\R^d\times \R^d$ with the standard Euclidean product metric.
Thus, for any $\gamma_1,\gamma_2\in \mathcal{P}(\mathcal{X\times \mathcal{X}})$, by using duality,
\begin{align}\label{eq: W2 using W1}
    W_2(\gamma_1,\gamma_2)&\geq W_1(\gamma_1,\gamma_2) \notag\\
    &=\sup_{Lip(\psi)\leq 1}\int \psi \, \mathrm d(\gamma_1-\gamma_2)\notag\\
    &\geq \frac{1}{D'}\int c(x,y)^2 \, \mathrm d(\gamma_1-\gamma_2).
\end{align}

\noindent\textbf{Step 2: Uniform inverse-Lipschitz for the slicers on a high-probability event.}
Now, we want to uniformly control the inverse-Lipschitz constant (or co-Lipschitz reciprocal) of the injective functions $g_{\xi_\eta}$. 
We recall that, for any fixed $\eta>0$, injectivity of the slicers $g_{\xi_\eta}$ is guaranteed for a.e. $\xi_\eta$  (Lemma \ref{lem:ae-inj}), allowing us to later take $\eta\to 0$.  We fix $\eta>0$.
For any distinct $x,x'\in\mathcal X$, let $u:=(x-x')/\|x-x'\|_2\in\mathbb S^{d-1}$, and from the reverse triangle inequality,
\begin{align}\label{eq: lip g inv}
    |g_{\xi_\eta}(x)-g_{\xi_\eta}(x')|&\geq \left||\langle\xi_\eta, x-x'\rangle|-|f(x)-f(x')|\right|\notag\\
    &\geq
\big(\eta|\langle Z,u\rangle|-L\big)\,\|x-x'\|_2.  
\end{align}
To beat the Lipschitz term $L$ with high probability, we will use the extra assumptions \ref{assump: extra}:
\begin{enumerate}
    \item $\Sigma =2 I_d$,
    \item $|\mathcal{X}|=B<\infty$.
\end{enumerate}
{Assuming} $\Sigma =2 I_d$, by Lemma  \ref{rem:proj-lap} one has
$\langle Z,u\rangle\sim\mathsf{Lap}_1(0,1)$ (i.e, we have isotropic variance --it does not depend on the direction $u$). Thus,
$$
\mathbb P\big(|\langle Z,u\rangle|<t\big)=1-e^{-t},\qquad t\geq 0.
$$
The {extra assumption} of $\mathcal{X}$ being \textbf{finite} with $B$ number of elements, gives at most $B'=\binom{B}{2}$ directions $\{u_j\}_{j=1}^{B'}\subset \mathbb S^{d-1}$. We can control $\mathbb P\bigl(\min_{1\le j\le B'}|\langle Z,u_j\rangle|<t\bigr)$ by the union bound  $\sum_{j=1}^{B'}\mathbb P(|\langle Z,u_j\rangle|<t)
=B'(1-e^{-t})$. Let $0<\delta<1$ such that $t_{\delta,B}:=-\ln\left(1-\frac{\delta}{B'}\right)>L/\eta$. So,
$$\mathbb P\Big(\underbrace{\min_{1\le j\le B'}|\langle Z,u_j\rangle|\geq t_{\delta,B}}_{=:G \, \text{ (``good event'')}}\Big)
\geq 1-B'(1-e^{-t_{\delta,B}})=:1-\delta.$$
Hence, if we denote by $K_Z$ the inverse-Lipschitz constant of $g_{\xi_\eta}$, from \ref{eq: lip g inv},
 on the ``good event'', with probability at least $1-\delta$, we get
\begin{equation*}
    K_Z \leq
\frac{1}{\eta t_{\delta,B}-L}.
\end{equation*}
In particular, 
\begin{equation}\label{eq: bound exp square}
    \sqrt{\mathbb E_Z(K_Z^2 \mathbf 1_{G})}\leq \frac{1}{\eta t_{\delta,B}-L} \quad \text{with } \mathbb P(G)\geq 1-\delta.
\end{equation}
%Since $\delta/B'$ is small, by using Taylor expansion $-\ln(1-\delta/B')\approx \delta/B'$, and the condition approximately reads as $\eta$ larger than $LB^2/\delta$ in order of magnitude (essentially, $\eta$ needs to be large).
Therefore, the inverse-Lipschitz constant $K_Z$
of the slicers $g_{\xi_\eta}$ is globally bounded on a high-probability event $G$. 

\noindent\textbf{Step 3: Bound for the change of $J_\eta$ with a fixed $f\in \mathcal F$.} Given $\mu_1,\mu_2,\nu_1,\nu_2\in \mathcal{P}(\mathcal{X})$, define
$\alpha_{i}^{(Z)}:=(g_{\xi_\eta})_\#\mu_i$ and $\beta_{i}^{(Z)}:=(g_{\xi_\eta})_\#\nu_i$.
Let
$\sigma_{i}^{(Z)}$ be the \emph{unique} 1D optimal plan between $\alpha_{i}^{(Z)}, \beta_{i}^{(Z)}$, with corresponding \emph{unique} lifted coupling
$\gamma_{i}^{(Z)}\in\ \Gamma(\mu_i,\nu_i)$
such that $(g_{\xi_\eta},g_{\xi_\eta})_{\#}\gamma_{i}^{(Z)}=\sigma_{i}^{(Z)}
$, for $i=1,2$. Finally, define, as before, 
\begin{equation*}
    J_\eta(\mu_i,\nu_i;f):=\int c(x,y)^p\, \mathrm d\bar\gamma_i^{(Z)}(x,y), 
\end{equation*}
where $\bar\gamma_i^{(Z)}=\mathbb E_Z[\gamma_{i}^{(Z)}]$ for $i=1,2$.

Using the above steps, for any $f\in \mathcal{F}$:
\begin{align}\label{eq: the bound}
    &\left|J_\eta(\mu_1,\nu_1;f)-J_\eta(\mu_2,\nu_2;f)\right|\notag\\
    &=\left|\mathbb E_Z\left(\int c(x,y)^2\, \mathrm d(\gamma_1^{(Z)}-\gamma_2^{(Z)})\right)\right| \notag\\
     &\leq\mathbb E_Z\left(\left|\int c(x,y)^2\, \mathrm d(\gamma_1^{(Z)}-\gamma_2^{(Z)})\right|\mathbf 1_G\right)\notag\\
    &+\mathbb E_Z\left(\left|\int c(x,y)^2\, \mathrm d(\gamma_1^{(Z)}-\gamma_2^{(Z)})\right|\mathbf 1_{G^c}\right)\notag\\
    &\leq \mathbb E_Z\left(\left|\int c(x,y)^2\, \mathrm d(\gamma_1^{(Z)}-\gamma_2^{(Z)})\right|\mathbf 1_G\right)\\    
    &+2 diam(\mathcal{X})^2\delta \notag.
\end{align}
Let us control the first term in \ref{eq: the bound}: Using inequality \ref{eq: W2 using W1} from Step 1, followed by Lemma \ref{lem:basic} (part 1) with $K_Z$ given from Step 2, and then Lemma \ref{lem:tools}, we have
\begin{small}
\begin{align}\label{eq: the bound 1}
    &\mathbb E_Z\left(\left|\int c(x,y)^2\, \mathrm d(\gamma_1^{(Z)}-\gamma_2^{(Z)})\right|\mathbf 1_G\right)\\
    &\leq D' \, \mathbb E_Z\left(W_2(\gamma_1^{(Z)},\gamma_2^{(Z)})\, \mathbf 1_G\right)\notag\\
    &\leq D' \, \mathbb E_Z\left(K_Z \, W_2(\sigma_1^{(Z)},\sigma_2^{(Z)})\, \mathbf 1_G\right)\notag\\
    &\leq D' \, \mathbb E_Z\left(K_Z\mathbf 1_G\big(W_2^2(\alpha_1^{(Z)},\alpha_2^{(Z)}) + W_2^2(\beta_1^{(Z)},\beta_2^{(Z)}) \big)^{\frac{1}{2}}\right)\notag\\
    &\leq D' \, \mathbb E_Z(K_Z \mathbf 1_G \, (L+\eta\|Z\|_2))\big( W_2^2(\mu_1,\mu_2) + W_2^2(\nu_1,\nu_2) \big)^{\frac{1}{2}}.\notag
\end{align}   
\end{small}
We recall that the terms $\big( W_2^2(\mu_1,\mu_2) + W_2^2(\nu_1,\nu_2) \big)^{\frac{1}{2}}$ and $W_2(\mu_1,\mu_2) + W_2(\nu_1,\nu_2)$ are comparable, since $$\sqrt{a^2+b^2}\leq a+b\leq \sqrt{2} \sqrt{a^2+b^2} \qquad \text{for } a,b\geq 0.$$
Moreover, by Cauchy-Schwartz inequality,
\begin{align*}
    &\mathbb E_Z(K_Z\mathbf 1_G(L+\eta\|Z\|_2))\\
    &\leq (\mathbb E_Z(K_Z^2\mathbf 1_G))^{\frac{1}{2}}(\mathbb E_Z((L+\eta\|Z\|_2)^2))^\frac{1}{2}.
\end{align*}    
\noindent Again, by Cauchy-Schwartz inequality $  \mathbb E_Z(\|Z\|_2)\leq (\mathbb E_Z(\|Z\|_2^2))^{\frac{1}{2}}$, and using $\mathbb E_Z(\|Z\|_2^2)=\mathrm{tr}(\Sigma)$ one can estimate
\begin{align*}
    \mathbb E_Z((L+\eta\|Z\|_2)^2)&\leq L^2+2L\eta\mathrm{tr}(\Sigma)^\frac{1}{2}+\eta^2\mathrm{tr}(\Sigma)\\
    &=(L+\eta(\mathrm{tr}(\Sigma))^\frac{1}{2})^2,
\end{align*}
which, in particular, if $\Sigma =2 I_d$, then $(\mathrm{tr}(\Sigma))^\frac{1}{2}=\sqrt{2d}$.
Therefore, replacing the above observations in \ref{eq: the bound 1} together with \ref{eq: bound exp square}, we have 
\begin{align*}%\label{eq: the bound 11}
    &\mathbb E_Z\left(\left|\int c(x,y)^2\, \mathrm d(\gamma_1^{(Z)}-\gamma_2^{(Z)})\right|\mathbf 1_G\right) \notag \\
    &\leq C(\eta,\delta)\,\big(W_2(\mu_1,\mu_2)+W_2(\nu_1,\nu_2)\big),     
\end{align*}
where
\begin{equation}\label{eq: C}
    C(\eta,\delta):=2\sqrt{2} \, diam(\mathcal{X})\, \frac{L+\eta\sqrt{2d}}{\eta t_{\delta,B}-L}.
\end{equation}
As a conclusion, 
\begin{align}\label{eq: final bound}
    &\left|J_\eta(\mu_1,\nu_1;f)-J_\eta(\mu_2,\nu_2;f)\right|\\
    &\leq C(\eta,\delta)\,\big(W_2(\mu_1,\mu_2)+W_2(\nu_1,\nu_2)\big)+ 2diam(\mathcal{X})^2\delta. \notag
\end{align}

\noindent\textbf{Step 4: From value stability to argmin transfer.}

If $W_2(\mu_1,\mu_2)+W_2(\nu_1,\nu_2)<\tau$,  inequality \ref{eq: final bound} implies
$$
\Big|\min_{f\in \mathcal F}J_\eta(\mu_1,\nu_1;f)-\min_{f\in \mathcal F}J_\eta(\mu_2,\nu_2;f)\Big|<\underbrace{C(\eta,\delta)\tau+2D^2\delta}_{=:C'}
$$
Indeed, if we define $v_i:=\min_{f\in\mathcal F}J_\eta(\mu_i,\nu_i;f)$,
and $\mathcal{S}_i:=\mathcal{S}(\mu_i,\nu_i)=\arg\min_{f\in\mathcal F}J_\eta(\mu_i,\nu_i;f)$ (non-empty and compact),
for $i=1,2$, from inequality \ref{eq: final bound} we have, in particular, 
that for any $f_1^\star\in \mathcal{S}_1$,
\begin{align*}
    v_2\leq J_\eta(\mu_2,\nu_2;f_1^\star)< J_\eta(\mu_1,\nu_1;f_1^\star)+C'=v_1+C',
\end{align*}
and similarly (interchanging roles), 
for any $f_2^\star\in \mathcal{S}_2$,
\begin{align*}
    v_1\leq J_\eta(\mu_1,\nu_1;f_2^\star)< J_\eta(\mu_2,\nu_2;f_2^\star)+C'=v_2+C',
\end{align*}
getting $ |v_1-v_2|< C'=C(\eta,\delta)\tau+2 diam(\mathcal X)^2\delta$. 
\\

\textbf{Step 5: Choice of $\tau, \eta, \delta$ given $\varepsilon>0$.} Similarly as in the proof of
Theorem \ref{thm:w_f}, we define $$U_\varepsilon(\mathcal{S}_1):=\{f\in \mathcal{F}\mid \, \|f-\mathcal{S}_1\|_\infty<\varepsilon\},$$ where $\|f-\mathcal{S}_1\|_\infty$ denotes the $L^\infty$-distance between $f$ and the compact set $\mathcal{S}_1$.
Define the ``margin'' around $\mathcal S_1$ by
$$
m_\varepsilon^{(1)}\ :=\ \inf_{f\in \mathcal{F}\setminus U_\varepsilon} J_\eta(\mu_1,\nu_1;f)\ -\ \min_{g\in\mathcal F}J_\eta(\mu_1,\nu_1;g).
$$
Notice that $m_\varepsilon^{(1)}>0$: indeed, by definition $m_\varepsilon^{(1)}\geq 0$, and if  $m_\varepsilon^{(1)}=0$, it would imply that there exists $\widetilde f\in \mathcal{F}\setminus U_\varepsilon(\mathcal S_1)$ such that $J_\eta(\mu_1,\nu_1;\widetilde f)=v_1$, i.e., $\widetilde f\in \mathcal{S}_1$ contradicting $\|\widetilde f-\mathcal{S}_1\|_\infty\geq \varepsilon>0$ (where the existence of minimizer $\widetilde f$ relies on the compactness of $\mathcal{F}$ and the continuity of $J_\eta$).
which is positive by compactness of $\mathcal F$ and continuity of $J_\eta$. Choose $\delta\in(0,1)$ and $\eta>0$ so that $\eta t_{\delta,B}-L>0$ and the bad-event penalty obeys
\begin{equation}\label{eq:bad-small}
2diam(\mathcal X)^2\,\delta\ \le\ m_\varepsilon^{(1)}/4.
\end{equation}
Then choose
\begin{equation}\label{eq:tau-choice}
0<\tau\ \leq\ \frac{m_\varepsilon^{(1)}}{4\,C(\eta,\delta)}.
\end{equation}
As a consequence, we obtain
\[
\Big|\min_{f\in \mathcal F}J_\eta(\mu_1,\nu_1;f)-\min_{f\in \mathcal F}J_\eta(\mu_2,\nu_2;f)\Big|<\frac{m_\varepsilon^{(1)}}{2}.
\]
By the usual margin argument (similarly as in the proof of Theorem~\ref{thm:w_f}), this yields $\mathcal S_2\subset U_\varepsilon(\mathcal S_1)$.
See Figure \ref{fig:sketch_margin} for a schematic visualization.
In fact,
if $f_2^\star\in \mathcal{S}_2$ but $f_2^\star\in \mathcal{F}\setminus U_\varepsilon(\mathcal{S}_1)$, by definition of $m_\varepsilon^{(1)}$,
\begin{equation*}\label{eq: J upper bound}
     v_1 + m_\varepsilon ^{(1)}\leq J_\eta(\mu_1,\nu_1;f_2^\star)<v_2+\frac{m_\varepsilon^{(1)}}{2}<v_1+m_{\varepsilon}^{(1)},
\end{equation*}
that is, $m_\varepsilon^{(1)} < m_\varepsilon^{(1)}$, which yields a contradiction. 

Changing the roles and defining the margin around $\mathcal{S}_2$
$$
m_\varepsilon^{(2)}\ :=\ \inf_{f\in \mathcal{F}\setminus U_\varepsilon} J_\eta(\mu_2,\nu_2;f)\ -\ \min_{g\in\mathcal F}J_\eta(\mu_2,\nu_2;g).
$$
we have that under the choices:
\begin{enumerate}
    \item $\delta\in(0,1)$ and $\eta>0$ so that  
$
\binom{B}{2}(1-e^{-L/\eta})< \delta\leq m_\varepsilon^{(1)}/8diam(\mathcal{X})^2$, and
\item  $
\tau\ \leq\ \frac{m_\varepsilon^{(2)}}{4\,C(\eta,\delta)}$,
\end{enumerate}
we get 
$ \mathcal S_1\subset U_\varepsilon(\mathcal S_2)$, where $$U_\varepsilon(\mathcal S_2):=\{f\in \mathcal{F}\mid \, \|f-\mathcal{S}_2\|_\infty<\varepsilon\}.$$ This part reads as the thesis of Theorem \ref{thm:w_f}: for every $f_1^\star\in\mathcal S_1$, under this choice of parameters $\eta,\delta,\tau$ depending on $\varepsilon>0$, we have  $\|f_2^\star-f_1^\star\|_\infty<\varepsilon$ for some $f_2^\star\in\mathcal S_2$.

\begin{figure}[ht!]
    \centering
    \includegraphics[width=0.5\linewidth]{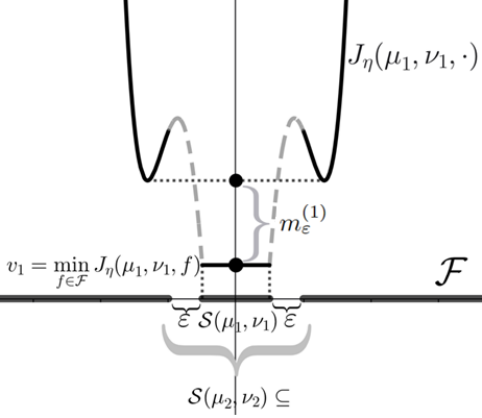}
    \caption{Sketch illustrating that $\mathcal{S}(\mu_2,\nu_2)\subset U_\varepsilon(\mathcal{S}(\mu_1,\nu_1))$ if $\tau$ is chosen appropriately depending on the gap $m_\varepsilon^{(1)}$.}
    \label{fig:sketch_margin}
\end{figure}

\textbf{Note:} Notice that using these quantitative bounds, there exists a trade-off between $\delta$ and $\eta$: Indeed, the condition $\eta t_{\delta,B}>L$ (equivalently, 
$\delta> \binom{B}{2}(1-e^{-L/\eta})$) ensures that the slicers $g_{\xi_\eta}$ are uniformly inverse-Lipschitz on a high-probability event. 
Larger values of $\eta$ make this condition easier to satisfy (allowing smaller~$\delta$), 
while smaller $\eta$ forces $\delta$ to be larger. 
On the other hand, the final stability bound requires $\delta$ to be sufficiently small in order to control the bad-event penalty. 
\\

\noindent\textbf{Comment: For a fixed pair of source-target $(\mu,\nu)$measures, $J_\eta$ is Lipschitz with respect to $(\mathcal{F},\|\cdot\|_\infty)$ with high probability.}
Here we repeat some arguments from step 3 but, instead of fixing $f\in \mathcal{F}$ and varying $\mu_i,\nu_i\in \mathcal{P}(X)$ ($i=1,2$), we fix a pair of source-target measures $\mu,\nu\mathcal{P}(\mathcal{X})$ and consider two different arbitrary slicers $f_1,f_2\in \mathcal{F}$.

Let $f_1,f_2\in \mathcal F$.
Fix a realization of $Z$ and set
\[
g^{(i)}(x):=g^{f_i}_{\xi_\eta}(x)=f_i(x)+\eta\langle Z,x\rangle \qquad i=1,2.
\]
Then for all $x\in\mathcal X$,
\[
|g^{(1)}(x)-g^{(2)}(x)| = |f_1(x)-f_2(x)|\ \le\ \|f_1-f_2\|_\infty.
\]

Let $\mu,\nu\in \mathcal{P}(\mathcal X)$.

For $i=1,2$, let $\alpha_{f_i}^{(Z)}:=(g^{(i)})_\#\mu$ and $\beta_{f_i}^{(Z)}:=(g^{(i)})_\#\nu$; let 
$\sigma_{f_i}^{(Z)}$ be the \emph{unique} 1D optimal plan between $\alpha_{f_i}^{(Z)}, \beta_{f_i}^{(Z)}$, with corresponding \emph{unique} lifted coupling
$\gamma_{f_i}^{(Z)}\in\ \Gamma(\mu,\nu)$
such that $(g^{(i)},g^{(i)})_{\#}\gamma_{i}^{(Z)}=\sigma_{i}^{(Z)}
$.

By Lemma~\ref{lem:basic}, for the pushforwards of $\mu$ we obtain
\[
W_2\big(\alpha_{f_1}^{(Z)},\alpha_{f_2}^{(Z)}\big)\ \le\ \|f_1-f_2\|_\infty,
\]
and similarly for the pushforwards of $\nu$,
\[
W_2\big(\beta_{f_1}^{(Z)},\beta_{f_2}^{(Z)}\big)\ \le\ \|f_1-f_2\|_\infty.
\]

Applying Lemma~\ref{lem:tools}, we get
\begin{align*}
W_2\big(\sigma_{f_1}^{(Z)},\sigma_{f_2}^{(Z)}\big)
&\leq \Big(
    W_2^2\big(\alpha_{f_1}^{(Z)},\alpha_{f_2}^{(Z)}\big)\\
   &\qquad +W_2^2\big(\beta_{f_1}^{(Z)},\beta_{f_2}^{(Z)}\big)
\Big)^{1/2} \\
&\leq \sqrt{2}\,\|f_1-f_2\|_\infty.
\end{align*}

As a consequence of steps 1 and 2 and the above inequality,
\begin{align*}
&\big|J_\eta(\mu,\nu;f_1)-J_\eta(\mu,\nu;f_2)\big| \\
&\quad = \Big|\mathbb E_Z\Big[\int c^2\,d\gamma_{f_1}^{(Z)}-\int c^2\,d\gamma_{f_2}^{(Z)}\Big]\Big| \\
&\quad \leq \mathbb E_Z\Big[\,\Big|\int c^2\,d\gamma_{f_1}^{(Z)}-\int c^2\,d\gamma_{f_2}^{(Z)}\Big|\,\Big] \\
&\quad \leq 4diam(\mathcal{X})\,\|f_1-f_2\|_\infty\,\mathbb E_Z[K_Z\mathbf 1_G] + 2diam(\mathcal{X})^2\delta\\
&\quad \leq \frac{4diam(\mathcal{X})\,(1-\delta)}{\eta t_{\delta,B}-L}\|f_1-f_2\|_\infty\, + 2diam(\mathcal{X})^2\delta.
\end{align*}

\section{Statistical Properties of Minibatch Training}
\label{sec:minibatch}
In this section, we consider the minibatch training of the slicer over finite discrete samples. In each minibatch, samples are reweighted to have uniform mass.

\begin{definition}[Minibatch STP loss as a two-sample U-statistic]
\label{def:mini}
Let $X=\{x_i\}_{i=1}^N$ and $Y=\{y_j\}_{j=1}^M$ be two sets of discrete samples, and write $[n]:=\{1,\dots,n\}$. 
Fix a batch size $B\le \min(N, M)$. Define the two-sample kernel for batches $x_{i_1:i_B}$ and $y_{j_1:j_B}$ with slicer $f:\mathbb{R}^d\to\mathbb{R}$ as
\[
h_B\!\left(f; x_{i_1:i_B},y_{j_1:j_B}\right)
:= \frac{1}{B}\sum_{i=1}^B d\big(x^{f}_{(i)},\,y^{f}_{(i)}\big)^p,
\]
where $x^{f}_{(i)},\,y^{f}_{(i)}\in\mathbb{R}^d$ denote the re-ordered $x_{i_1:i_B}$ and $y_{j_1:j_B}$, paired by the lifted coupling from 1D optimal coupling, i.e. $f(x^{f}_{(1)})\le \cdots \le f(x^{f}_{(B)})$ and $f(y^{f}_{(1)})\le \cdots \le f(y^{f}_{(B)})$ for the slicer $f:\mathbb{R}^d\to\mathbb{R}$.

The minibatch STP loss (U-statistic of order $B,B$) for a slicer $f:\mathbb{R}^d\to\mathbb{R}$ is defined as 
\[
J_{N, M, B}(f)
:= \frac{1}{\binom{N}{B}\binom{M}{B}}
\sum_{S\in\binom{[N]}{B}}\;
\sum_{T\in\binom{[M]}{B}}
h_B\big(f;\, x_S,\, y_T\big),
\]
where $S$ is the set of $B$ indices drawn from $[N]$ without replacement and $T$ is the set of $B$ indices drawn from $[M]$ without replacement. $x_S=(x_i)_{i\in S}$ and $y_T=(y_j)_{j\in T}$.

The incomplete estimator draws $K$ i.i.d. \ pairs $\{(S_k,T_k)\}_{k=1}^K$ uniformly from 
$\binom{[N]}{B}\times\binom{[M]}{B}$, and averages:
\[
\overline{J}_{B,K}(f)
:= \frac{1}{K}\sum_{k=1}^K h_B\big(f;\,x_{S_k},\,y_{T_k}\big).
\]

In the following context, $(S,T)\sim\mathrm{Unif}\!\left(\binom{[N]}{B}\times\binom{[M]}{B}\right)$.
 (i.e., uniform distribution over all size-$B$ subsets.)
\end{definition}

\begin{proposition}[Incomplete vs.\ complete deviation]
\label{prop:incomplete}
Fix the datasets $X=\{x_i\}_{i=1}^N,Y=\{y_j\}_{j=1}^M$. Let $R:=\max_{i, j} d(x_i, y_j)^p$. Then $\mathbb{E}_{S,T}[h_B\!\big(f;\,x_S,y_T\big)]=J_{N, M,B}(f)$ and, for any $\epsilon>0$,
\[
\Pr\!\left(\left|\overline{J}_{B,K}(f)-J_{N, M,B}(f)\right|\ge \epsilon \,\right)
\;\le\; 2\exp\!\Big(\!-\,\frac{2K\epsilon^2}{R^2}\Big).
\]
In particular,
\begin{gather*}
\mathrm{Var}\!\left(\overline{J}_{B,K}(f)\right)\le \frac{R^2}{4K},\\
\mathbb{E}\!\left[\left|\overline{J}_{B,K}(f)-J_{N, M,B}(f)\right|\right]
\le \frac{R}{2\sqrt{K}}.
\end{gather*}
\end{proposition}

\begin{proof}
By definition,
\[
J_{N,M,B}(f)
=\frac{1}{\binom{N}{B}\binom{M}{B}}
\sum_{S\in\binom{[N]}{B}}\sum_{T\in\binom{[M]}{B}}
h_B\!\big(f;\,x_S,y_T\big),
\]
so taking expectation over a uniformly chosen $(S,T)$ yields
\[
\mathbb{E}_{S,T}\,h_B\!\big(f;\,x_S,y_T\big)=J_{N,M,B}(f).
\]

Next, for each draw $\{(S_k,T_k)\}_{k=1}^K$,
\[
\overline{J}_{B,K}(f)
=\frac{1}{K}\sum_{k=1}^K h_B\!\big(f;\,x_{S_k},y_{T_k}\big).
\]
Since each summand in $h_B$ is of the form
$d(x^{f}_{(i)},y^{f}_{(i)})^p\in[0,R]$, we have $0\le h_B(f;\,x_S,y_T)\le R$ for all $S,T$.
Hence, the variables
$h_B\!\big(f;\,x_{S_k},y_{T_k}\big)$ are i.i.d.\ and lie in $[0,R]$.

Applying Hoeffding’s inequality to $\overline{J}_{B,K}(f)$ gives, for any $\epsilon>0$,
\[
\Pr\!\left(\left|\overline{J}_{B,K}(f)-
\mathbb{E}\overline{J}_{B,K}(f)\right|\ge \epsilon \right)
\le 2\exp\!\Big(-\frac{2K\epsilon^2}{R^2}\Big).
\]
Using $\mathbb{E}\overline{J}_{B,K}(f)=\mathbb{E}_{S,T}\,h_B(f;\,x_S,y_T)=J_{N,M,B}(f)$
yields the stated tail bound.

For the variance bound, by independence,
\[
\mathrm{Var}\!\left(\overline{J}_{B,K}(f)\right)
=\frac{1}{K}\,\mathrm{Var}\!\left(h_B(f;\,x_S,y_T)\right)
\le \frac{1}{K}\cdot \frac{R^2}{4}=\frac{R^2}{4K},
\]
where we used that a random variable supported on an interval of length $R$ has variance at most $R^2/4$.
Finally, by Cauchy--Schwarz,
\[
\mathbb{E}\!\left[\big|\overline{J}_{B,K}(f)-J_{N,M,B}(f)\big|\right]
\le \sqrt{\mathrm{Var}\!\left(\overline{J}_{B,K}(f)\right)}
\le \frac{R}{2\sqrt{K}}.
\]
\end{proof}
\begin{lemma}[McDiarmid's Inequality for two i.i.d. samples]
Let $X, Y\subset \mathbb{R}^d$, $\mu\in\mathcal{P}(X)$, $\nu\in\mathcal{P}(Y)$. Let $F:X^N\times Y^M\to\mathbb R$ for $N, M\geq 1$. Assume for all $1\leq i\leq N,1\leq j\leq M$, the change of $F$ by substituting one coordinate is bounded, i.e. there exists $a_i$ and $b_j$ such that
\begin{gather*}
\big|F(\ldots,x_i,\ldots;y)-F(\ldots,x'_i,\ldots;y)\big|\le a_i,\\
\big|F(x;\ldots,y_j,\ldots)-F(x;\ldots,y'_j,\ldots)\big|\le b_j.
\end{gather*}
where $x\in X^N, y\in Y^M$ and $x_i,x_i'\in X, y_j, y'_j\in Y$. If $x_1,\dots,x_N\!\stackrel{\text{i.i.d.}}{\sim}\!\mu$ and
$y_1,\dots,y_M\!\stackrel{\text{i.i.d.}}{\sim}\!\nu$ are all independent, then
\begin{align*}
\Pr\!\left(\big|F(x_{1:N},y_{1:M})-\mathbb EF(x_{1:N},y_{1:M})\big|\ge\epsilon\right)\\
\le 2\exp\!\left(-\frac{2\epsilon^2}{\sum_{i=1}^{N}a_i^2+\sum_{j=1}^{M}b_j^2}\right).
\end{align*}
\end{lemma}
\begin{proof}
Let $m:=N+M$ and define the independent sequence $Z=\{z_l\}_{l=1}^m$ with 
$z_1:=x_1,\dots,z_N:=x_N,z_{N+1}:=y_1,\dots,z_m:=y_{M}$.
With a slight abuse of notation, we write $F(Z):=F(x_{1:N},y_{1:M})$ and the Doob martingale
$M_l:=\mathbb E[F(Z)\mid z_1,\dots,z_l]$ for $l=1,\dots,m$ and $M_0=\mathbb{E}(F(Z))$, with increments
$\Delta_l:=M_l-M_{l-1}$ for $1\le l\le m$.

Fix $l$ and the prefix $(z_1,\dots,z_{l-1})$. Consider
$\varphi_l(z):=\mathbb E[F(z_1,\dots,z_{l-1},z,z_{l+1},\dots,z_m)|z_1,\dots,z_{l-1}]$.
By the bounded-differences assumption on the $l$th coordinate,
$|\varphi_l(z)-\varphi_l(z')|\le \alpha_l$ for $z, z'\in X$ or $z, z'\in Y$, where
\[
\alpha_l=\begin{cases}
a_l,& l\le N,\\
b_{\,l-N},& l>N.
\end{cases}
\]
Hence $\Delta_l=\varphi_l(z_l)-\mathbb E[\varphi_l(z_l)\mid z_1,\dots,z_{l-1}]$
is centered and supported in an interval of length at most $\alpha_l$.
By Hoeffding’s lemma,
\[
\mathbb E\!\left[e^{\lambda \Delta_l}\mid z_1,\dots,z_{l-1}\right]
\le \exp\!\left(\frac{\lambda^2\alpha_l^2}{8}\right).
\]
Iterating and using the tower property,
\begin{align*}
\mathbb E\,e^{\lambda(F-\mathbb EF)}
&=\mathbb E\,e^{\lambda\sum_{l=1}^m \Delta_l}\\
&\le \exp\!\left(\frac{\lambda^2}{8}\sum_{l=1}^m \alpha_l^2\right)\\
&=\exp\!\left(\frac{\lambda^2}{8}\left(\sum_{i=1}^{N}a_i^2+\sum_{j=1}^{M}b_j^2\right)\right).
\end{align*}
By Chernoff’s method,
\[
\Pr\!\left(F-\mathbb EF\ge \epsilon\right)
\le \exp\!\left(-\lambda\epsilon+\frac{\lambda^2}{8}\Big(\sum_i a_i^2+\sum_j b_j^2\Big)\right).
\]
Optimizing at $\lambda^\star=\dfrac{4\epsilon}{\sum_i a_i^2+\sum_j b_j^2}$ gives
\[
\Pr\!\left(F-\mathbb EF\ge \epsilon\right)
\le \exp\!\left(-\frac{2\epsilon^2}{\sum_i a_i^2+\sum_j b_j^2}\right).
\]
Apply the same to $-(F-\mathbb EF)$ and we can obtain
\[
\Pr\!\left(\big|F-\mathbb EF\big|\ge \epsilon\right)
\le 2\exp\!\left(-\frac{2\epsilon^2}{\sum_{i=1}^{N}a_i^2+\sum_{j=1}^{M}b_j^2}\right).
\]
\end{proof}

\begin{proposition}[Dataset $\to$ population deviation, fixed $f$]
Let $x_{1:N}\stackrel{\text{i.i.d.}}{\sim}\mu$ and $y_{1:M}\stackrel{\text{i.i.d.}}{\sim}\nu$.
Fix $f$ and a batch size $B\le \min(N,M)$. Assume $d(x,y)\le D$ on $\mathcal X$ and set $R:=D^p$.
Define the population target $J_B(f):=\mathbb{E}\,h_B\!\big(f; x_{i_1:i_B},y_{j_1:j_B}\big)$ with
$x_{i_1:i_B}\sim\mu^{\otimes B}$ and $y_{j_1:j_B}\sim\nu^{\otimes B}$. Then, for any $\epsilon>0$,
\begin{align*}
\Pr\!\left(\left|J_{N,M,B}(f)-J_B(f)\right|\ge \epsilon\right)\\
\le
2\exp\!\left(
-\frac{2\,\epsilon^2}{\,B^2 R^2\left(\frac{1}{N}+\frac{1}{M}\right)}
\right).
\end{align*}
\end{proposition}

\begin{proof}
$J_{N,M,B}(f)$ is a two-sample U-statistic with kernel $h_B$, hence
$\mathbb{E}\,J_{N,M,B}(f)=J_B(f)$.

View $J_{N,M,B}(f)$ as a function of the independent inputs
$(x_1,\dots,x_{N},y_1,\dots,y_{M})$. Replacing a single $x_i$ affects exactly the
fraction $\binom{N-1}{B-1}/\binom{N}{B}=B/N$ of the summands; since
each summand lies in $[0,R]$ we have a change of at most $(B/N)R$. Similarly, replacing
a single $y_j$ changes the value by at most $(B/M)R$. McDiarmid’s inequality with bounded
differences $a_i=(B/N)R$ for each $x_i$ and $b_j=(B/M)R$ for each $y_j$ yields
\begin{gather*}
\Pr\left(\left|J_{N,M,B}(f)-\mathbb{E}J_{N,M,B}(f)\right|\ge \epsilon\right)\\
\le2\exp\!\left(
-\frac{2\,\epsilon^2}{\,\sum_i a_i^2+\sum_j b_j^2}
\right)\\
=2\exp\!\left(
-\frac{2\,\epsilon^2}{\,B^2 R^2\left(\frac{1}{N}+\frac{1}{M}\right)}
\right),
\end{gather*}
and $\mathbb{E}J_{N,M,B}(f)=J_B(f)$ finishes the proof.
\end{proof}

\begin{lemma}[Subadditivity and scale control for moduli of continuity]
\label{lem:moduli}
Let $I\subset\mathbb{R}$ be an interval and $g:I\to\mathbb{R}$.
Define the modulus of continuity
\[
\omega_g(t)\ :=\ \sup_{\substack{x,y\in I\\ |x-y|\le t}}\ |g(x)-g(y)|,\qquad t\ge 0.
\]
Then $\omega_g(t)$ has the following properties
\begin{enumerate}
\item[(i)] (Monotonicity) $\omega_g(0)=0$ and $\omega_g$ is nondecreasing.
\item[(ii)] (Subadditivity) For all $s,t\ge 0$,
\[
\omega_g(s+t)\ \le\ \omega_g(s)+\omega_g(t).
\]
\item[(iii)] (Scale-specific linear bound) For any $\eta>0$ and all $t\ge 0$,
\[
\omega_g(t)\ \le\ \Big(\tfrac{t}{\eta}+1\Big)\,\omega_g(\eta).
\]
\end{enumerate}
\end{lemma}

\begin{proof}
(i) Trivial: the supremum over an empty displacement is $0$; enlarging the admissible set in $t$ cannot decrease the supremum.

(ii) Fix $x,y\in I$ with $|x-y|\le s+t$. Choose $z$ on the segment $[x,y]$ so that $|x-z|\le s$ and $|z-y|\le t$ (possible since $I$ is an interval). Then
\[
|g(x)-g(y)|\ \le\ |g(x)-g(z)|+|g(z)-g(y)|\ \le\ \omega_g(s)+\omega_g(t).
\]
Taking the supremum over such $x,y$ gives $\omega_g(s+t)\le \omega_g(s)+\omega_g(t)$.

(iii) Write $t=m\eta+r$ with $m=\lfloor t/\eta\rfloor\in\mathbb{N}$ and $r\in[0,\eta)$. By (ii),
\begin{align*}
    \omega_g(t)\ =\ \omega_g(m\eta+r)\ \le\ m\,\omega_g(\eta)+\omega_g(r)\ \\
    \le\ (m+1)\,\omega_g(\eta)
\ =\ \Big(\big\lceil\tfrac{t}{\eta}\big\rceil\Big)\,\omega_g(\eta),
\end{align*}

and since $\lceil t/\eta\rceil\le t/\eta+1$, the second inequality follows.
\end{proof}

\begin{corollary}[Scale-specific linearization in expectation]
\label{cor:exp_moduli}
Let $T\ge 0$ be any random variable with $\mathbb{E}T<\infty$. For any $\eta>0$,
\[
\mathbb{E}\,\omega_g(T)\ \le\ \omega_g(\eta)\,\Big(\tfrac{\mathbb{E}T}{\eta}+1\Big).
\]
\end{corollary}

\begin{proposition}[Minibatch $\to$ full STP gap under bi-Lipschitz slicer]
\label{prop:mb-vs-full-bilip}
Let $X=\{x_i\}_{i=1}^N$, $Y=\{y_j\}_{j=1}^M$ and $f:\R^d\to\R$.
Define $\mu_N=\tfrac1N\sum_{i=1}^N\delta_{x_i}$, $\nu_M=\tfrac1M\sum_{j=1}^M\delta_{y_j}$ and let
$\gamma=\{\gamma_{ij}\}$ be the lifted plan from the 1D optimal transport plan between $f_\#\mu_N$ and $f_\#\nu_M$, with
$\gamma_{ij}\ge 0$, $\sum_j\gamma_{ij}=1/N$, $\sum_i\gamma_{ij}=1/M$.
Recall the $\text{STP}_p^p$ objective. For convenience we denote it as $J_{N,M}$ in connection to the minibatch version:
\[
J_{N,M}(f):=\sum_{i=1}^N\sum_{j=1}^M \gamma_{ij}\,d(x_i,y_j)^p.
\]
Assume there exist constants $c_X,c_Y>0$ and $L_X,L_Y\ge 0$ such that
\begin{align}
\label{eq:bi-lip-lower}
|f(x)-f(x')| &\;\ge\; c_X\,d(x,x') \quad \text{for all }x,x'\in X,\\
\label{eq:bi-lip-lower-y}
|f(y)-f(y')| &\;\ge\; c_Y\,d(y,y') \quad \text{for all }y,y'\in Y,\\
\label{eq:cost-lip}
\big|d(x,y)^p - d(x',y')^p\big|
&\;\le\; L_X\,d(x,x') + L_Y\,d(y,y') \\
&\quad\text{for all }x,x',y,y'.\nonumber
\end{align}
Write the projected oscillations
\begin{align*}
\osc_X(f):=\max_{x\in X} f(x)-\min_{x\in X} f(x),\\
\osc_Y(f):=\max_{y\in Y} f(y)-\min_{y\in Y} f(y).
\end{align*}
Then, deterministically (for the fixed datasets $X,Y$),
\begin{gather*}
\label{eq:mb-full-gap}
\big|\,J_{N,M,B}(f)-J_{N,M}(f)\,\big|
\le \frac{1}{B}\!\left(\frac{L_X}{c_X}\,\osc_X(f)+\frac{L_Y}{c_Y}\,\osc_Y(f)\right)\\
\quad+2
\frac{L_X}{c_X}\omega_X(\frac{1}{2}\sqrt{\frac{N-B}{NB}}\big)
+2\frac{L_Y}{c_Y}\omega_Y(\frac{1}{2}\sqrt{\frac{M-B}{MB}}\big).\\
\end{gather*}
where $\omega_X,\omega_Y$ are the quantile moduli for the pushforward empirical measures $f_{\#}\mu_N$ and $f_{\#}\nu_M$ (nondecreasing, zero at $0$, step-like in the discrete case).

\end{proposition}

\begin{proof}
Order the samples by the slicer: $f(x^f_{(1)})\leq\cdots\leq f(x^f_{(N)})$ and
$f(y^f_{(1)})\leq\cdots\leq f(y^f_{(M)})$.
Define the stepwise-constant ``quantile selectors''
\begin{align*}
X_f(u):=x^f_{(i)}\ \text{ for }u\in\big((i{-}1)/N,\,i/N\big],\\
Y_f(u):=y^f_{(j)}\ \text{ for }u\in\big((j{-}1)/M,\,j/M\big].
\end{align*}
Then the full objective admits the integral form
\begin{equation}
\label{eq:J-full-as-integral}
J_{N,M}(f)\;=\;\int_0^1 d\!\big(X_f(u),Y_f(u)\big)^p\, \mathrm du.
\end{equation}
Partition $[0,1]$ into $B$ equal sub-intervals $I_k=((k{-}1)/B,\,k/B]$, and assume
$u_k\in I_k$, $v_k\in I_k$ be the rank locations of the $k$-th order statistics of a size-$B$ subset from $X$ and $Y$ respectively, i.e. the minibatch is constructed by picking one pair of points from each $I_k$. For any such subset pair $(S,T)$,
\[
h_B\!\big(f;x_S,y_T\big)
\;=\;\frac1B\sum_{k=1}^B d\!\big(X_f(u_k),\,Y_f(v_k)\big)^p.
\]

Fix a block $I_k=((k{-}1)/B,\,k/B]$. Define the bivariate cost
\[
\scriptd_\times(u,v)\ :=\ d\!\big(X_f(u),Y_f(v)\big)^p,\qquad (u,v)\in[0,1]^2,
\]
and the diagonal (comonotone) integrand
\[
\scriptd(u)\ :=\ \scriptd_\times(u,u)\ =\ d\!\big(X_f(u),Y_f(u)\big)^p.
\]
Let the bin average be
\[
\overline{\scriptd}_{k}\ :=\ B\int_{I_k}\scriptd(u)\,du.
\]
Then for $u_k,v_k\in I_k$,
\begin{align*}
\Bigl|\frac{1}{B}\,\scriptd_\times(u_k,v_k)-\int_{I_k}\scriptd(u)\, \mathrm du\Bigr|
=\frac{1}{B}\,\bigl|\scriptd_\times(u_k,v_k)-\overline{\scriptd}_{k}\bigr|\\
\le\ \frac{1}{B}\,\osc_{I_k\times I_k}(\scriptd_\times),
\end{align*}
where $\osc_{I_k\times I_k}(\scriptd_\times):=\sup_{(u,v),(u',v')\in I_k\times I_k}
\big|\scriptd_\times(u,v)-\scriptd_\times(u',v')\big|$.

Next, bound the bin oscillation via the Lipschitz and bi-Lipschitz assumptions:
for any $(u,v),(u',v')\in I_k\times I_k$,
\[
\begin{aligned}
&\bigl|\scriptd_\times(u,v)-\scriptd_\times(u',v')\bigr|\\
&=\bigl|d(X_f(u),Y_f(v))^p-d(X_f(u'),Y_f(v'))^p\bigr|\\
&\le L_X\,d\!\big(X_f(u),X_f(u')\big)+L_Y\,d\!\big(Y_f(v),Y_f(v')\big)\\
&\le \frac{L_X}{c_X}\,\bigl|f(X_f(u))-f(X_f(u'))\bigr|\\
  & + \frac{L_Y}{c_Y}\,\bigl|f(Y_f(v))-f(Y_f(v'))\bigr|.
\end{aligned}
\]
Taking supremum over $I_k\times I_k$ yields
\[
\osc_{I_k\times I_k}(\scriptd_\times)
\ \le\ \frac{L_X}{c_X}\,\osc_{I_k}\!\big(f\!\circ X_f\big)
     +\frac{L_Y}{c_Y}\,\osc_{I_k}\!\big(f\!\circ Y_f\big).
\]
Summing over $k=1,\dots,B$ and using monotonicity of $f\!\circ X_f$ and $f\!\circ Y_f$,
\begin{align*}
\sum_{k=1}^B \osc_{I_k}\!\big(f\!\circ X_f\big)=\osc_{[0,1]}\!\big(f\!\circ X_f\big)=\osc_X(f),\\
\sum_{k=1}^B \osc_{I_k}\!\big(f\!\circ Y_f\big)=\osc_{[0,1]}\!\big(f\!\circ Y_f\big)=\osc_Y(f).
\end{align*}
Therefore,
\begin{align*}
\left|\frac{1}{B}\sum_{k=1}^B \scriptd_\times(u_k,v_k)-\int_0^1 \scriptd(u)\, \mathrm du\right|\\
\le\
\frac{1}{B}\!\left(\frac{L_X}{c_X}\,\osc_X(f)+\frac{L_Y}{c_Y}\,\osc_Y(f)\right),
\end{align*}
and since $\int_0^1 \scriptd(u)\, \mathrm du=J_{N,M}(f)$ and $h_B(f;x_S,y_T)=\frac{1}{B}\sum_{k=1}^B \scriptd_\times(u_k,v_k)$, the minibatch–full gap bound follows.

Now, if the samples are not stratified, we introduce another bound for the mismatch. Let $S\subset[N]$, $T\subset[M]$ be drawn uniformly without replacement (independently),
with $|S|=|T|=B$, and let $X_{f,S},Y_{f,T}$ be the corresponding batch selectors.
Then
\begin{align*}
h_B(f;x_S,y_T)
=\frac1B\sum_{k=1}^B d\!\big(X_{f,S}(u_k),Y_{f,T}(v_k)\big)^p,\\
u_k,v_k\in I_k.
\end{align*}
Add and subtract the full selectors at the same $(u_k,v_k)$ and use the triangle inequality:
\begin{align*}
&\big|\,h_B(f;x_S,y_T)-J_{N,M}(f)\,\big|\\
&\le
\left|\frac1B\sum_{k=1}^B d\!\big(X_f(u_k),Y_f(v_k)\big)^p
-\int_0^1 d\!\big(X_f(u),Y_f(u)\big)^p\, \mathrm du\right|\\
&+\underbrace{\frac1B\sum_{k=1}^B\big|\,d\!\big(X_{f,S}(u_k),Y_{f,T}(v_k)\big)^p
- d\!\big(X_f(u_k),Y_f(v_k)\big)^p\,\big|}_{\text{selector mismatch}}.
\end{align*}
The first term is already bounded by
$\tfrac1B\big(\tfrac{L_X}{c_X}\osc_X(f)+\tfrac{L_Y}{c_Y}\osc_Y(f)\big)$.

For the mismatch term, 
\begin{align*}
&\big|\,d\!\big(X_{f,S}(u_k),Y_{f,T}(v_k)\big)^p
- d\!\big(X_f(u_k),Y_f(v_k)\big)^p\,\big|
\\
&\le\; L_X\,d\big(X_{f,S}(u_k),X_f(u_k)\big)+L_Y\,d\big(Y_{f,T}(v_k),Y_f(v_k)\big)\\
&\le\; \frac{L_X}{c_X}\Delta_X(u_k)+\frac{L_Y}{c_Y}\Delta_Y(v_k),
\end{align*}
where $\Delta_X(u):=\big|f(X_{f,S}(u))-f(X_f(u))\big|$ and similarly for $\Delta_Y$.
Let $Q_X(u), Q_Y(u)$ be the empirical quantile of $f_\#\mu_X$, and $Q_{X, S}(u), Q_{Y, T}(u)$ be the quantile of the subsampled distributions. Then $f\!\circ X_{f}(u)=Q_X(u)$, $f\!\circ X_{f, S}(u)=Q_{X, S}(u)$ and $\Delta_X(u)=\big|Q_X(u)-Q_{X, S}(u)\big|$  (and similarly for $Y$). Define the quantile modulus
\begin{align*}
\omega_X(\eta):=\sup_{\substack{u,v\in[0,1]\\ |u-v|\le \eta}}\big|Q_X(u)-Q_X(v)\big|\\
\omega_Y(\eta):=\sup_{\substack{u,v\in[0,1]\\ |u-v|\le \eta}}\big|Q_Y(u)-Q_Y(v)\big|
\end{align*}
For each $k\in[B]$, let $\alpha_k$ (resp.\ $\beta_k$) be the full-data index of the $k$-th order in the subsample
from $X$ (resp.\ $Y$), and define the \emph{actual} full-data levels
\[
U_k:=\frac{\alpha_k}{N},\qquad V_k:=\frac{\beta_k}{M}.
\]
Fix deterministic references $u_k,v_k\in I_k=((k{-}1)/B,k/B]$ (e.g.\ $u_k=v_k=k/B$).
Since $f\!\circ X_{f, S}(u_k)=Q_X(U_k)$, for each $k$,
\begin{gather*}
\big|Q_{X,S}(u_k)-Q_X(u_k)\big|
=\big|Q_X(U_k)-Q_X(u_k)\big|\\
\le \omega_X\!\big(|U_k-u_k|\big),
\end{gather*}
and analogously $\big|Q_{Y,T}(v_k)-Q_Y(v_k)\big|\le \omega_Y(|V_k-v_k|)$.
Then
\begin{gather*}
\big|d(X_{f,S}(u_k),Y_{f,T}(v_k))^p-d(X_f(u_k),Y_f(v_k))^p\big|\\
\le \frac{L_X}{c_X}\,\omega_X\!\big(|U_k-u_k|\big)
   +\frac{L_Y}{c_Y}\,\omega_Y\!\big(|V_k-v_k|\big).
\end{gather*}
Averaging over $k=1,\dots,B$ and over all subsets $(S,T)$ (i.e.\ taking expectation w.r.t.\ the uniform law on size-$B$ subsets), we obtain the bound
\[
\begin{aligned}
&\big|J_{N,M,B}(f)-J_{N,M}(f)\big|\\
&=\big|\mathbb{E}_{S, T}h_B(f;x_S,y_T)-J_{N,M}(f)\big|\\
&\le\mathbb{E}_{S, T}\big|\,h_B(f;x_S,y_T)-J_{N,M}(f)\big|\\
&\le \frac{1}{B}\!\left(\frac{L_X}{c_X}\,\osc_X(f)+\frac{L_Y}{c_Y}\,\osc_Y(f)\right)\\
&\quad+\frac{1}{B}\sum_{k=1}^B\!\left(
\frac{L_X}{c_X}\,\E\,\omega_X\!\big(|U_k-u_k|\big)
+\frac{L_Y}{c_Y}\,\E\,\omega_Y\!\big(|V_k-v_k|\big)\right).\\
\end{aligned}
\]
For $\mathbb{E}\omega_X\!\big(|U_k-u_k|\big)$ and $\sigma=\frac{1}{2}\sqrt{\frac{N-B}{NB}}$ we have by Corollary \ref{cor:exp_moduli}, 
\begin{align*}
\mathbb{E}\omega_X\!\big(|U_k-u_k|\big)&\le\big(\frac{\mathbb{E}|U_k-u_k|}{\frac{1}{2}\sqrt{\frac{N-B}{NB}}}+1\big)\omega_X(\frac{1}{2}\sqrt{\frac{N-B}{NB}})\\
&\le2\omega_X(\frac{1}{2}\sqrt{\frac{N-B}{NB}})
\end{align*}
where the second inequality comes from the fact that sampling $B$ points without replacement from $N$ points gives 
\[Var(U_k)\le\frac{N-B}{4BN}\]
and consequently 
\[\mathbb{E}|U_k-u_k|\le\frac{1}{2}\sqrt{\frac{N-B}{NB}}.\]
where $u_k=\mathbb{E}U_k$ as a convenience choice.

Similarly, $\mathbb{E}\omega_Y\!\big(|V_k-v_k|\big)\le2\omega_Y(\frac{1}{2}\sqrt{\frac{M-B}{MB}}\big)$.
Therefore, we get the final bound
\[
\begin{aligned}
&\big|J_{N,M,B}(f)-J_{N,M}(f)\big|\\
&\le \frac{1}{B}\!\left(\frac{L_X}{c_X}\,\osc_X(f)+\frac{L_Y}{c_Y}\,\osc_Y(f)\right)\\
&\quad+2
\frac{L_X}{c_X}\omega_X(\frac{1}{2}\sqrt{\frac{N-B}{NB}}\big)
+2\frac{L_Y}{c_Y}\omega_Y(\frac{1}{2}\sqrt{\frac{M-B}{MB}}\big).\\
\end{aligned}
\]
\end{proof}

\section{Experiment Details}
In this section, we provide comprehensive experimental settings and implementation specifics to ensure full clarity and reproducibility. Figure \ref{fig:symmetric_train} shows the training pipeline throughout the experiment section.
\begin{figure}[t]
    \centering
    \includegraphics[width=\linewidth]{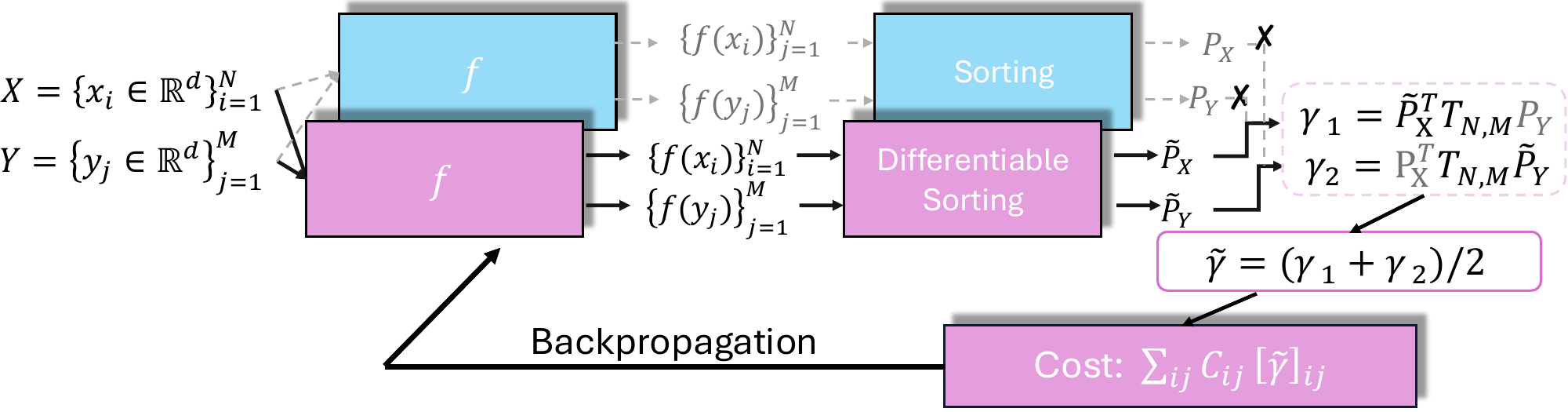}
    \caption{Training with two-branch symmetric gradients through differentiable sorting. The slicer $f$ projects $X$ and $Y$ to one-dimensional samples that are (soft) sorted to obtain $\tilde P_X$ and $\tilde P_Y$ alongside hard permutations $P_X,P_Y$.
Two plans are constructed,
$\gamma_{1}=\tilde P_X^{\!\top}T_{N,M}P_Y$ and $\gamma_{2}=P_X^{\top}T_{N,M}\tilde P_Y$, 
and the cost with their average is optimized $\tilde\gamma=\tfrac12(\gamma_1+\gamma_2)$ via $\sum_{i,j} c_{ij}\,[\tilde\gamma]_{ij}$.}
    \label{fig:symmetric_train}
\end{figure}
\label{app:details}
\subsection{Full vs. Minibatch Comparison}
\begin{figure}[H]
    \centering
    \includegraphics[width=.9\linewidth]{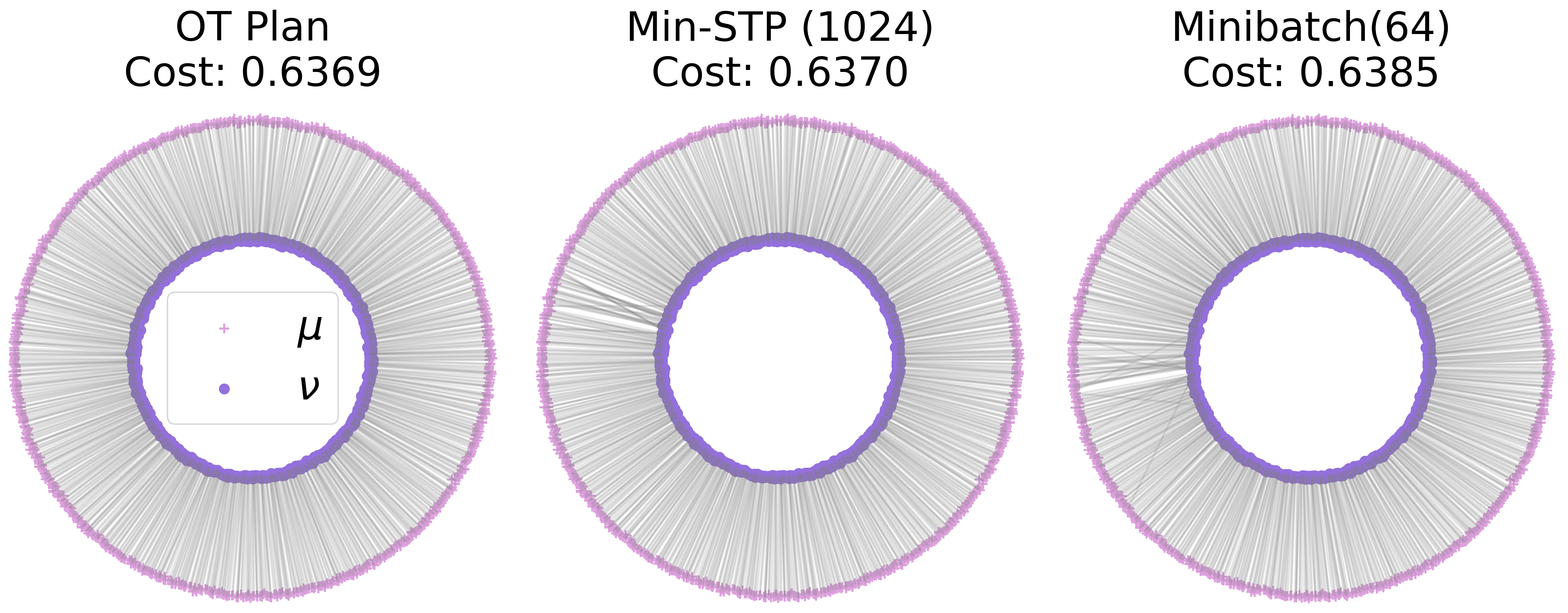}
    \caption{Transport plans and costs under different training schemes, along with the optimal plan/cost. Each panel visualizes pointwise correspondences (gray segments) between two ring
distributions, source $\mu$ and target $\nu$, with $N=M=1024$ points in each distribution. \textit{left}: exact optimal transport (OT) plan. \textit{middle}: $\mSTP$ trained with the full batch (all 1024 samples). \textit{right}: mini-batch $\mSTP$ with batch sizes $B=64$.}

    \label{fig:mini-batch}
\end{figure}

\subsection{Comparisons with DGSWP \cite{chapel2025differentiable}}
In addition to the toy example in Figure~\ref{fig:dgswp-vs-lapsum}, we compare our methods with DGSWP in the Amortized Point Cloud Alignment.

We compare for 6 pairs of categories. For fairness, we use the same neural network architecture for the slicer across methods. The train/test correlation coefficients are shown below (higher is better).

\begin{table}[H]
\centering
\caption{Pairwise Pearson correlation results.}
\label{tab:pairwise_corr}

\small
\begin{tabular}{lccc}
\toprule
Pair 
& Min STP (amortized) 
& Min STP 
& DGSWP \\
\midrule

Chair vs Bathtub
& 0.949/0.894
& 0.974/0.973
& 0.898/0.817
\\

Desk vs Bathtub
& 0.955/0.950
& 0.980/0.978
& 0.928/0.913
\\

Desk vs Sofa
& 0.916/0.914
& 0.957/0.948
& 0.873/0.879
\\

Monitor vs Bathtub
& 0.941/0.945
& 0.980/0.983
& 0.927/0.914
\\

Night Stand vs Toilet
& 0.889/0.869
& 0.977/0.964
& 0.863/0.849
\\

Night Stand vs Chair
& 0.922/0.900
& 0.971/0.967
& 0.819/0.829
\\

\bottomrule
\end{tabular}
\end{table}

\subsection{Transferability under Gradual Drift}
\subsubsection{Model Configuration}
We employ a Set-Transformer \cite{lee2019set} architecture to process unordered point sets. The input consists of two-dimensional vectors $(x,y)$, and the model outputs a scalar prediction for each element. The encoder is composed of two stacked Induced Set Attention Blocks (ISAB), each with hidden dimension $64$, $4$ attention heads, and $16$ learned inducing points, along with layer normalization. The ISAB modules implement a two-stage attention mechanism in which a fixed set of inducing vectors first attends to the input set and the input subsequently attends back to these induced features, providing an efficient and scalable approximation of full self-attention. The decoder consists of two Set Attention Blocks (SAB) with the same hidden dimension and number of heads, enabling global context propagation across all elements in the encoded representation. Finally, a linear mapping layer projects the decoder output to a one-dimensional scalar. This design leverages the Set-Transformer's inherent permutation equivariance and expressive attention-based aggregation, making it well-suited for learning functions over sets. The configurations are summarized in Table \ref{tab:set-transformer}

\begin{table}[h!]
\centering
\begin{tabular}{l c}
\hline
\textbf{Configuration Parameter} & \textbf{Value} \\ \hline
Input dimension  & $2$ \\
Output dimension  & $1$ \\
Hidden dimension  & $64$ \\
Number of inducing points (ISAB) & $16$ \\
Number of attention heads & $4$ \\
Encoder structure & ISAB $\rightarrow$ ISAB \\
Decoder structure & SAB $\rightarrow$ SAB \\
Output layer & Linear($64 \to 1$) \\ \hline
\end{tabular}
\caption{Summary of the SetTransformer configuration used \ref{exp:1}.}
\vspace{-0.3in}
\label{tab:set-transformer}
\end{table}
\subsubsection{Full Training Dynamics}
Figure~\ref{fig:complete_train} shows the full evaluation trajectories for all seven tasks.
\begin{figure}
    \centering
    \includegraphics[width=\linewidth]{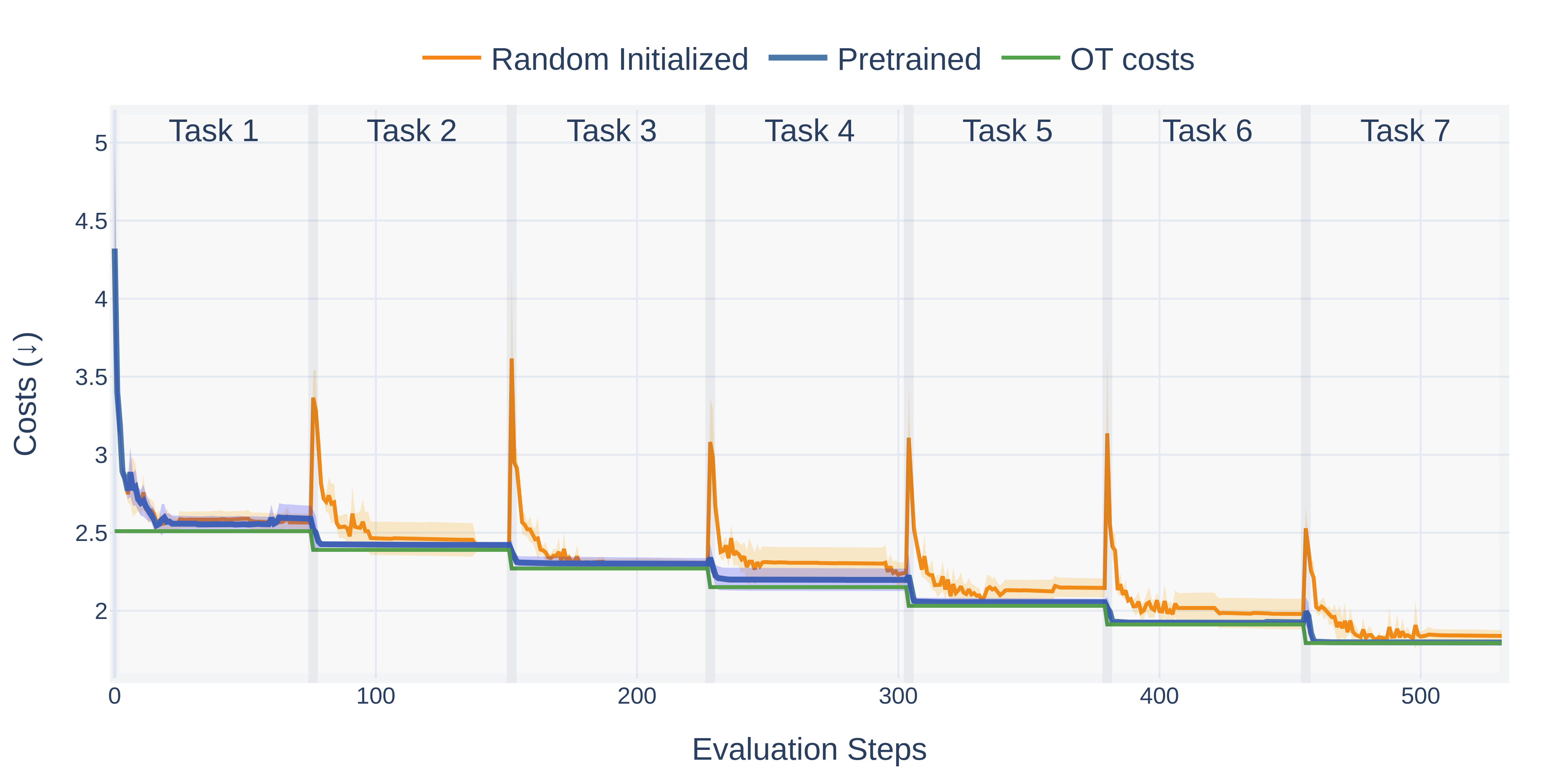}
    \caption{We plot the transport costs (over 5 runs) evaluated throughout the entire task sequence for two initialization strategies: 
randomly initialized slicer (orange), pretrained slicer (blue), and the exact OT costs (green) as a lower bound. The pretrained model exhibits fast convergence at each new task and maintains stable performance, whereas random initialization produces pronounced spikes at task boundaries before stabilizing.}
    \label{fig:complete_train}
\end{figure}

\subsubsection{Ablation studies}
To isolate the effects of minibatching, transferability, and Laplace noise, we generate two tasks similarly to the first two tasks, but with the sample size increased to 1024. The goal is to learn the optimal slicer for Task 2. We then compare the resulting transport cost at epochs 20 and 200 across different batch sizes, with and without transfer from task 1, and with and without Laplace noise / LapSum. The results are reported in the following table (averaged over 2 runs), where lower cost is better; for reference, the OT cost is 2.3688.

Overall, larger batch sizes tend to improve stability, transfer provides the strongest gain in the early stage, LapSum yields more consistent improvements across batch sizes and at later stages of training, and the combination of transfer and LapSum remains the closest to the OT cost most consistently.

\begin{table}[t]
\centering
\caption{Ablation results across different batch sizes and training iterations.}
\label{tab:ablation}

\small
\begin{tabular}{lcccccccc}
\toprule

& \multicolumn{4}{c}{Iter. 20}
& \multicolumn{4}{c}{Iter. 200} \\

\cmidrule(lr){2-5}
\cmidrule(lr){6-9}

Batchsize
& 128 & 256 & 512 & 1024
& 128 & 256 & 512 & 1024 \\
\midrule

No transfer \& no LapSum
& 3.1418 & 3.1088 & 3.1156 & 3.0680
& 2.6466 & 2.6278 & 2.7207 & 2.5679
\\

No transfer \& LapSum
& 3.1176 & 3.0668 & 3.1121 & 3.0910
& 2.4459 & 2.4642 & 2.4588 & 2.4628
\\

Transfer \& no LapSum
& 2.8252 & 2.8006 & 2.6998 & 3.0888
& 2.5213 & 2.5796 & 2.5224 & 2.4447
\\

Transfer \& LapSum
& 2.5520 & 2.5616 & 2.5169 & 2.5355
& 2.3978 & 2.3819 & 2.3810 & 2.3797
\\

\bottomrule
\end{tabular}
\end{table}

\subsection{Amortized Min-STP for Point Cloud Alignment}
\subsubsection{Dataset}
We used the ModelNet10 dataset \cite{wu20153d}, a curated subset of the Princeton ModelNet collection containing 10 object categories and 5000 CAD models, with 4078 shapes for training and 922 for testing. Following standard practice, we applied the \texttt{NormalizeScale} transform to center each shape and scale it to the unit sphere, ensuring consistent global geometry across the dataset. We then sampled 1024 points uniformly from the mesh surface using the \texttt{SamplePoints(1024)} transform, yielding point clouds $X \in \mathbb{R}^{1024 \times 3}$ used to train the PointNet autoencoder. These normalized and uniformly sampled point clouds form the input for computing the context vectors used in our ModelNet10 alignment experiments.

\begin{table}[H]
\centering
\caption{Statistics of the ModelNet10 dataset. Each object category contains CAD meshes converted to point clouds. Following standard practice, we sample 1024 points per shape using uniform surface sampling.}
\vspace{0.2cm}
\begin{tabular}{lccc}
\toprule
\textbf{Category} & \textbf{\# Train} & \textbf{\# Test} & \textbf{Total} \\
\midrule
Bathtub     & 106 & 50  & 156 \\
Bed         & 515 & 100 & 615 \\
Chair       & 889 & 100 & 989 \\
Desk        & 200 & 86  & 286 \\
Dresser     & 200 & 86  & 286 \\
Monitor     & 465 & 100 & 565 \\
Night Stand & 287 & 100 & 387 \\
Sofa        & 680 & 100 & 780 \\
Table       & 392 & 100 & 492 \\
Toilet      & 344 & 100 & 444 \\
\midrule
\textbf{Total} & 4,078 & 922 & 5,000 \\
\bottomrule
\end{tabular}
\end{table}

\subsubsection{Context Vectors}
To obtain a compact representation of each 3D shape in ModelNet10, we pretrain a PointNet-style autoencoder on point clouds \(X \in \mathbb{R}^{N \times 3}\). The encoder \(\phi\) applies a sequence of shared 1D convolutions and batch–normalization layers followed by global max pooling, producing a latent code \(c_X = \phi(X) \in \mathbb{R}^{512}\). The decoder \(\psi\) maps \(c_X\) back to a reconstructed point set \(\hat{X} \in \mathbb{R}^{N \times 3}\). We train the autoencoder end-to-end using the symmetric Chamfer distance between \(X\) and \(\hat{X}\) as the reconstruction loss, which encourages the latent code to capture the global geometry and coarse part structure of each object while remaining invariant to point permutations. After training, we discard the decoder and keep the encoder \(\phi\). For every point cloud \(X\) used in our ModelNet10 alignment experiments, we compute its latent embedding \(c_X\) and use this vector as a \emph{context vector} to condition the alignment model, allowing the learned transport plans to adapt to the specific shape instance while sharing parameters across the dataset.

\begin{table}[h!]
\centering
\caption{Training hyperparameters used for the PointNet-style autoencoder}
\vspace{0.2cm}
\begin{tabular}{l c}
\toprule
\textbf{Hyperparameter} & \textbf{Value} \\
\midrule
Latent dimension $d_z$ & 512 \\
Batch size & 8 (ModelNet10) \\
Optimizer & Adam \\
Learning rate & $5 \times 10^{-4}$ \\
Training epochs & 1000 \\
Point sampling & SamplePoints(1024) \\
\bottomrule
\end{tabular}
\end{table}

\vspace{-0.3in}
\subsection{Min-STP based Flow}
\label{subsec:flow}

For the flow matching experiments, we first train a slicer for each dataset using a SetTransformer \cite{lee2002memory} architecture with two self-attention blocks (SAB) in both the encoder and decoder. During training, we treat each point cloud as an individual target distribution and use an isotropic Gaussian as the source distribution. Similar to Section 4.2, we obtain a context vector from the latent space of a pretrained autoencoder and concatenate it with all source and target points. %The source, target, and context features are projected through separate linear layers.

Once the slicer is trained, we integrate it into the OT-MeanFlow framework to obtain a transport plan for each mini-batch. We follow the training setup and baselines provided in \cite{akbari2026transportbasedmeanflows}, and all experiments use single-step generation (NFE = 1). Additional results on the ModelNet10 \cite{wu20153d} dataset are shown in figure \ref{fig:mn_plots}. Similar to ShapeNet-Chairs, $\mSTP$ can generate point clouds that closely mimic the ground truth shapes.

\begin{figure}[H]
    \centering
    \includegraphics[width=0.7\linewidth]{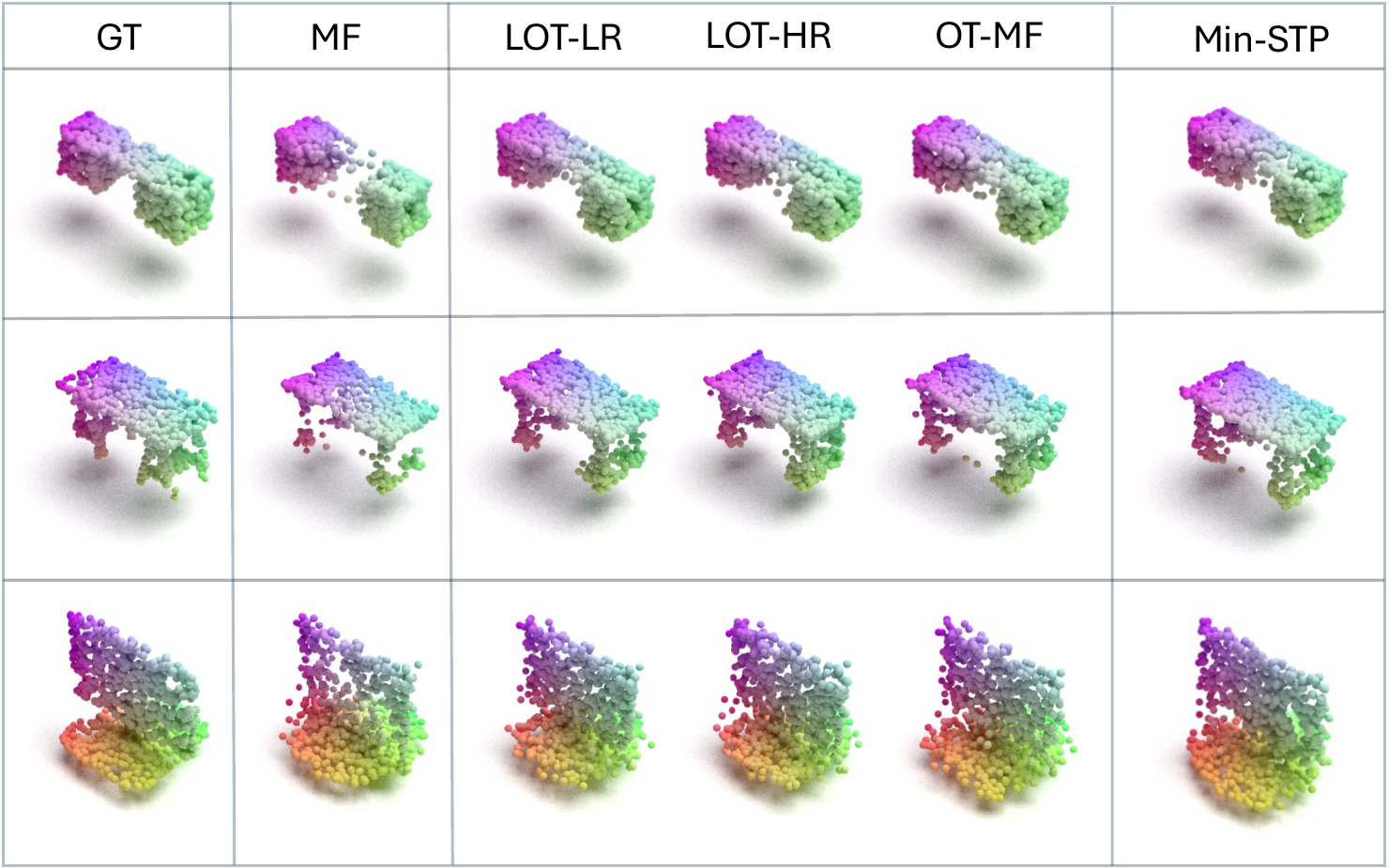}
    \caption{Single step sample generation on ModelNet10}

    \label{fig:mn_plots}
\end{figure}

\begin{figure}[H]
    \centering
    % \vspace{-0.15in}
    \includegraphics[width=0.7\linewidth]{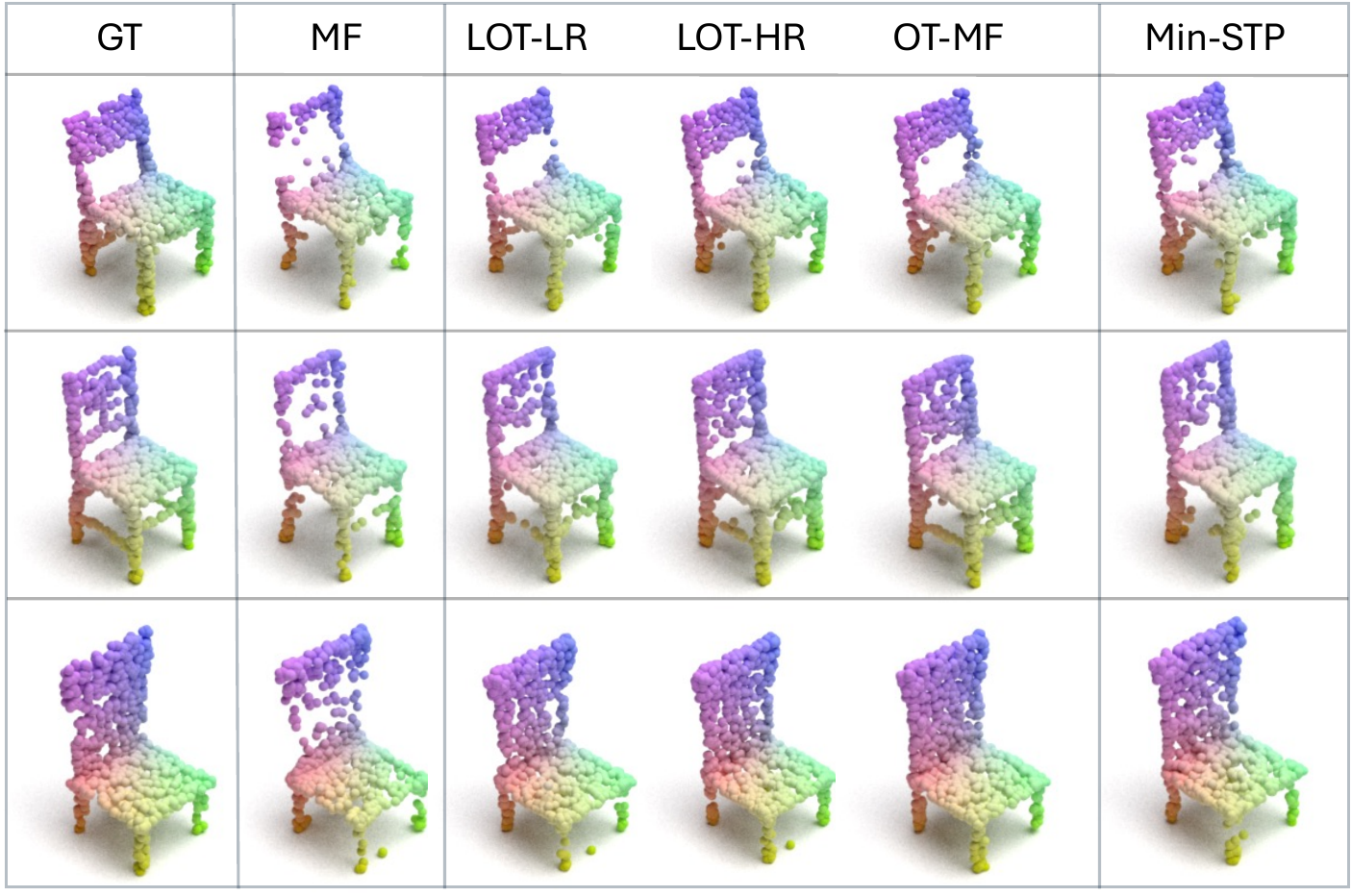}

    \caption{Single-step sample generation on ShapeNet Chairs.}
    \label{fig:sn_plots}
    % \vspace{-0.15in}
\end{figure}

\subsection{Additional Experiments on Unpaired Image-to-Image Translation}
Additional results for Man-to-Woman image translation is available in Figure \ref{fig:alae_plots_m2w}.

\begin{figure}[t]
    \centering
    % \vspace{-.1in}
    \includegraphics[width=1.0\linewidth]
    {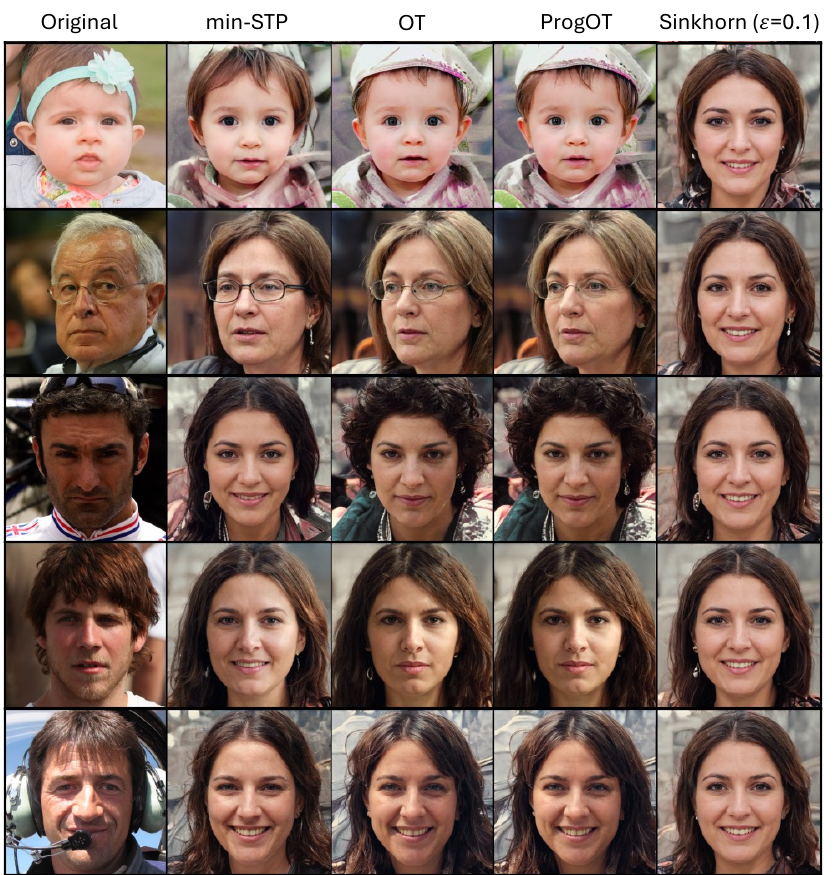}
    \vspace{-.2in}
    \caption{Man-to-Woman translation with different methods}
    \vspace{-.2in}
    \label{fig:alae_plots_m2w}
\end{figure}
\vspace{-0.1in}
\section{Algorithms}
\vspace{-0.1in}
In this section we present the key algorithms used in our method, including the mini-batch training procedure and the soft permutation operator. The mini-batch training procedure is detailed in Algorithm~\ref{alg:minstp-singlepair-gd}, with its workflow illustrated in Figure~\ref{fig:symmetric_train}. For soft permutations, we borrowed the official implementation \footnote{\url{https://github.com/gmum/LapSum}} of LapSum \cite{struski2025lapsum}. However, the code was only for Soft Top-$k$ Scores (see Algorithm \ref{alg:lapsum}), here we show how we are constructing a Soft Permutation Matrix from Soft Top-$k$ Scores.
Given a score vector $r \in \mathbb{R}^n$, we generate a soft permutation
matrix by repeatedly applying the soft top-$k$ operator and then differencing
the resulting cumulative masks. The procedure is:

\begin{algorithm}
\caption{Mini-batch training with slicer $f_\theta:\mathbb{R}^d\!\to\!\mathbb{R}$}
\label{alg:minstp-singlepair-gd}
\begin{algorithmic}
\State {\bfseries Input:} Datasets $X=\{x_i\}_{i=1}^N$, $Y=\{y_j\}_{j=1}^M$; batch size $B$; mini-batches per epoch $k$; number of epochs $E$; $\alpha$ in LapSum; learning rate $\eta$.
\State {\bfseries Output:} Optimal parameters $\theta$
\State Precompute $C_{ij}\gets c(x_i,y_j)$.
\State Set $a\gets \tfrac{1}{N}\mathbf{1}_N$, $b\gets \tfrac{1}{M}\mathbf{1}_M$.
\For{$e=1$ {\bfseries to} $E$}
  
  \State Set the batch gradient $G\gets 0$.
  
  \For{$i=1$ {\bfseries to} $k$}
    \State Draw random permutations $\sigma$ of $[N]$ and $\tau$ of $[M]$.
    
    \State Select first $B$ indices $I\gets \sigma[1:B]$, $J\gets \tau[1:B]$.
    
    \State Set the pair of batches $X_B\gets X[I]$, $Y_B\gets Y[J]$.
    
    \State $u_x\gets f_\theta(X_B)$, $u_y\gets f_\theta(Y_B)$.
    
    \State $\tilde{P}_x\gets \mathrm{LapSum}(u_x;\alpha)$ (\ref{alg:lapsum}), $P_y\gets \mathrm{HardSort}(u_y)$, $\gamma_1\gets \tilde{P}_x^\top P_y$. 
    
    \State $P_x\gets \mathrm{HardSort}(u_x)$, $\tilde{P}_y\gets \mathrm{LapSum}(u_y;\alpha)$, $\gamma_2\gets P_x^\top \tilde{P}_y$. 
    
    \State $\tilde{\gamma}_B\gets \tfrac{1}{2}(\gamma_1+\gamma_2)$ 
    
    \State $\mathcal{L}_i\gets \frac{1}{B}\langle \tilde{\gamma}_B, C[I,J]\rangle$. 
    
    \State $G\gets G + \nabla_\theta \mathcal{L}_i$. 
  \EndFor
  \State $\theta\gets \theta - \eta \cdot \frac{1}{k}\, G$. 
\EndFor
\State \textbf{Return:} $\theta$
\end{algorithmic}
\end{algorithm}

\begin{enumerate}
    \item \textbf{Compute soft top-$k$ masks for all $k$.}  
    For each $k \in \{1,\ldots,n-1\}$, apply the differentiable soft top-$k$
    operator to a copy of $r$:
    \[
        \text{softk}_k = \mathrm{SoftTopK}(r, k, \alpha),
    \]
    producing a matrix $\text{softk} \in \mathbb{R}^{(n-1)\times n}$ where
    each row is a softened indicator of the top-$k$ elements.

    \item \textbf{Pad with boundary rows.}  
    Add a row of zeros at the top and a row of ones at the bottom:
    \[
        R = 
        \begin{bmatrix}
            \mathbf{0}^\top \\[0.2em]
            \text{softk} \\[0.2em]
            \mathbf{1}^\top
        \end{bmatrix}
        \in \mathbb{R}^{(n+1)\times n}.
    \]

    \item \textbf{Difference consecutive rows to obtain a soft permutation.}  
    The soft permutation matrix is obtained by discrete differentiation:
    \[
        P_{\ell i} \;=\; R_{\ell,i} - R_{\ell-1,i},
        \qquad \ell = 1,\dots,n.
    \]
    Each row of $P$ corresponds to the (soft) probability that element $i$
    occupies rank $\ell$ in the sorted order of $r$.
\end{enumerate}

The resulting $P \in \mathbb{R}^{n\times n}$ is a row-stochastic, differentiable relaxation of a permutation matrix, consistent with the ordering encoded by the soft top-$k$ operator.

\begin{algorithm}[H]
\caption{LapSum ($\mathrm{SoftTopK}$)}
\label{alg:lapsum}
\begin{algorithmic}

\State {\bfseries Input:} Sequence $(s_i)_{i=0}^{n-1}$, parameter $k\in(0,n)$
\State Sort $s$ in decreasing order into $r=(r_i)_{i=0}^{n-1}$ 
\State Set $r_{-1}\gets +\infty$, $r_n\gets -\infty$ 
\State Initialize $a_{n-1}\gets 0$, $b_0\gets 0$, $c_0\gets 0$ 
\For{$j= n-1$ {\bfseries to} $0$}
  \State $a_{j-1}\gets (1+a_j)\,\exp(r_j-r_{j-1})$ 
\EndFor
\For{$j= 0$ {\bfseries to} $n-1$}
  \State $b_{j+1}\gets (1+b_j)\,\exp(r_{j+1}-r_j)$ 
  \State $c_{j+1}\gets 1+c_j$ 
\EndFor
\State Set $w_{-1}\gets 0$, $w_n\gets n$ 
\For{$j= 0$ {\bfseries to} $n-1$}
  \State $w_j\gets \tfrac12 a_j\,\exp(r_{j+1}-r_j)\;-\;\tfrac12 b_{j+1}\;+\;c_{j+1}$. 
\EndFor
\State Find $j\in\{0,\ldots,n\}$ such that $k\in[w_{j-1},w_j]$. 
\If{$j=0$}
  \State $b\gets r_0-\log 2-\log k+\log a_0$. 

\ElsIf{$0<j<n$}
  \State $b\gets r_{j+1}+\log a_j-\log\!\Bigl(k-c_{j+1}+\sqrt{\,\lvert k-c_{j+1}\rvert^2+a_j b_{j+1}\,e^{\,r_{j+1}-r_j}}\Bigr)$. 

\Else
  \State $b\gets r_{n-1}+\log 2+\log(c_n-k)-\log b_n$. 
\EndIf
\For{$i=0$ {\bfseries to} $n-1$}
  \State $p_i\gets \mathrm{Lap}(s_i-b)$. 
\EndFor
\State \textbf{Return:} $p=(p_i)_{i=0}^{n-1}$.
    
\end{algorithmic}
\end{algorithm}

\begin{figure*}
    \centering
    \section{Additional Experiments on Amortized Min-STP}
    \includegraphics[width=\linewidth]{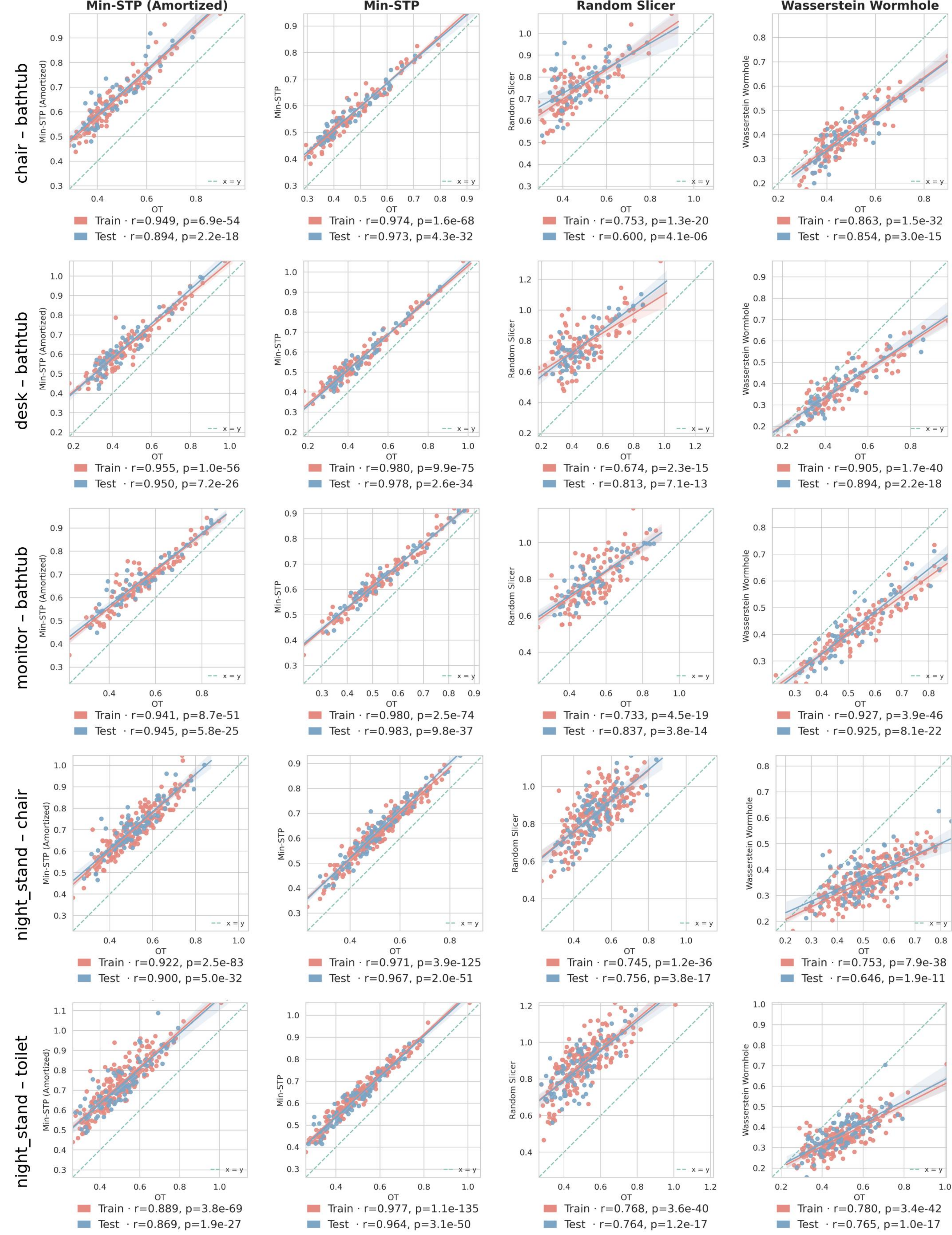}
    \caption{Additional experiments on amortized $\mSTP$.}
    \label{fig:morepairs}
\end{figure*}
%%%%%%%%%%%%%%%%%%%%%%%%%%%%%%%%%%%%%%%%%%%%%%%%%%%%%%%%%%%%

% \newpage
% \input{checklist.tex}

\end{document}